\documentclass[twoside,11pt]{article}

\usepackage{jmlr2e}

\usepackage{hyperref}       
\usepackage{url}            
\usepackage{booktabs}       
\usepackage{amsfonts}       
\usepackage{nicefrac}       
\usepackage{microtype}      
\usepackage{times}
\usepackage{calc}
\usepackage{mathtools}
\usepackage{color}
\usepackage{graphicx} 

\usepackage{DejaVuSans}
\usepackage{parskip}
\usepackage{epsfig}
\usepackage{verbatim}
\usepackage{xcolor}
\usepackage{colortbl}
\usepackage{listings}
\usepackage{subfigure}
\usepackage{algorithm}
\usepackage{algpseudocode}
\usepackage{booktabs}
\usepackage{enumitem}
\usepackage{bbm}
\usepackage{multirow}
\usepackage{pifont}
\usepackage{wrapfig}
\usepackage{longtable}
\usepackage{eqparbox}
\usepackage{xspace}
\usepackage{kk}
\usepackage{defns_mmf}

\jmlrheading{-}{2020}{-}{-}{-}{-}{
  Kirthevasan Kandasamy,
  Gur-Eyal Sela,
  Joseph E Gonzalez,
  Michael I Jordan,
  Ion Stoica
}

\ShortHeadings{Online Learning Demands in Max-min Fairness}
{Kandasamy, Sela, Gonzalez, Jordan, Stoica}
\firstpageno{1}

\begin{document}

\title{Online Learning Demands in Max-min Fairness}

\newcommand{\cmusymbol}{{\text{\incmtt{1}}}}
\newcommand{\ricesymbol}{{\text{\incmtt{2}}}}

\author{\name Kirthevasan Kandasamy \email kandasamy@eecs.berkeley.edu \\
       \name Gur-Eyal Sela  \email ges@eecs.berkeley.edu \\
       \name Joseph E Gonzalez \email jegonzal@eecs.berkeley.edu \\
       \name Michael I Jordan \email jordan@cs.berkeley.edu \\
       \name Ion Stoica \email istoica@cs.berkeley.edu  \\
       \addr 
       RISELab, University of California, Berkeley,
       CA 94723, USA \\
      }

\editor{-}

\maketitle

\algnewcommand\AlgFunction{\item[\textbf{Function}]}
\algnewcommand\AlgClass{\item[\textbf{Class}]}
\newcommand{\AlgReturn}{\textbf{return} }
\newcommand{\AlgAttribute}{\textbf{attribute}}
\newcommand{\AlgAttributes}{\textbf{attributes}\; }

\algnewcommand\algorithmicmethod{\textbf{method}}
\algdef{SE}[METHOD]{Method}{EndMethod}%
   [2]{\algorithmicmethod\ \textproc{#1}\ifthenelse{\equal{#2}{}}{}{(#2)}}%
   {\algorithmicend\ \algorithmicmethod}%

\algnewcommand\algorithmicprivmethod{\textbf{private-method}}
\algdef{SE}[METHOD]{PrivMethod}{EndPrivMethod}%
   [2]{\algorithmicprivmethod\ \textproc{#1}\ifthenelse{\equal{#2}{}}{}{(#2)}}%
   {\algorithmicend\ \algorithmicprivmethod}%

\algnewcommand{\IfThenElse}[3]{
  \State \algorithmicif\ #1,\ \; \algorithmicthen\ #2,\  \; \algorithmicelse\ #3}

\algtext*{EndMethod} 
\algtext*{EndPrivMethod} 
\algtext*{EndIf} 
\algtext*{EndFor} 
\algtext*{EndWhile} 

\newcommand{\insertDefnSep}{\vspace{0.25in}}
\newcommand{\insertDefnSmallSep}{\vspace{0.10in}}

\newcommand{\insertAlgoMMF}{
\begin{algorithm}[t]
\vspace{0.02in}

\begin{algorithmic}[1]
\Require entitlements $\{\entitli\}_{i=1}^n$, reported demands $\{\demi\}_{i=1}^n$.
\State $r\leftarrow 1, \quad \entitl\leftarrow 1,
          \quad S\leftarrow \{1,\dots, n\}, \quad \alloc\leftarrow \zerov_n$.
        \label{lin:reinit}
\For {$j$ in ascending order of $\frac{\demi}{\entitli}$}
    \If{$\demj < r \frac{\entitlj}{\entitl}$}
        \label{lin:mmfifcondn}
        \State $\allocj \leftarrow \demj$.
        \label{lin:allocdemand}
        \State
        $S\leftarrow S \backslash \{j\},
            \quad r\leftarrow r - \demi, \quad \entitl\leftarrow \entitl-\entitli$.
        \label{lin:reupdate}
    \Else
        \Comment{Allocate proportionally to all remaining agents and exit.}
        \State $\allock \leftarrow r\frac{\entitlk}{\entitl}$ for all $k\in S$.
        \label{lin:allocsaturate}
        \State \textbf{Break}.
    \EndIf
\EndFor \\
\Return $\alloc$
\end{algorithmic}
\caption{$\;$\mmf \label{alg:mmf}}
\end{algorithm}
}

\newcommand{\insertAlgoMMFLearnSP}{
\begin{algorithm}[t]
\vspace{0.02in}

\begin{algorithmic}[1]
\Require entitlements $\{\entitli\}_{i}$, definitions for the function $r'$, 
        the function \explphase, and the user class \userclass.
\State \uci $\leftarrow$ \userclass.\ucinit\callemptys for all users $i$.
\Comment{Instantiate \userclasss for all users}
\For{$q = 1, 2, \dots$}
    \State \explphase($\{\uci\}_i$) \Comment{Execute exploration phase}
        \label{lin:spudiub}
    \State $\udubi\leftarrow$ \uci.\ucgetudub\callemptys \, for each user $i$.
            \Comment{Compute upper bound for unit demand}
    \For{$r=1,\dots,r'(q)$},
        \State $\{\voli\}_{i}\leftarrow$ obtain loads from agents.
        \State $\alloc\leftarrow\text{\mmf}(\{\entitli\}_{i}, \{\voli\times\udubi\}_{i})$.
        \Comment{\mmf{} from Algorithm~\ref{alg:mmf}}
        \label{lin:spallocub}
        \State  Allocate according to $\alloc$.
     \EndFor
\EndFor
\end{algorithmic}
\caption{$\;$\mmflearnsp \label{alg:mmflearnsp}}
\end{algorithm}
}

\newcommand{\insertAlgoMMFLearnNSP}{
\begin{algorithm}[t]
\vspace{0.02in}

\begin{algorithmic}[1]
\Require entitlements $\{\entitli\}_{i}$, user class definition \userclass.
\State \uci $\leftarrow$ \userclass.\ucinit\callempty for all users $i$.
\Comment{Instantiate  for all users}
\State Allocate $\entitli$ to each user and collect rewards and SGCs $\{(X_i,
        \sigmai)\}_{i}$ from all users.
                 \Comment{Round 1}
\State \uci.\ucrecordfb($\alloci/\voli$, $X_i$, $\sigmai$) for each user $i$.
\For{$t=2,3,\dots$},
        \State \parbox[t]{\dimexpr\textwidth-\leftmargin-\labelsep-\labelwidth}{%
               $\udi\leftarrow$ \uci.\ucgetudrec\callempty \, for each user $i$.
            \Comment{Obtain recommendation for unit demand}
                \strut}
        \label{lin:nspudirec}
        \State $\{\voli\}_{i}\leftarrow$ obtain loads from agents.
        \State $\alloc\leftarrow\text{\mmf}(\{\entitli\}_{i}, \{\voli\times\udi\}_{i})$.
        \Comment{\mmf{} from Algorithm~\ref{alg:mmf}}
        \label{lin:nspalloc}
        \State \parbox[t]{\dimexpr\textwidth-\leftmargin-\labelsep-\labelwidth}{%
                Allocate according to $\alloc$ and collect rewards and SGCs
                $\{(X_i, \sigmai)\}_{i}$ from all users.
        \strut}
        \State \uci.\ucrecordfb($\alloci/\voli$, $X_i$, $\sigmai$) for each user $i$.
\EndFor
\end{algorithmic}
\caption{$\;$\mmflearnnsp \label{alg:mmflearnnsp}}
\end{algorithm}
}

\newcommand{\insertBothMMFLearnAlgos}{
\begin{table}
\begin{minipage}[t]{2.63in}
  \vspace{0pt}
  \insertAlgoMMFLearnSP
\end{minipage}%
  \hspace{0.05in}
\begin{minipage}[t]{2.8in}
  \vspace{0pt}
  \insertAlgoMMFLearnNSP
\end{minipage}
  \vspace{-0.1in}
\end{table}
}

\newcommand{\insertTableResultsSummary}{
\newcommand{\insertbetweentablegaps}{\multicolumn{4}{l}{} \\[-0.05in]}
\begin{table}[t]
\centering
\begin{tabular}{l|c|c|c}
\toprule
 &\begin{tabular}{@{}c@{}} Efficiency \\ $\LOTT \in \bigOtilde(\textrm{?}) $\end{tabular} &
  \begin{tabular}{@{}c@{}} Fairness \\ $\UeiiT - \UiT \in \bigOtilde(\textrm{?}) $\end{tabular} &
  \begin{tabular}{@{}c@{}} Strategy-proofness \\ $\UpiiT - \UiT \in \bigOtilde(\textrm{?}) $\end{tabular} 
    \\
\midrule
\insertbetweentablegaps
\multicolumn{4}{l}{1. Deterministic Feedback} \\
\midrule
 \begin{tabular}{@{}c@{}} Algorithms~\ref{alg:mmflearnsp} \&~\ref{alg:detfbspmodel}
    \\ Theorem~\ref{thm:detfbsp}\end{tabular} &
 $n^{\nicefrac{3}{2}}\sqrtT$ & 
 $\sqrtn\sqrtT$ &
  $\leq 0$
 \\
\midrule
 \begin{tabular}{@{}c@{}} Algorithms~\ref{alg:mmflearnsp} \&~\ref{alg:detfbmodel}
    \\ Theorem~\ref{thm:detsp}\end{tabular} &
    $n\log(T)$ &
    $n\log(T)$ &
    $1$
 \\
\midrule
 \begin{tabular}{@{}c@{}} Algorithms~\ref{alg:mmflearnnsp} \&~\ref{alg:detfbmodel}
    \\ Theorem~\ref{thm:detnsp}\end{tabular} &
    $n$ &
    $1$ &
     --
 \\
\midrule
\insertbetweentablegaps
\multicolumn{4}{l}{2. Stochastic Feedback with Parametric Payoffs (bounds hold w/p $\geq 1-\delta$)} \\
\midrule
 \begin{tabular}{@{}c@{}} Algorithms~\ref{alg:mmflearnsp} \&~\ref{alg:glmfbmodel}
    \\ Theorem~\ref{thm:glmsp}\end{tabular} &
 $n T^{\nicefrac{2}{3}}$ &
 $\leq \; 0$  &
 $\leq \; 0$ 
 \\
\midrule
 \begin{tabular}{@{}c@{}} Algorithms~\ref{alg:mmflearnnsp} \&~\ref{alg:glmfbmodel}
    \\ Theorem~\ref{thm:glmnsp}\end{tabular} &
 $nT^{\nicefrac{1}{2}}$ &
 $\leq \; 0$  &
 $n T^{\nicefrac{1}{2}}$,\quad BNIC
 \\
\midrule
\insertbetweentablegaps
\multicolumn{4}{l}{3. Stochastic Feedback with Nonparametric Payoffs
    (bounds hold w/p $\geq 1-\delta$)} \\
\midrule
 \begin{tabular}{@{}c@{}} Algorithms~\ref{alg:mmflearnsp}
    \&~\ref{alg:nonparfbmodelpart1}-\ref{alg:nonparfbmodelpart4}
    \\ Theorem~\ref{thm:nonparsp}\end{tabular} &
    \quad$G^{\nicefrac{-3}{2}}\,n^{\nicefrac{4}{3}}T^{\nicefrac{2}{3}}$\quad\quad
&
    $\nonth T^{\nicefrac{2}{3}}$
&
    $\nmtwth T^{\nicefrac{2}{3}}$
 \\
\midrule
 \begin{tabular}{@{}c@{}} Algorithms~\ref{alg:mmflearnnsp}
\&~\ref{alg:nonparfbmodelpart1}-\ref{alg:nonparfbmodelpart4}
    \\ Theorem~\ref{thm:nonparnsp}\end{tabular} &
    $G^{\nicefrac{-3}{2}}\,nT^{\nicefrac{1}{2}}$
&
    $G^{\nicefrac{-3}{2}}\,T^{\nicefrac{1}{2}}$
     &
     --
 \\
\bottomrule
\end{tabular}
\caption{%
Summary of asymptotic rates for efficiency, fairness, and strategy-proofness
for the three feedback models when using different mechanisms,
where we have hidden any logarithmic dependence on $T$ when there are polynomial terms.
BNIC indicates that the mechanism is (asymptotically) Bayes-Nash incentive-compatible.
When the bounds hold non-asymptotically, we write 
`$\leq 0$', e.g. $\UeiiT-\UiT\leq 0$ in both mechanisms for the second model.
In the third model, the bounds hold for any $G>\ntg$ (Definition~\ref{def:ntg}).
\label{tb:resultssummary}
}
\end{table}
}

\newcommand{\insertAlgoDetFbSPModel}{
\begin{algorithm}[t]
\vspace{0.02in}

\begin{algorithmic}[1]
\AlgFunction $r'(q)$ \\
    \Return $nq$

\insertDefnSep

\AlgFunction \explphase($\{\uci\}_i$)
\For{$i=1,\dots,n$}
    \State $\voli\leftarrow$ obtain load from user $i$.
    \State $h\leftarrow \lceil \log_2(q+1)\rceil$, \quad
           $k\leftarrow 2q - 2^h + 1$.
    \State $\udi \leftarrow \udmax k /2^h$.
    \State Allocate $\udi\times\voli$ to user $i$ only and
            obtain feedback $X_i$.
    \State \uci.\ucrecordfb($\udi$, $X_i$).
\EndFor

\insertDefnSep

\AlgClass \userclass \\
\AlgAttributes $\udub$ \Comment{Upper bound on the unit demand}
\insertDefnSmallSep
\Method{\ucinit\callempty}{}
\State $\ucself.\udub\leftarrow \udmax$.
\EndMethod

\insertDefnSmallSep

\Method{\ucgetudub\callempty}{}
\Comment{Main interface for Algorithm~\ref{alg:mmflearnsp}}
\State \AlgReturn $\ucself.\udub$
\EndMethod

\insertDefnSmallSep

\Method{\ucrecordfb($\normalloc$, $X_i$)}{}
\Comment{SGCs are not relevant for deterministic feedback}
\label{lin:ucrecordfbssp}
\If{$X_i > \threshi$}
\State $\ucself.\udub \leftarrow \min(\ucself.\udub,  \;\normalloc)$
\label{lin:detfbspupdate}
\EndIf
\EndMethod

\end{algorithmic}
\caption{$\;$Definitions for the Deterministic Feedback Model with  Strategy-Proofness%
\label{alg:detfbspmodel}}
\end{algorithm}
}

\newcommand{\insertAlgoDetFbModel}{
\begin{algorithm}[t]
\vspace{0.02in}

\begin{algorithmic}[1]
\AlgFunction $r'(q)$ \\
    \Return $\lfloor e^q \rfloor$

\insertDefnSep

\AlgFunction \explphase($\{\uci\}_i$)
\For{$j=1,\dots,2n$}
    \State $i\leftarrow \; j\mod n$
    \State $\voli\leftarrow$ obtain load from user $i$.
    \State $\udi\leftarrow$ \uci.\ucgetudrecforub().
    \State Allocate $\udi\times\voli$ to user $i$ only and obtain feedback $X_i$.
    \State \uci.\ucrecordfb($\udi$, $X_i$)
\EndFor

\insertDefnSep

\AlgClass \userclass \\
\AlgAttributes $\udub, \udlb$ \Comment{Upper and lower bounds on unit demand}
\insertDefnSmallSep
\Method{\ucinit\callempty}{}
\State $\ucself.\udlb\leftarrow 0$, \;$\ucself.\udub\leftarrow \udmax$.
\EndMethod

\insertDefnSmallSep

\Method{\ucgetudub\callempty}{}
\Comment{Main interface for Algorithm~\ref{alg:mmflearnsp}}
\State \AlgReturn $\ucself.\udub$
\EndMethod

\insertDefnSmallSep

\Method{\ucgetudrec\callempty}{}
\Comment{Main interface for Algorithm~\ref{alg:mmflearnnsp}}
\State \AlgReturn $\frac{1}{2}\left(\ucself.\udlb + \ucself.\udub \right)$
\EndMethod

\insertDefnSmallSep

\Method{\ucrecordfb($\normalloc$, $X_i$)}{}
\Comment{SGCs are not relevant for deterministic feedback}
\label{lin:ucrecordfbsp}
\If{$X_i < \threshi$}
\State $\ucself.\udlb \leftarrow \max(\ucself.\udlb, \; \normalloc)$
\Else
\State $\ucself.\udub \leftarrow \min(\ucself.\udub,  \;\normalloc)$
\EndIf
\EndMethod

\end{algorithmic}
\caption{$\;$Definitions for the Deterministic Feedback Model \label{alg:detfbmodel}}
\end{algorithm}
}

\newcommand{\insertAlgoGLMFbModel}{
\begin{algorithm}[t]
\vspace{0.02in}

\begin{algorithmic}[1]
\AlgFunction $r'(q)$ \\
    \Return $\lfloor 5q^{\nicefrac{1}{2}} / 6 \rfloor$

\insertDefnSep

\AlgFunction \explphase($\{\uci\}_i$)
\State $\{\voli\}_{i}\leftarrow$ obtain loads from agents.
\State  Allocate $\entitli$ to each user $i$  and  obtain feedback $X_i$.
\State \uci.\ucrecordfb($\entitli/\voli$, $X_i$, $\sigmai$) for each user $i$.

\insertDefnSep

\AlgClass \userclass \\
\AlgAttributes $\Dcal, \udub$ \Comment{Past data $(\Dcal)$ and upper bound on unit demand $(udub)$}
\insertDefnSmallSep
\Method{\ucinit\callempty}{}
\State $\ucself.\udub\leftarrow \udmax$.
\EndMethod

\insertDefnSmallSep

\Method{\ucgetudub\callempty}{}
\Comment{Main interface for Algorithm~\ref{alg:mmflearnsp}}
\State \AlgReturn $\ucself.\udub$
\EndMethod

\insertDefnSmallSep

\Method{\ucgetudrec\callempty}{}
\Comment{Main interface for Algorithm~\ref{alg:mmflearnnsp}}
\State \AlgReturn $\ucself.\udub$
\EndMethod

\insertDefnSmallSep

\Method{\ucrecordfb($\normalloc$, $X_i$, $\sigmai$)}{}
\State Add $(\normalloci, X_i, \sigmai)$ to $\ucself.\Dcal$.
\State $\ucself.\udub \leftarrow$ Compute upper bound as described in~\eqref{eqn:thetaitdefn}
and~\eqref{eqn:udglmconfint}.
\EndMethod

\end{algorithmic}
\caption{$\;$Definitions for the Stochastic Feedback Model with Parametric
Payoffs \label{alg:glmfbmodel}}
\end{algorithm}
}

\newcommand{\insertAlgoNonparFbModelOne}{
\begin{algorithm}[t]
\vspace{0.02in}

\begin{algorithmic}[1]
\AlgFunction $r'(q)$ \\
\label{lin:nonparrpq}
    \Return $\lfloor 5nq^{\nicefrac{1}{2}} / 6 \rfloor$

\insertDefnSep

\AlgFunction \explphase($\{\uci\}_i$)
\label{lin:nonparexplphase}
\For{$i=1,\dots,n$}
    \State $\voli\leftarrow$ obtain load from user $i$.
    \State $\udi\leftarrow$ \uci.\ucgetudrecforub().
    \State Allocate $\udi\times\voli$ to user $i$ only and obtain feedback $X_i$ and SGC
                $\sigmai$.
    \State \uci.\ucrecordfb($\udi$, $X_i$, $\sigmai$)
\EndFor

\insertDefnSep

\AlgClass \userclass \\
\AlgAttributes $\Tcal, \VS, \fbar, \flb, \fub, \Blb, \Bub, \hub, \kub$
\Comment{%
$\Tcal$ stores the expanded tree, $\VS$ stores}
\Statex \Comment{the sum of inverse squared SGCs. $\fbar, \flb, \fub, \Blb, \Bub$ are used to
compute the lower/}
\Statex \Comment{upper bounds.
$(\hub,\kub)$ are used in computing an upper confidence bound for $\udtruei$.}

\insertDefnSmallSep
\Method{\ucinit\callempty}{}
\State $\ucself.\Tcal \leftarrow \{(0, 1)\}$.
\State $\ucself.\Blbhhkk{0}{1} \leftarrow 0$, \quad $\ucself.\Bubhhkk{0}{1} \leftarrow 1$.
\State \ucself.\ucexpandnode($(0,1)$).
\Comment{Line~\ref{lin:ucexpandnode}}
\EndMethod

\insertDefnSmallSep
\Method{\ucgetudub\callempty}{}
\label{lin:nonparucgetudub}
\Comment{Main interface for Algorithm~\ref{alg:mmflearnsp}}
    \State \AlgReturn \; $\udmax \kub/2^{\hub}$.
\Comment{Right-side boundary of $\Ihkub$. $(\hub,\kub)$ is updated in Line~\ref{lin:hkubupdate}}
\EndMethod

\insertDefnSmallSep
\Method{\ucgetudrec\callempty}{}
\label{lin:nonparucgetudrec}
\Comment{Main interface for Algorithm~\ref{alg:mmflearnnsp}}
    \If{$t=\ttilde$}
        \State $\ucself.\ucrefreshboundsintree\callempty$
        \label{lin:callrefreshinucgetudrec}
    \Comment{Line~\ref{lin:ucrefreshboundsintree}}
    \EndIf
    \State $(h, k)\leftarrow (0, 1)$
        \Comment{\insrfont{is-a-leaf} returns true if $(h,k)$ is a leaf of $\Tcal$}
    \While{\;\textbf{not} $\insrfont{is-a-leaf}(\ucself.\Tcal, (h, k))$ \;\textbf{and}\;
           $\ucself.\VShk \geq \tauht$}
        \If{$\ucself.\ucgetBval(h+1,2k-1) >  \ucself.\ucgetBval(h+1, 2k)$}
    \Comment{Line~\ref{lin:ucgetBval}}
            \State $(h, k) \leftarrow (h+1, 2k-1)$
        \Else
            \State $(h, k) \leftarrow (h+1, 2k)$
        \EndIf
    \EndWhile
    \State \AlgReturn \;An arbitrary point in $\Ihk$.
\EndMethod

\algstore{nonparfbmodel}
\end{algorithmic}
\caption{$\;$Definitions for the Stochastic Feedback Model with Nonparametric 
Payoffs -- Part I\label{alg:nonparfbmodelpart1}}
\end{algorithm}
}

\newcommand{\insertAlgoNonparFbModelTwo}{
\begin{algorithm}[t]
\begin{algorithmic}[1]
\algrestore{nonparfbmodel}

\insertDefnSmallSep
\Method{\ucrecordfb($\normalloc$, $X_i$, $\sigmai$)}{}
\label{lin:nonparucrecordfb}
    \State $(h, k)\leftarrow (0, 1)$
    \While{$(h,k) \in \ucself.\Tcal$ \;\textbf{and}\;
           $\ucself.\VShk \geq \tauht$}
    \Comment{While loop computes $\Pit$}
        \label{lin:stopPit}
    \State $\ucself.\ucassigntonode((h, k), \normalloc, X_i, \sigmai)$
    \Comment{Line~\ref{lin:ucassigntonode}}
        \IfThenElse{$\normalloc<\frac{1}{2}(\lhk+\rhk)$}%
                   {$(h, k) \leftarrow (h+1, 2k-1)$}%
                   {$(h, k) \leftarrow (h+1, 2k)$}
    \EndWhile
    \State \ucself.\ucupdateboundsonpathtoroot($(h,k)$)
        \label{lin:callupdateboundsonpath}
    \Comment{Line~\ref{lin:ucupdateboundsonpathtoroot}}
    \If{$\insrfont{is-a-leaf}(h,k)$ \textbf{and}
           $\ucself.\VShk \geq \tauht$}
        \State \ucself.\ucexpandnode($(h,k)$)
\Comment{Line~\ref{lin:ucexpandnode}}
    \EndIf
\EndMethod

\insertDefnSmallSep
\Method{\ucgetudrecforub\callempty}{}
\label{lin:nonparucgetudrecforub}
\Comment{Used by \explphases (line~\ref{lin:nonparexplphase})}
    \State $(\hub, \kub)\leftarrow \ucself.\ucubtraverse()$
    \label{lin:hkubupdate}
\Comment{Line~\ref{lin:nonparucubtraverse}}
    \State \AlgReturn \;An arbitrary point in $\Ihkub$.
\EndMethod

\insertDefnSmallSep
\PrivMethod{\ucubtraverse\callempty}{}
\label{lin:nonparucubtraverse}
    \State $(h, k)\leftarrow (0, 1)$
        \Comment{\insrfont{is-a-leaf} returns true if $(h,k)$ is a leaf of $\Tcal$}
    \While{\;\textbf{not} $\insrfont{is-a-leaf}(\ucself.\Tcal, (h, k))$ \textbf{and}
           $\ucself.\VShk \geq \tauht$}
        \IfThenElse{$\ucself.\Blbhhkk{h+1}{2k} \geq \thresh$}%
                   {$(h, k) \leftarrow (h+1, 2k-1)$}
                   {$(h, k) \leftarrow (h+1, 2k-1)$} \hspace{-0.10in}
            \label{lin:nonparucubtraverseifcondn}
    \EndWhile
\State   \AlgReturn $(h,k)$
\EndPrivMethod

\insertDefnSmallSep
\PrivMethod{\ucgetBval$(h, k)$}{}
\label{lin:ucgetBval}
\State \AlgReturn $\ucself.\min(\ucself.\Bubhk-\thresh,
                          \thresh-\ucself.\Blbhk)$
\EndPrivMethod

\insertDefnSmallSep
\PrivMethod{\ucexpandnode($(h, k)$)}{}
\label{lin:ucexpandnode}
\State $\ucself.\Tcal \leftarrow \ucself.\Tcal \cup \{(h+1, 2k-1), (h+1, 2k)\}$.
\State $(\ell, u)\leftarrow \ucself.\ucgetboundsforunexpandednode((h+1, 2k-1))$
    \Comment{Line~\ref{lin:ucgetboundsforunexpandednode}}
\State $\ucself.\Blbhhkk{h+1}{2k-1} \leftarrow \ell$,
        \quad $\ucself.\Bubhhkk{h+1}{2k-1} \leftarrow u$.
\State $\ucself.\Blbhhkk{h+1}{2k} \leftarrow \ell$, \quad
        $\ucself.\Bubhhkk{h+1}{2k} \leftarrow u$.
\State \AlgReturn $\ucself.\udub$
\EndPrivMethod

\insertDefnSmallSep
\PrivMethod{\ucupdateboundsfornodesatsamedepth($(h, k)$)}{}
\label{lin:ucupdateboundsfornodesatsamedepth}
    \For{$k'=k+1,\dots,2^h$}
        \IfThenElse{$\ucself.\Blbhhkk{h}{k'} < \ucself.\Blbhk$}%
                   {$\ucself.\Blbhhkk{h}{k'} \leftarrow\ucself.\Blbhk$}
                   {\textbf{break}}
    \EndFor
    \For{$k'=k-1,\dots,1$}
        \IfThenElse{$\ucself.\Bubhhkk{h}{k'} > \ucself.\Bubhk$}%
                   {$\ucself.\Bubhhkk{h}{k'} \leftarrow \ucself.\Bubhk$}%
                   {\textbf{break}}
    \EndFor
\EndPrivMethod

\algstore{nonparfbmodeltwo}
\end{algorithmic}
\caption{$\;$Definitions for the Stochastic Feedback Model with Nonparametric 
Payoffs -- Part II \label{alg:nonparfbmodelpart2}}
\end{algorithm}
}

\newcommand{\insertAlgoNonparFbModelThree}{
\begin{algorithm}[t]
\begin{algorithmic}[1]
\algrestore{nonparfbmodeltwo}

\insertDefnSmallSep
\PrivMethod{$\ucassigntonode((h, k), \normalloc, X_i, \sigmai)$}{}
\label{lin:ucassigntonode}
    \If{$\ucself.\VShk = 0$}
        \State $\ucself.\fbarhk \leftarrow X_i$
        \State $\ucself.\VShk \leftarrow \sigmai^{-2}$
    \Else
        \State $\ucself.\fbarhk \leftarrow \frac{\ucself.\VShk\cdot\ucself.\fbarhk + X_i/\sigmai^2}
                                            {\VShk + \sigmai^{-2}}$
        \State $\ucself.\VShk \leftarrow \ucself.\VShk + \sigmai^{-2}$
    \EndIf
%
    \State $\ucself.\flbhk \leftarrow \ucself.\fbarhk -
                \betattilde\;\ucself.\VShk^{-\nicefrac{1}{2}} - L\cdot2^{-h}$
    \State $\ucself.\fubhk \leftarrow \ucself.\fbarhk +
                \betattilde\;\ucself.\VShk^{-\nicefrac{1}{2}} + L\cdot2^{-h}$
\EndPrivMethod

\insertDefnSmallSep
\PrivMethod{\ucupdateboundsonpathtoroot($(h,k)$)}{}
\label{lin:ucupdateboundsonpathtoroot}
    \If{\insrfont{is-a-leaf}$(\ucself.\Tcal, (h, k))$}
        \State $(\ell, u)\leftarrow \ucself.\ucgetboundsforunexpandednode((h+1, 2k-1))$
    \Comment{Line~\ref{lin:ucgetboundsforunexpandednode}}
        \State $\ucself.\Blbhk\leftarrow \max(\ucself.\flbhk, \; \ucself.\Blbhk, \ \ell)$
        \State $\ucself.\Bubhk\leftarrow \min(\ucself.\fubhk, \; \ucself.\Bubhk, \; u)$
        \State $\ucself.\ucupdateboundsfornodesatsamedepth((h, k))$
    \Comment{Line~\ref{lin:ucupdateboundsfornodesatsamedepth}}
        \State $(h, k)\leftarrow (h-1, \lfloor(k+1)/2\rfloor)$ \Comment{Set $(h, k)$ to its parent}
    \EndIf
    \While{$h\neq -1$} \Comment{Stop when you reach $(0,1)$}
        \State $\ucself.\Blbhk\leftarrow \max(\ucself.\flbhk, \;  \ucself.\Blbhk, \; 
                                              \ucself.\Blbhhkk{h+1}{2k-1})$
        \State $\ucself.\Bubhk\leftarrow \min(\ucself.\fubhk, \;  \ucself.\Bubhk, \; 
                                              \ucself.\Bubhhkk{h+1}{2k})$
        \State $\ucself.\ucupdateboundsfornodesatsamedepth((h, k))$
    \Comment{Line~\ref{lin:ucupdateboundsfornodesatsamedepth}}
        \State $(h, k)\leftarrow (h-1, \lfloor(k+1)/2\rfloor)$ \Comment{Set $(h, k)$ to its parent}
        
    \EndWhile
\EndPrivMethod

\insertDefnSmallSep
\Method{$\ucgetconfinterval(\normalloc)$}{}
\label{lin:ucgetconfinterval}
    \State $(h,k) \leftarrow (0, 1)$
    \State $\bcheck\leftarrow 0, \quad \bhat\leftarrow 1$
    \While{$(h,k)\in\ucself.\Tcal$}
        \State $\bcheck \leftarrow \max(\bcheck, \ucself.\Blbhk), \quad
                \bhat \leftarrow \min(\bhat, \ucself.\Bubhk)$
        \IfThenElse{$\normalloc<\frac{1}{2}(\lhk+\rhk)$}%
                   {$(h, k) \leftarrow (h+1, 2k-1)$}%
                   {$(h, k) \leftarrow (h+1, 2k)$}
    \EndWhile
    \State $(l,u) \leftarrow \ucself.\ucgetboundsforunexpandednode(h, k)$
    \Comment{Line~\ref{lin:ucgetboundsforunexpandednode}}
    \State $\bcheck \leftarrow \max(\bcheck, l), \quad
            \bhat \leftarrow \min(\bhat, u)$
    \State \AlgReturn $(\bcheck, \bhat)$
\EndMethod

\algstore{nonparfbmodelthree}
\end{algorithmic}
\caption{$\;$Definitions for the Stochastic Feedback Model with Nonparametric 
Payoffs -- Part III \label{alg:nonparfbmodelpart3}}
\end{algorithm}
}

\newcommand{\insertAlgoNonparFbModelFour}{
\begin{algorithm}[t]
\begin{algorithmic}[1]
\algrestore{nonparfbmodelthree}

\insertDefnSmallSep
\PrivMethod{\ucgetboundsforunexpandednode($(h, k)$)}{}
\label{lin:ucgetboundsforunexpandednode}
\If{$h > \uchmax(\ucself.\Tcal)$}
    \Comment{$\uchmax(\Tcal) = \max\{h\,;\; (h,k) \text{ has been expanded in } \Tcal\}$}
    \State \AlgReturn $(0, 1)$
\Else
    \State $\ell\leftarrow$ $\Blbhhkk{h}{k'}$ where $k'$ is the largest $k''<k$ such that
            $(h,k'')$ has been expanded.
    \State $u\leftarrow$ $\Bubhhkk{h}{k'}$ where $k'$ is the smallest $k''>k$ such that
            $(h,k'')$ has been expanded.
    \State $(\ell', u')\leftarrow \ucself.\ucgetboundsforunexpandednode(h+1, 2k-1)$
    \Comment{Recurse}
    \State \AlgReturn $(\max(\ell, \ell'), \min(u, u'))$
\EndIf
\EndPrivMethod

\insertDefnSmallSep
\PrivMethod{\ucrefreshboundsintree\callempty}{}
\label{lin:ucrefreshboundsintree}
    \For{$(h,k) \in \ucself.\Tcal$} \Comment{Update $\flb,\fub$ values of all nodes with new
$\ttilde$ value}
    \State $\ucself.\flbhk \leftarrow \ucself.\fbarhk -
                \betattilde\;\ucself.\VShk^{-\nicefrac{1}{2}} - L\cdot2^{-h}$
    \State $\ucself.\fubhk \leftarrow \ucself.\fbarhk +
                \betattilde\;\ucself.\VShk^{-\nicefrac{1}{2}} + L\cdot2^{-h}$
    \EndFor
    \For{$h = \uchmax(\ucself.\Tcal), \dots, 0$}
    \Comment{$\uchmax(\Tcal) = \max\{h\,;\; (h,k) \text{ has been expanded in } \Tcal\}$}
        \State $\bcheckmax\leftarrow 0$
        \Comment{Set $\Blb$ and ensure it is non-decreasing right to left}
        \For{$k$ in increasing order among expanded nodes $(h,k)$ at height $h$}
            \If{$\insrfont{is-a-leaf}((h,k),\ucself.\Tcal)$}
                \State $(l,u) \leftarrow \ucself.\ucgetboundsforunexpandednode(h+1, 2k-1)$
                \State $\bcheck \leftarrow \max(\ucself.\flbhk, \ucself.\Blbhk, l)$
            \Else
                \State $\bcheck \leftarrow \max(\ucself.\flbhk, \ucself.\Blbhk,
                                                       \ucself.\Blbhhkk{h+1}{2k-1})$
            \EndIf
            \State $\ucself.\Blbhk \leftarrow \max(\bcheckmax, \bcheck)$
            \State $\bcheckmax \leftarrow \ucself.\Blbhk$
        \EndFor
        \State $\bhatmin\leftarrow 1$
        \Comment{Set $\Bub$ and ensure it is non-increasing left to right}
        \For{$k$ in decreasing order among expanded nodes $(h,k)$ at height $h$}
            \If{$\insrfont{is-a-leaf}((h,k),\ucself.\Tcal)$}
                \State $(l,u) \leftarrow \ucself.\ucgetboundsforunexpandednode(h+1, 2k-1)$
                \State $\bhat \leftarrow \min(\ucself.\fubhk, \ucself.\Bubhk, u)$
            \Else
                \State $\bhat \leftarrow \min(\ucself.\fubhk, \ucself.\Bubhk,
                                                       \ucself.\Bubhhkk{h+1}{2k})$
            \EndIf
            \State $\ucself.\Bubhk \leftarrow \min(\bhatmin, \bhat)$
            \State $\bhatmin \leftarrow \ucself.\Bubhk$
        \EndFor
    \EndFor
\EndPrivMethod

\end{algorithmic}
\caption{$\;$Definitions for the Stochastic Feedback Model with Nonparametric 
Payoffs -- Part IV \label{alg:nonparfbmodelpart4}}
\end{algorithm}
}

\newcommand{\imarrwthree}{1.755in}
\newcommand{\imhspthree}{-0.00in}
\newcommand{\imarrwtwosubfig}{2.31in}
\newcommand{\imarrwtwo}{2.71in}
\newcommand{\imhsptwo}{0.25in}
\newcommand{\imtextspace}{-0.10in}
\newcommand{\imcaptionspace}{-0.05in}

\newcommand{\insertFigTreeIllus}{
\newcommand{\treesimulwidth}{5.5in}
\begin{figure*}
\centering
\includegraphics[width=\treesimulwidth]{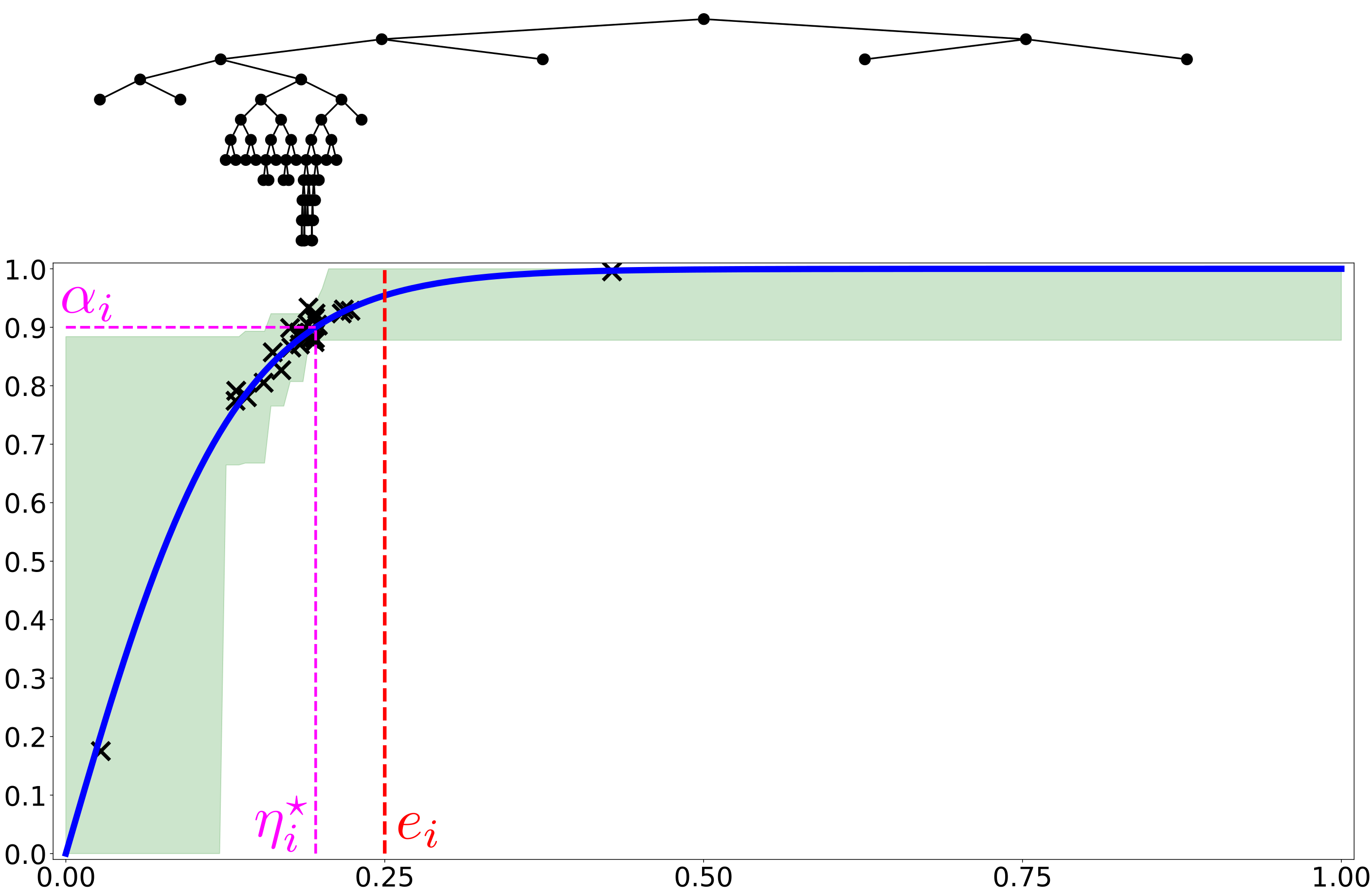}
\\[0.35in]
\includegraphics[width=\treesimulwidth]{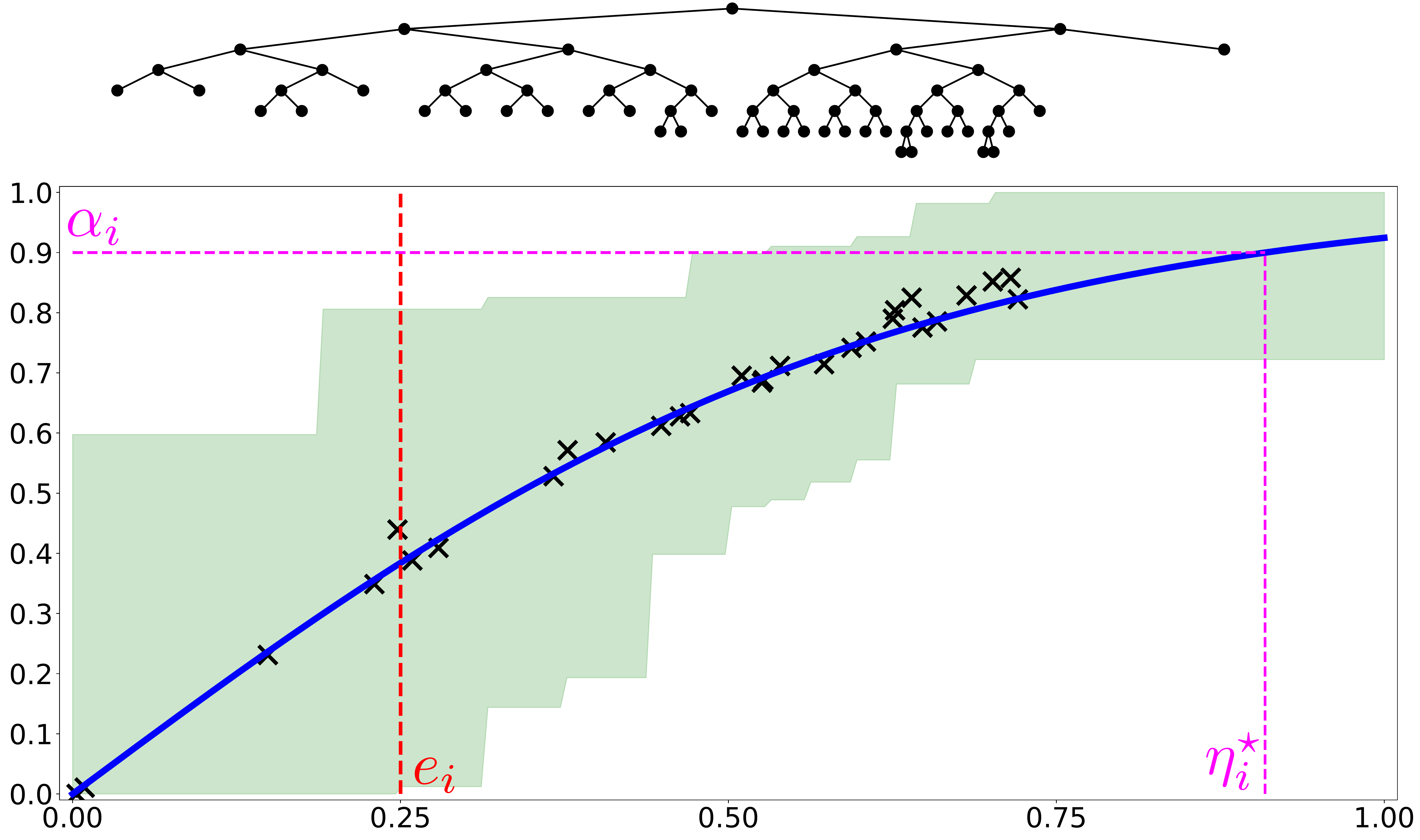}
\caption{%
An illustration of the nonparametric tree-based estimator for two users whose demands
$\udtruei$ are smaller and larger then their entitlement in the top and bottom figures respectively.
Here, $\entitli$ and $\udtruei$ denote the entitlement and unit demand.
The blue curve is the payoff $\payoffi$ and the $\times$'s are the data collected,
i.e. allocation-reward pairs.
The shaded region represents the confidence interval for the payoff, which is
computed using the \ucgetconfinterval{} method in line~\ref{lin:ucgetconfinterval} of
Algorithms~\ref{alg:nonparfbmodelpart1}-\ref{alg:nonparfbmodelpart4}.
In this simulation, we had four users with equal entitlement, and $\threshi=0.9, \volit=1$
for all $i,t$.
\label{fig:treeillus}
}
\vspace{\imtextspace}
\end{figure*}
}

\newcommand{\insertFigPayoffIllus}{
  \begin{wrapfigure}{r}{2.4in} 
    \centering 
    \vspace{-0.38in}
    \includegraphics[width=2.4in]{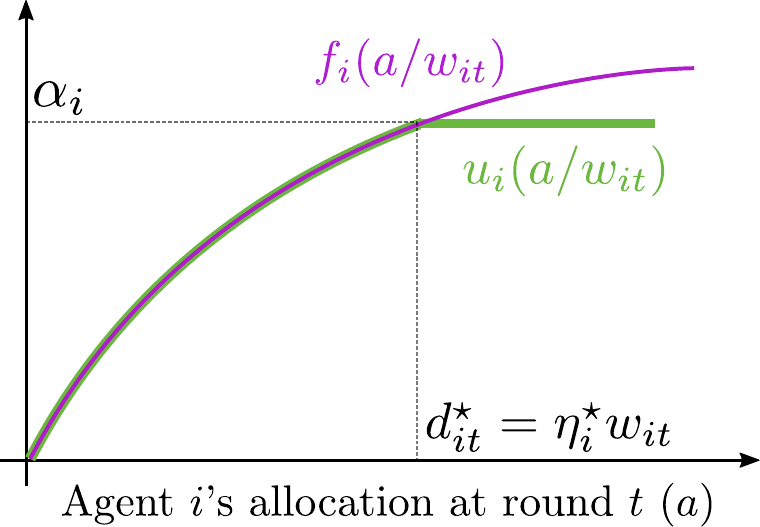}
    \vspace{-0.36in}
  \end{wrapfigure}
}

\newcommand{\insertFigSynResults}{
\begin{figure*}
\centering
\includegraphics[width=\imarrwthree]{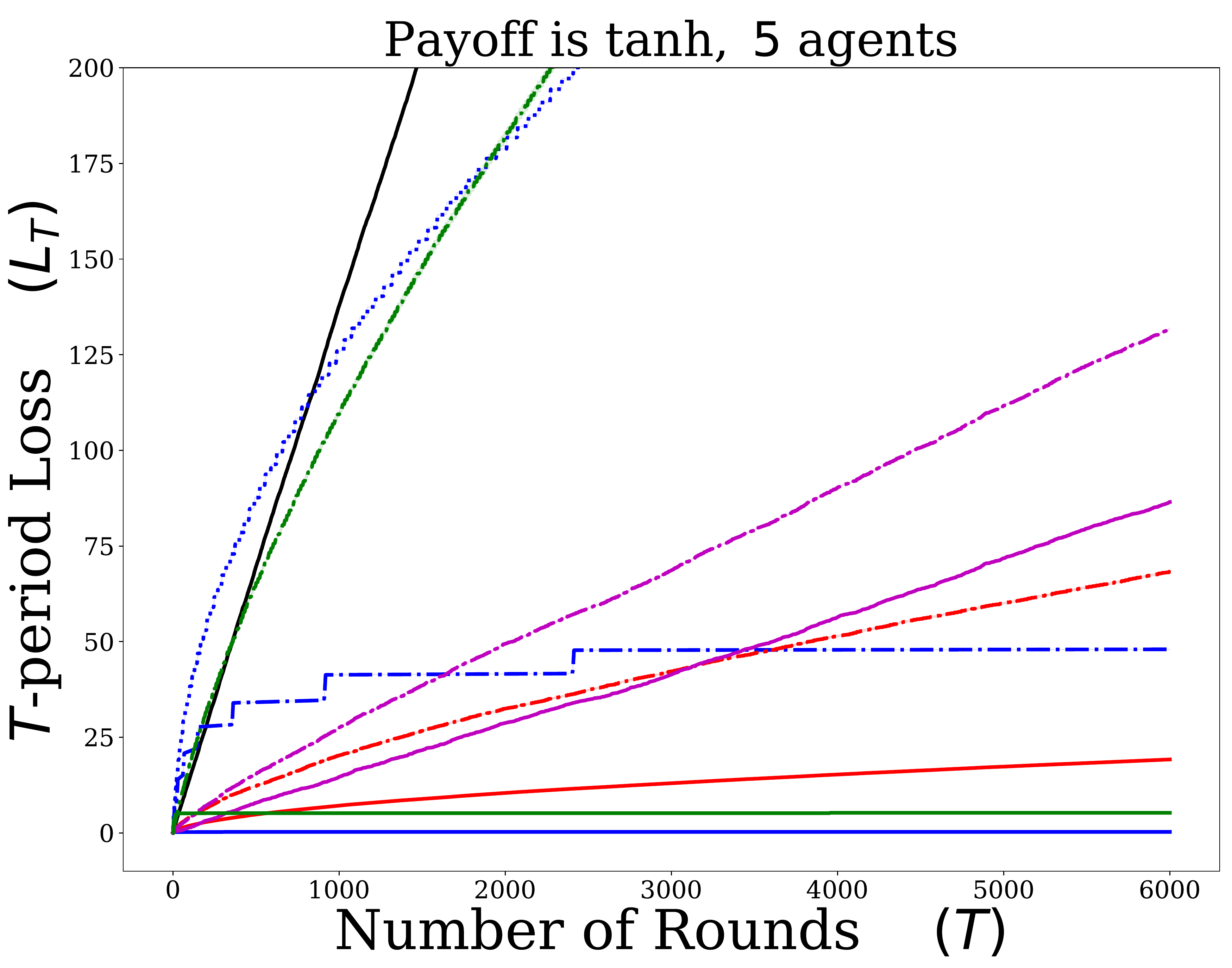}
\hspace{\imhspthree}
\includegraphics[width=\imarrwthree]{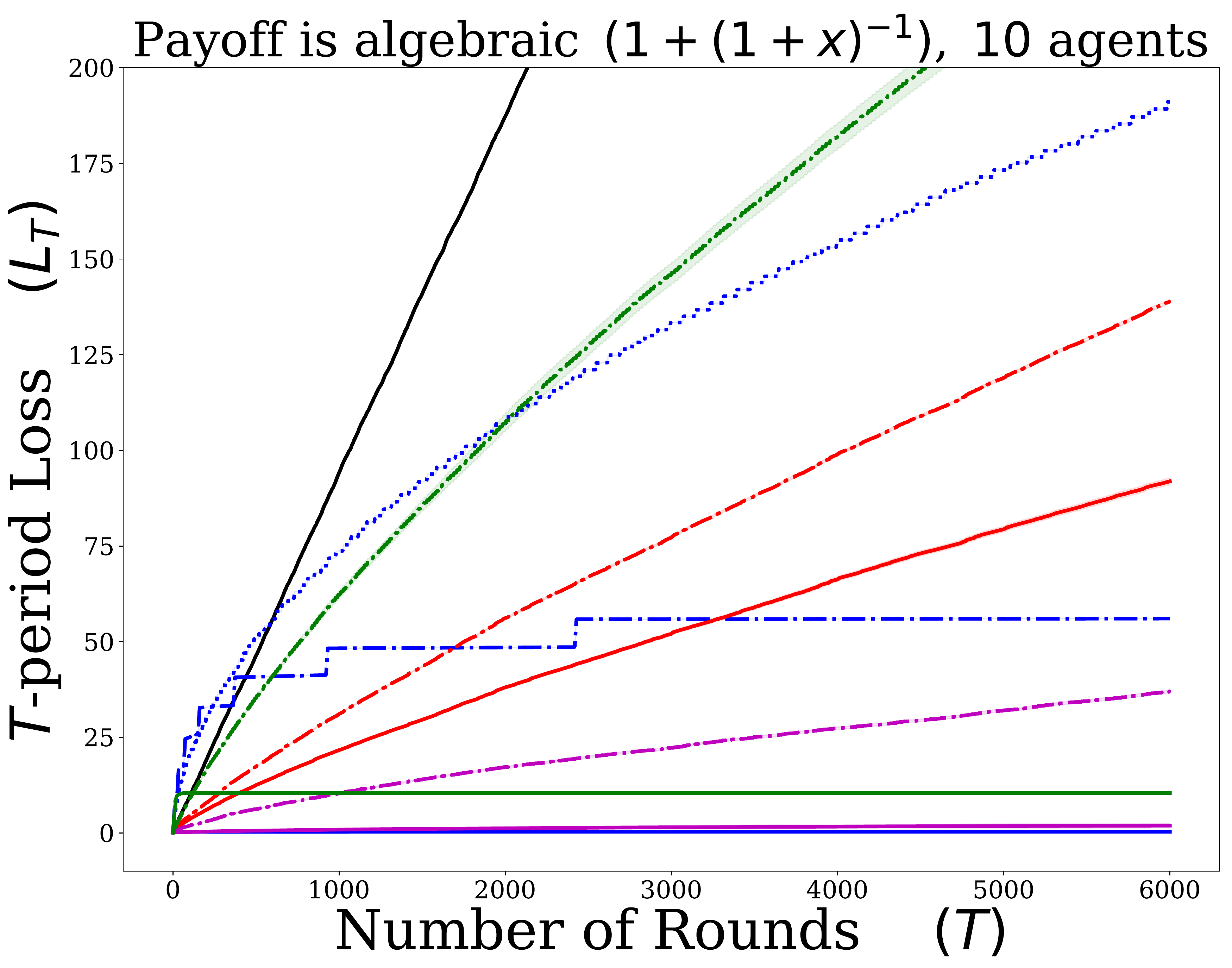}
\hspace{\imhspthree}
\includegraphics[width=\imarrwthree]{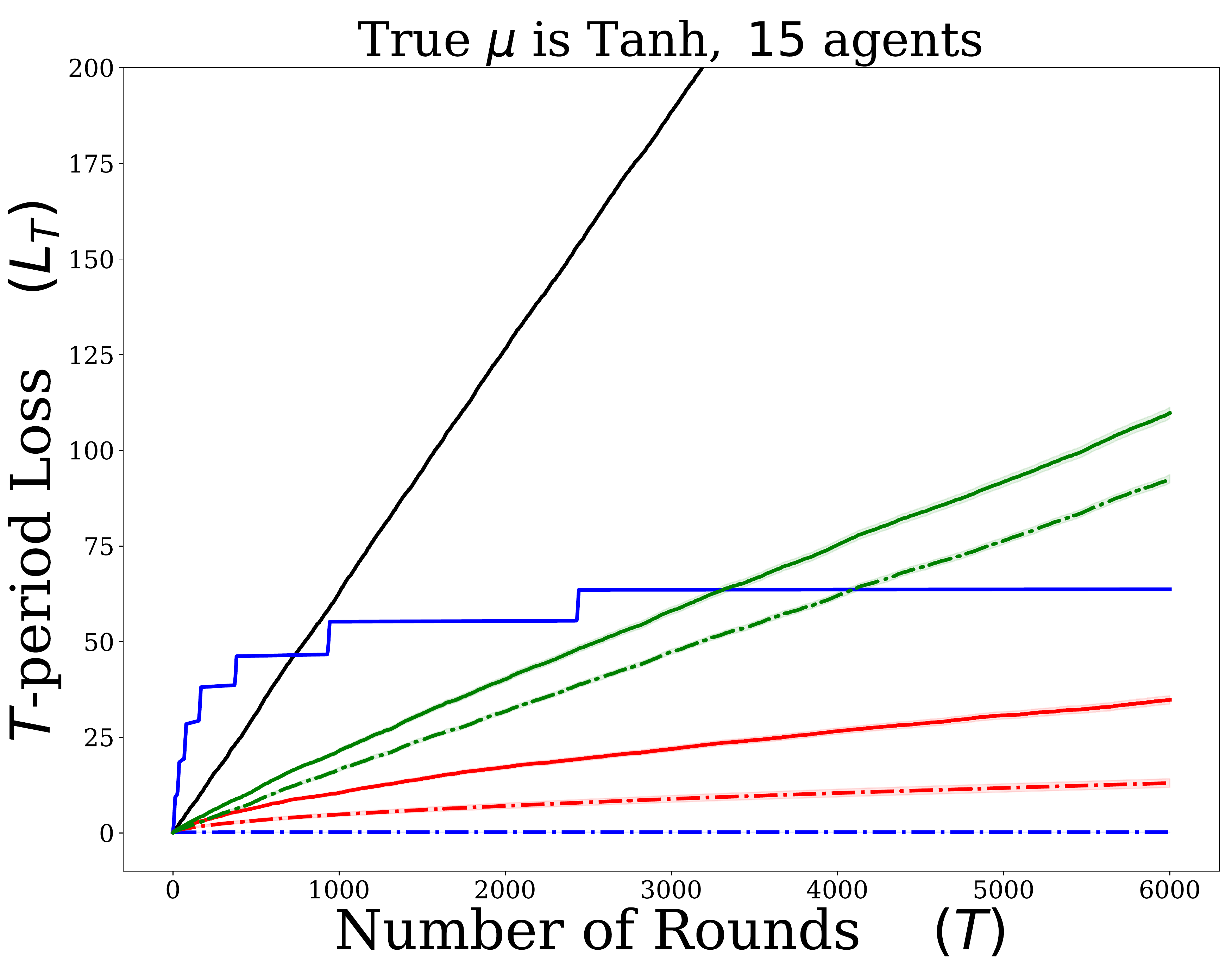}
\vspace{\imcaptionspace}
\caption{
\small
Results on the synthetic experiments.
The rewards are drawn from models where the payoffs are $\tanh$ functions or
algebraic functions of the form $1-(1+x)^{-1}$, as indicated in the title.
We plot the loss on the $y$-axis (lower is better).
All methods were executed for $T=2000$ rounds.
All figures were averaged over 10 runs and the shaded region  (not visible on all curves)
indicates one standard error.
\label{fig:synthetic}
}
\vspace{\imtextspace}
\end{figure*}
}

\newcommand{\insertFigPredServ}{
  \begin{wrapfigure}{r}{2.0in} 
    \centering 
    \vspace{-0.27in}
    \includegraphics[width=1.9in]{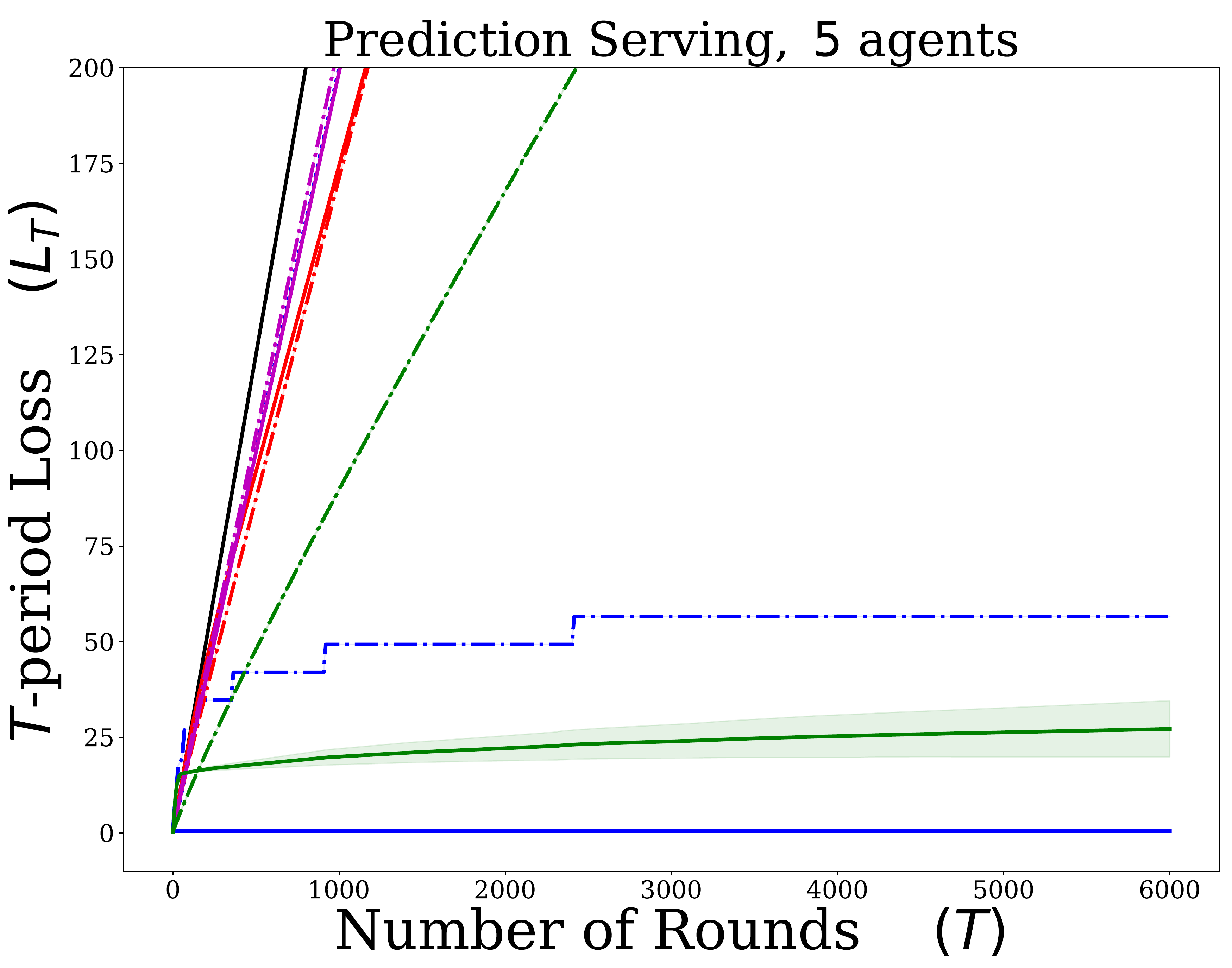}
    \vspace{-0.08in}
    \caption{\small
            Results on the prediction serving task.
            See Figure~\ref{fig:synthetic}
             for more details.
            \label{fig:predserv}}
    \vspace{-0.22in}
  \end{wrapfigure}
}

\newcommand{\insertFigExpResults}{
\begin{figure*}
\centering
\subfigure[]{
    \includegraphics[width=\imarrwtwo]{figs/tanh-5_losses}
    \label{fig:tanh5}
}
\hspace{\imhsptwo}
\subfigure[]{
    \includegraphics[width=\imarrwtwo]{figs/vpi-10_losses}
    \label{fig:vpi10}
}
\\
\subfigure[]{
    \includegraphics[width=\imarrwtwo]{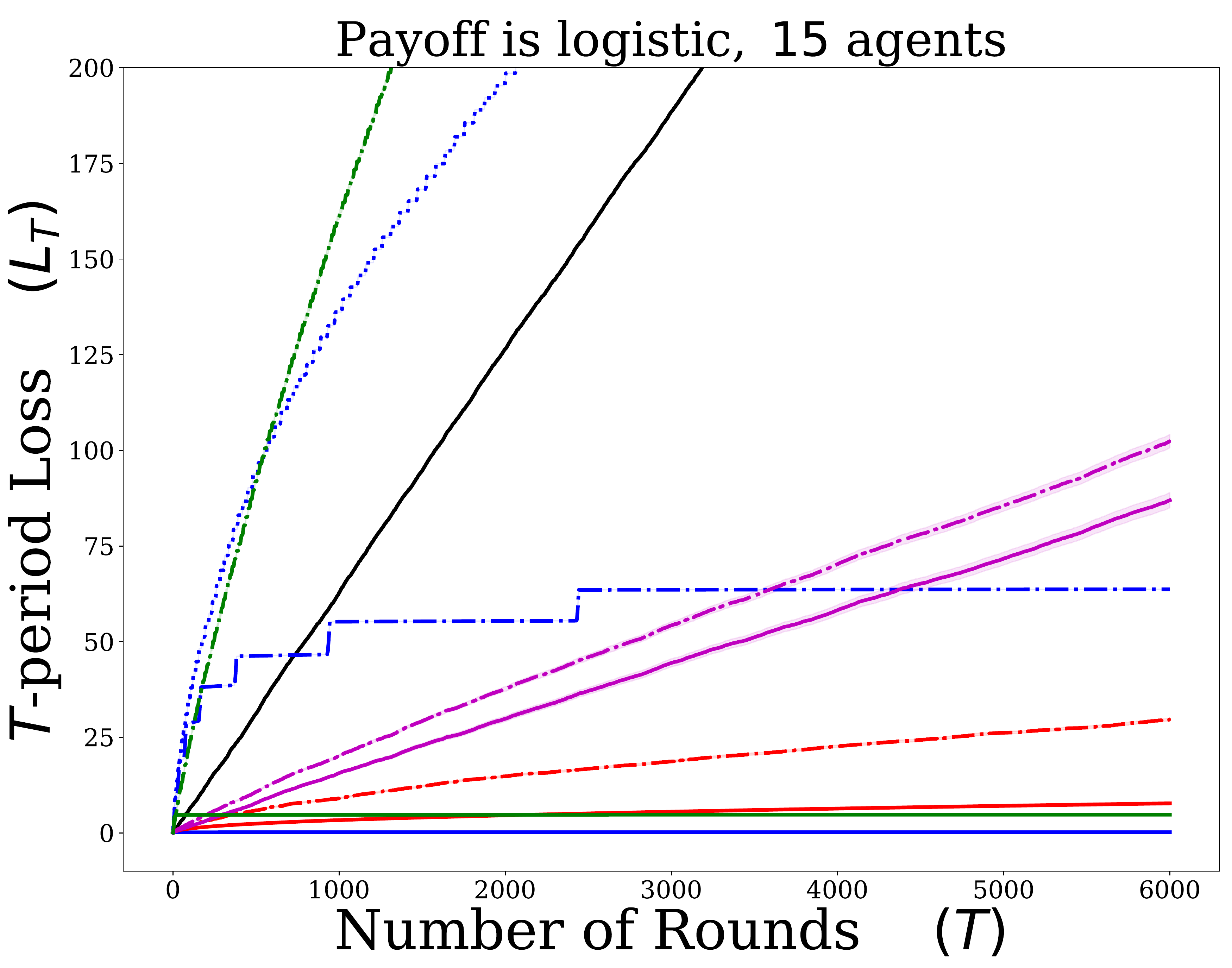}
    \label{fig:logistic15}
}
\hspace{\imhsptwo}
\subfigure[]{
    \includegraphics[width=\imarrwtwo]{figs/predserv-5_losses}
    \label{fig:predserve}
}
\\
\subfigure[]{
    \includegraphics[width=4.7in]{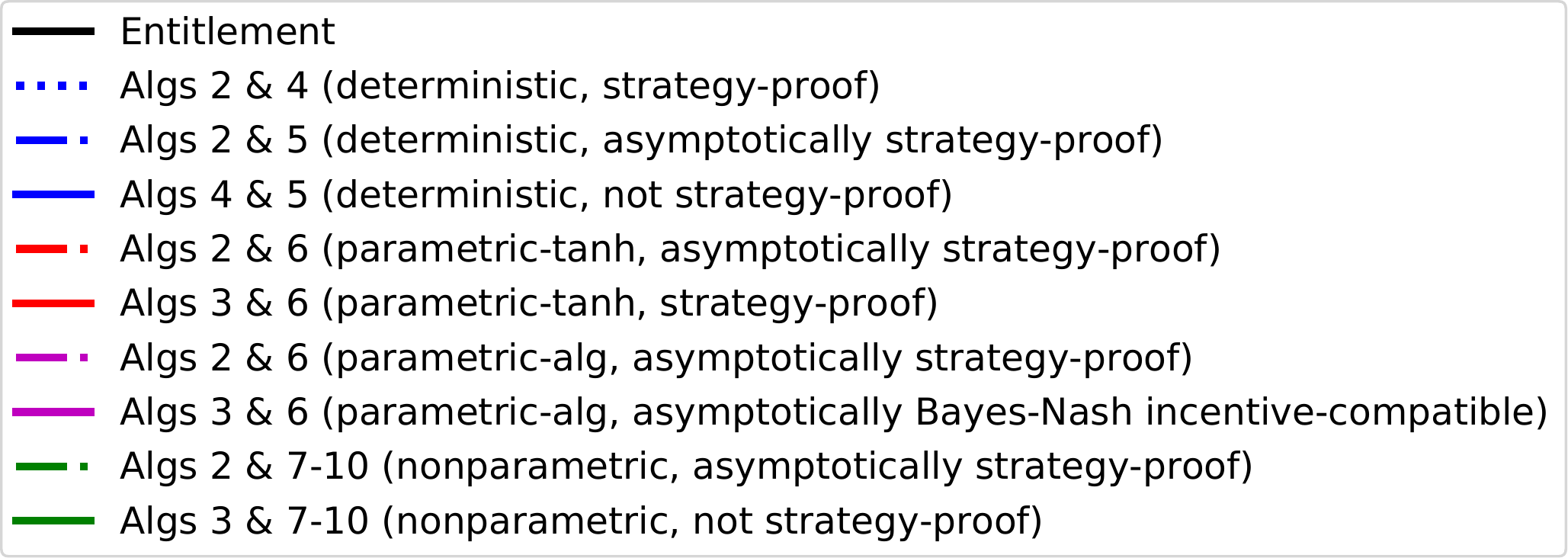}
    \label{fig:legend}
}
\vspace{\imcaptionspace}
\caption{%
Figures~\subref{fig:tanh5}--\subref{fig:logistic15}: Results on the synthetic experiments.
The rewards are drawn from models where the payoffs are $\tanh$ functions,
algebraic functions of the form $1-(1+x)^{-1}$,
or logistic functions, as indicated in the title.
Figure~\subref{fig:predserve}: Results on the prediction-serving task.
Figure~\subref{fig:legend}: Legend for figures~\subref{fig:tanh5}--\subref{fig:predserve}.
In all figures, we plot the loss on the $y$-axis (lower is better).
All figures were averaged over 5 runs and the shaded region  (not visible in most curves)
indicates one standard error.
\label{fig:synthetic}
}
\vspace{\imtextspace}
\end{figure*}
}

\newcommand{\insertFigNTGIllus}{
\begin{figure*}
\centering
\includegraphics[height=1.6in]{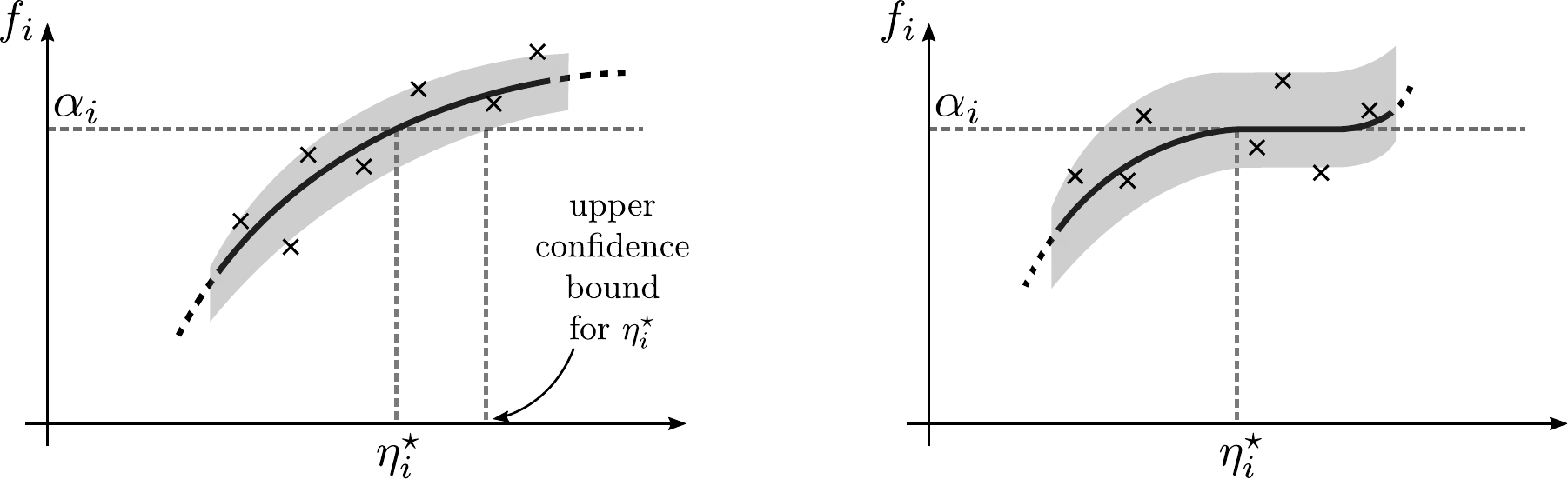}
\caption{
\small
The payoff curve, data ($\times$'s) and confidence intervals (shaded region) for two different
scenarios.
Left: A payoff curve with positive NTG where it is possible to accurately estimate the unit demand
$\udtruei$ when we get more data.
Right: A payoff curve where the NTG is $0$; here, we may not be able to tightly upperbound
$\udtruei$ since the lower confidence bound for $\payoffi$ will necessarily be smaller than
$\threshi$ for some $x>\udtruei$.
\label{fig:ntgillus}
}
\vspace{\imtextspace}
\end{figure*}
}

\begin{abstract}
We describe mechanisms for the allocation of a scarce resource among multiple users
in a way that is efficient, fair, and strategy-proof, but when users do not know their resource
requirements.
The mechanism is repeated for multiple rounds and a user's requirements can change on each round.
At the end of each round, users provide
feedback about the allocation they received, enabling the mechanism to learn user
preferences over time.
Such situations are common in the shared usage of a compute cluster among many users in an
organisation,
where all teams may not precisely know the amount of resources needed to execute their jobs.
By understating their requirements, users will receive less than they need and
consequently not achieve their goals.
By overstating them, they may siphon
away precious resources that could be useful to others in the organisation.
We formalise this task of online learning in fair division
via notions of efficiency, fairness, and strategy-proofness applicable to this setting,
and study this problem under three types of feedback:
when the users' observations are deterministic, when they are stochastic and follow a parametric
model, and when they are stochastic and nonparametric.
We derive mechanisms inspired by the classical max-min fairness procedure that
achieve these requisites, and  quantify the extent to which they are
achieved via asymptotic rates.
We corroborate these insights with an experimental evaluation on synthetic problems
and a web-serving task.

\textbf{Keywords:} Fair division, Mechanism design, Strategy-proofness, Fairness, Online learning
\end{abstract}


\section{Introduction}
\label{sec:intro}

The fair division of a finite resource among a set of rational agents (users) is a well-studied
problem in game theory~\citep{procaccia2013cake}.
In this paper, we study a formalism for fair division, when users have certain resource
\emph{demands};
a user's utility for the amount of resources they receive increases up to this demand, but
does not increase thereafter.
Such use cases arise when sharing computational resources among different applications within a
computer~\citep{linuxfq2020}, or when sharing among different human users in shared-computing
platforms~\citep{verma2015large,psc2020,lbl2020,boutin2014apollo}.
Each user has an entitlement (fair share) to a resource;
however, some users' demands will be smaller than their
entitlement, while some users' demands might be larger.
Hence, allocating the resource simply in proportion to their entitlements will result in unused
resources allocated to the former agents that could be have been allocated to the latter agents
with an unmet demand.
Users may request some amount of resources, which may not necessarily be equal their demand,
 from a mechanism
responsible for allocating this resource among the users.
Fair division is the design of mechanisms for allocating a scarce resource in a way that
is \emph{efficient}, meaning that no resources are left unused when there is an unmet demand,
\emph{fair}, meaning that an agent participating in the mechanism
is at least as happy as when she can only use her fair share,
and
\emph{strategy-proof}, meaning that agents are incentivised to report their true demands when
requesting resources.

Max-min Fairness (\mmf) is one of the most popular mechanisms for fair division that
satisfies the above desiderata.
Since being first introduced in the networking literature~\citep{demers1989analysis},
it has been used in a plethora of applications such as
scheduling data-centre jobs~\citep{chen2018scheduling,ghodsi2013choosy},
load balancing~\citep{nace2008max},
fair queueing in the Linux OS~\citep{linuxfq2020} and packet processing~\citep{ghodsi2012multi},
sharing wireless~\citep{huang2001max} and data-centre~\citep{shieh2011sharing} networks,
and many more~\citep{hahne1991round,li2015energy,liu2013max}.
Moreover, \mmfs and its variants have been implemented in popular open source platforms
such as Hadoop~\citep{hadoopfs}, Spark~\citep{zaharia2010spark}, and
Mesos~\citep{hindman2011mesos}.

To the best of our knowledge,
all mechanisms for fair division in the literature, including \mmf, assume that users know their
demands.
However, in many practical applications, users typically care about achieving a certain
desired level of performance.
While their demands are determined by this performance level, users often have
difficulty in translating their performance requirements to resource requirements
since real world systems can be complex and hard to model~\citep{venkataraman2016ernest}.
In this work, we propose shifting the burden of estimating these demands from
the user to the mechanism, and doing so in a manner that satisfies efficiency, fairness, and
strategy-proofness.
As an example,
consider an organisation where a compute cluster
is shared by users who are serving live web traffic.
Each user wishes to meet a certain service level objective (SLO), such as a given threshold 
on the fraction of queries completed within a specified time limit.
Applying a mechanism such as \mmfs and expecting an efficient allocation requires that all users
precisely know their demands (the amount of resources needed to meet their SLOs).
If a user understates this demand, 
she risks not meeting her own SLOs.
If she instead overstates her demand, she may take away precious resources from other users
who could have used those resources to achieve their SLOs,
resulting in an inefficient allocation.
Prior applied research suggests that while modern data centres operate well below capacity
(usually 30-50\%), most users are unable to meet their
SLOs~\citep{rzadca2020autopilot,delimitrou2013paragon,delimitrou2014quasar}.

In this work, we design a multi-round mechanism for such
instances when agents may not know their demands to satisfy their performance objectives.
At the beginning of each round,
the agents report the load of the traffic they need to serve in that round,
then the mechanism assigns an allocation to each user based on past information
while accounting for the load.
At the end of the
round, agents provide feedback on the allocation (e.g. the extent to which their SLOs were
achieved).
Satisfying strategy-proofness and fairness is more challenging in
this setting.
Since an agent reports feedback on each round, it provides her more
opportunity to
manipulate outcomes than typical settings for fair allocation where she reports a single demand.
Additionally, in order to find an efficient allocation, the mechanism needs to estimate the
demands of all agents; when doing so it risks violating the fairness criterion, especially
for agents whose demands are less than their entitlements.

More generally,
fairness is an important topic  that has garnered
attention in recent times in the machine learning community.
Fairness can be construed in many ways, and this paper studies a concrete instantiation
that arises in resource allocation.
While this topic has been studied in the game theory literature fairly extensively, our paper
focuses on the learning problem when users do not know their resource requirements.
As we will see shortly, in our setting,
fairness and efficiency can be conflicting: an exactly fair allocation can result in worse outcomes
to all agents;
however, by considering a weaker notions of fairness which hold probabilistically and/or
asymptotically, we can achieve outcomes that are
beneficial to everyone.
We believe that many of these ideas can be applied in various other online learning settings where
similar fairness constraints arise.

This manuscript is organised as follows.
In Sections~\ref{sec:fairdivision} and~\ref{sec:mmf},
we briefly review fair division and describe the \mmfs algorithm.
In Sections~\ref{sec:onlinefairdivision} and~\ref{sec:fbmodels}, we formalise 
online learning in fair division and
define notions of efficiency, fairness, and
strategy-proofness which are applicable in this setting.
We also propose three feedback models motivated by practical use cases.
In Section~\ref{sec:methods}, we
describe our mechanisms and present our theoretical results,
quantifying how fast they learn via asymptotic rates.
In Section~\ref{sec:experiments},
we evaluate the proposed methods empirically in synthetic experiments
and a web-serving task.
All proofs are given in the Appendix.

\subsection*{Related Work}

In addition to the many practical applications described above,
the fair allocation of resources
has  inspired a line of theoretical work.
The include mechanisms for dynamically changing demands%
~\citep{freeman2018dynamic,tang2014long,cole2013mechanism},
for allocating multiple resource types%
~\citep{gutman2012fair,parkes2015beyond,ghodsi2011dominant,li2013egalitarian},
and when there is a stream of resources~\citep{aleksandrov2017pure}.
None of these works consider the problem of learning agents' demands when they are unknown.

There is a long line of work in the intersection of online learning and mechanism
design~\citep{amin2013learning,mansour2015bayesian,athey2013efficient,nazerzadeh2008dynamic,%
babaioff2014characterizing}.
The majority of them focus on auction-like settings,
and assume that agents know their
preferences---such as their value for items in an auction---and the goal of the mechanism is to
elicit those preferences truthfully.
Some work has studied instances where the agents do not know their preferences,
but can learn them via repeated participations in a mechanism.
Some examples include~\citet{weed2016online}, where an agent learns to bid in a repeated Vickrey
auction, and~\citet{liu2019competing}, where agents on one side of a matching market learn their
preferences for alternatives on the other side.
In contrast, here, learning happens on the side of the mechanism,
imposing minimal burden on agents who may not be very sophisticated.
Therefore, the onus is on the mechanism to ensure that all agents sufficiently explore all
allocations, while
ensuring that they are incentivised to report their feedback truthfully, so that
the mechanism can learn these preferences.
This is similar to~\citet{kandasamy2020mechanism} who study VCG mechanism design with bandit
feedback where a mechanism
 chooses outcomes and prices for the users; the users in turn provide feedback about the outcomes
which the mechanism uses when determining future outcomes and prices.
The novelty in our work relative to existing literature is our focus
on combining online learning with fair allocation.

\section{Problem Setup}
\label{sec:setup}

In this section, we will first review fair allocation and describe \mmf, one of the most
common methods for fair allocation.
We will then describe the learning problem.

\subsection{Fair Division}
\label{sec:fairdivision}

There are $n$ agents sharing a resource of size $1$.
Agent $i$ has an entitlement $\entitli$ to the resource, where $\entitli>0$ and
$\sum_{i=1}^n\entitli=1$.
At any given instant, let $\demtruei$ denote the true demand for user $i$,
where $\demtruei \geq 0$ for all users $i$.
Continuing with the example from Section~\ref{sec:intro},
this resource could be a compute cluster shared by $n$ users in an organisation.
The entitlements are set by the management depending on whether a user's workload consists
of time-sensitive live traffic or offline job processing, and $\demtruei$ is the amount of the
resource user $i$ needs to achieve her SLOs.
In some use cases, $\entitli$ may represent the contribution of each
agent to a federated system,  such as in universities where it could be set based on the
contribution of each research group to purchase a cluster.

In a mechanism for fair allocation, each agent reports their demand $\demi$ 
(not necessarily truthfully) to the mechanism;
the mechanism returns an allocation vector $\alloc\in\RR_+^n$, where $\alloci$ is the amount of the
resource allocated to agent $i$.
Here, $\sum_{i=1}^n \alloc_i \leq 1$.
Let $\utili:\RR_+\rightarrow\RR_+$ denote agent $i$'s utility function where $\utili(a)$ is the
value agent $i$ derives if the mechanism allocates an amount $a$ of the resource to her.
She wishes to receive resources up to
her demand, but has no value for resources beyond her demand;
i.e. her utility $\utili$ is strictly increasing up to $\demi$,
but $\utili(a) = \utili(\demi)$ for all $a>\demi$.

The literature on fair division typically considers three desiderata for a mechanism:
\emph{(i)} efficiency,
\emph{(ii)} fairness,
\emph{(iii)} strategy-proofness.
\emph{Efficiency} means that there are no unused resources when there is an unmet demand.
To define it formally,
let $\lotur, \lotor, \lotud$ be as defined below for given $\dem,\alloc\in\RR_+^n$:
\begin{align*}
\lotur(\alloc) = 1-\sum\nolimits_{i=1}^n \alloc_i,
\hspace{0.2in}
\lotor(\dem, \alloc) = \sum\nolimits_{i=1}^n (\alloc_i - \dem_i)^+,
\hspace{0.2in}
\lotud(\dem, \alloc) = \sum\nolimits_{i=1}^n (\dem_i  - \alloc_i)^+.
\label{eqn:lotdefn}
\numberthis
\end{align*}
Here, $y^+=\max(y, 0)$.
If $\demtrue\in\RR_+^n$ is the vector of true demands and $\alloc\in\RR_+^n$ is a vector
of allocations output by a mechanism,
then $\lotur(\alloc)$ is the amount of unallocated resources in this instance,
$\lotor(\demtrue,\alloc)$ is the sum of over-allocated resources (allocated over a
user's demand) and
$\lotud(\demtrue,\alloc)$ is the sum of unmet demands (allocated under a user's demand). 
Then, $\lotur(\alloc) + \lotor(\demtrue,\alloc)$ is the total amount of resources that are not
being used, and
$\lot(\demtrue,\alloc) = \min(\lotur(\alloc) + \lotor(\demtrue,\alloc), \lotud(\demtrue,\alloc))$
is the amount of resources that are not being used but could have been used to improve the utility
of some agent.
A mechanism is efficient if, for all $\demtrue$, and when all agents report their true
demands,  $\;\lot(\demtrue,\alloc) = 0$.
An efficient mechanism is Pareto-optimal, in that one user's
utility can be increased only by decreasing the utility of another.
Next, a mechanism is \emph{fair} (also known as sharing incentive or individual rationality)
if the utility a truthful user derives from an allocation is at least as
much as if she had been allocated her entitlement;
i.e., for all $i, a, d$,  $\utili(\alloc_i) \geq \utili(\entitli)$;
recall that the allocation an agent receives depends on the demands reported by the other agents.
Finally, a mechanism is \emph{strategy-proof} if no agent benefits by misreporting their demands.
That is, consider any agent $i$ and fix the demands reported by all other agents.
Let $\alloc^\star$ be the allocation returned by the mechanism when agent $i$
reports her true demand $\demtruei$,
and $\alloc$ be the allocation vector when she reports any other demand $\dem$.
Strategy-proofness means that for any agent $i$ and for all reported demands
from other agents, $\utili(\alloc_i^\star) \geq \utili(\alloc_i)$.

Observe that the above formalism does not assume that agents' utilities are comparable,
i.e. we make no interpersonal comparisons of utility~\citep{hammond1990interpersonal}.
The utilities
are used solely to specify an agent's preferences over different allocations.
Therefore, notions such as utilitarian welfare which accumulates the utilities of all agents are not
meaningful in this setting.

\subsection{Max-min Fairness}
\label{sec:mmf}
\insertAlgoMMF
Algorithm~\ref{alg:mmf} describes max-min fairness, a popular mechanism for fair division.
First,
it allocates the demands to users whose demands are small relative to their entitlement;
for all other agents whose demands cannot be satisfied simultaneously,
it allocates in proportion to their entitlements.
The following theorem shows that \mmfs satisfies the above desiderata;
its proof is given in Appendix~\ref{sec:pfmmf}, where we have also established other useful
properties of \mmf.

\insertprethmspacing
\begin{theorem}
\label{thm:mmf}
\emph{\mmf} (Algorithm~\ref{alg:mmf}) is efficient, fair, and strategy-proof.
\end{theorem}

As an example, consider $4$ users with equal entitlements and true
demands $\{0.1, 0.28, 0.4, 0.5\}$.
\mmfs returns the allocation $\{0.1, 0.28, 0.31, 0.31\}$.
Instead, had we allocated the resources equally according to their entitlements,
i.e. $\{0.25, 0.25, 0.25, 0.25\}$, agent 2 will not have met her demand,
agents 3 and 4 will not have received as much, while agent 1 will have been
sitting on 0.15 of the resource---this is fair and trivially strategy-proof, but inefficient.

\subsection{Online Estimation of Demands}
\label{sec:onlinefairdivision}

We begin
our formalism for the online learning version of the fair division problem with
a description of the environment.  
We consider a multi-round setting, where,
on round $t$, agent $i$ faces a load $\volit$ with (unknown)
demand $\demtrueit$.
In order to be able to effectively learn, we need a form of 
feedback for each agent which informs us of the agent's utility, and additionally be able to
relate the loads from different time steps.
To this end,
we define an (unknown) \emph{payoff} function $\payoffi:\RR_+\rightarrow \RR$ for each agent
$i$, which will characterise the agent's utility function.
When user $i$ receives an allocation $\allocit$ on round $t$, she observes a reward $\Xit$.
We will consider different feedback models where these rewards can be deterministic,
in which case $\Xit=\payoffi(\allocit/\volit)$, or they can be stochastic,
in which case $\Xit$
is drawn from a $\sigmait$ sub-Gaussian distribution with mean $\payoffi(\allocit/\volit)$.
%

The payoff $\payoffi$ is a \emph{non-decreasing} function of the amount of resources
allocated per unit load.
The agent wishes to achieve a certain (known) threshold payoff $\threshi$ and hence,
her true demand at round
$t$ is given by $\demtrueit = \volit\payoffinvi(\threshi)$.
This implies that each agent's true demand increases proportionally with the load, with
the demand per unit load, or \emph{unit demand} for short, being $\payoffinvi(\threshi)=:\udtruei$.
We will assume that $\udtruei\in[0, \udmax]$, where $\udmax$ is known.

\insertFigPayoffIllus
The agent's utility $\utili:[0, \udmax]\rightarrow\RR_+$, as illustrated in the figure to the right,
is also a function of the amount of resources allocated per unit load which increases
\emph{strictly} up to
$\udtruei$, and does not increase beyond $\udtruei$.
In the figure, we have shown  $\payoffi(a) = \utili(a)$ for $a<\udtruei$,
although this is not
necessary---we only require that $\utili$ be increasing up to $\udtruei$ and remain flat thereafter.
Additionally, we will assume that $\utili$ is Lipschitz continuous with Lipschitz constant $\Lipi$.

We will make some mild assumptions so as to avoid degenerate cases in our analysis.
First, we will assume that both $\{\volit\}_{i,t}$ and $\{\sigmait\}_{i,t}$ are fixed sequences.
Second, for all $i, t$,
$\volit\in (\volmin, \volmax]$ and $\sigmait \in (\sigmamin, \sigmamax]$ for some
$\volmin, \sigmamin>0$ and $\volmax, \sigmamax<\infty$; $\volmin,\volmax,\sigmamin,\sigmamax$ need
not be known.
We comment more on the bounded-from-below conditions on $\volit,\sigmait$ in our proofs.
Third, we will assume $\volmax\,\udmax \leq 1$, which states that at the very least each
agent should be able to meet
their demands if they have the entire resource to themselves;
this assumption can also be relaxed, and we will comment further in our proofs.
The following example helps us motivate the above formalism.

\insertprethmspacing
\begin{example}[Web serving]
\label{eg:webserving}
\emph{
Continuing with the example from Section~\ref{sec:intro}, say that the SLO of each team
is to ensure that a given threshold, say $0.95$, of the queries are completed on time on average.
The load $\volit$ is the number of queries agent $i$ receives in round $t$.
Her reward $\Xit\in[0, 1]$ is the fraction of queries completed on time.
The probability each query will succeed increases with the amount of resources allocated per query;
specifically, the success of each query is a Bernoulli event with probability
$\payoffi(\allocit/\volit)$.
Therefore, $\payoffi(\allocit/\volit)=\EE[\Xit]$ denotes the expected fraction of queries completed
on time.
The utility of the agent $\utili = \min(\payoffi, 0.95)$ increases with this expected fraction but
is capped at $0.95$. Hence, her demand is $\demtrueit = \volit\payoffinvi(0.95)$.
Finally, since $\Xit$ is a sum of $\volit$ bounded random variables, it is
$\sigmait=1/(2 \sqrt{\volit})$ sub-Gaussian.
In web services, it is common to set SLOs via such thresholds since the amount of resources needed
to complete all queries on time could be possibly infinite. Moreover, since the quality of the
overall  service is usually bottlenecked by external
factors~\citep{mogul2019nines}, there is
little value to exceeding such a threshold.
}
\end{example}

In an online mechanism for fair allocation with unknown demands, each agent states their threshold
$\threshi$ ahead of time.
At the beginning of each round,
each agent reports their load $\volit$ to the mechanism, then the mechanism returns an allocation
vector,
and at the end of the round each agent reports their reward $\Xit$
back to the mechanism.
The mechanism should use this feedback to estimate demands in an online fashion and quickly converge
to an efficient ellocation in a manner that is fair and strategy-proof.
As we will see, achieving these desiderata exactly is quite challenging in our setting and hence
we will define asymptotic variants to make the problem tractable.

\textbf{Efficiency:}
We define the \emph{loss} $\LOTT$ to be the sum of resources left on the table over $T$ rounds:
\begin{align*}
\LOTT := \sum\nolimits_{t=1}^T \lot(\demtrueit, \allocit),
\hspace{0.25in}
\text{where }
\lot(\demtrueit,\allocit) \hspace{-0.02in}=\hspace{-0.02in} \min(\lotur(\allocit)
\hspace{-0.02in}+\hspace{-0.02in} \lotor(\demtrueit,\allocit),
\lotud(\demtrueit,\allocit)).
\numberthis
\label{eqn:LOTTdefn}
\end{align*}
Recall the definitions of $\lotur,\lotor,\lotud$ from~\eqref{eqn:lotdefn}.
A mechanism is \emph{asymptotically efficient} if, when all agents are reporting truthfully,
$\LOTT\in\littleO(T)$.
We will say that a mechanism is \emph{probably asymptotically efficient} if this holds with
probability at least $1-\delta$, where $\delta\in(0,1)$ is pre-specified.

\textbf{Fairness:}
Let $\UiT,\UeiiT$, respectively be the sum of an agent's utilities when
she participates in the mechanism truthfully for $T$ rounds, and when
she has her entitlement to herself.
Precisely,
\begin{align*}
\numberthis
\label{eqn:UiT}
\UiT = \sum\nolimits_{t=1}^T \utili\left(\allocit/\volit\right),
\hspace{0.7in}
\UeiiT = \sum\nolimits_{t=1}^T \utili\left(\entitli/\volit\right),
\end{align*}
A mechanism is \emph{asymptotically fair} if, $\UeiiT - \UiT \in \littleO(T)$.
Similarly, a mechanism is \emph{probably asymptotically fair} if  $\UeiiT - \UiT \in \littleO(T)$
with probability at least $1-\delta$, where $\delta\in(0,1)$ is pre-specified.
%
To understand why we consider an asymptotic version of fairness,
observe that achieving an efficient allocation
requires that the mechanism accurately estimates the demands.
This is especially the case for agents whose demand is lower than their entitlement since the excess
resources can be allocated to other users who might need them.
However, in doing so, it invariably risks allocating less than the demand and consequently violating
fairness.
Asymptotic fairness means that these violations will vanish over time.

\textbf{Strategy-proofness:}
We will inroduce a variety of strategy-proofnes definitions in this paper.
To define them,
let $\UiT$ be as defined in~\eqref{eqn:UiT}.
Let $\pi$ be an arbitrary (non-truthful) policy that an agent may follow and let $\UpiiT$ be the sum
of utilities when following this policy.
A mechanism is \emph{strategy-proof} if $\UpiiT - \UiT \leq 0$ regardless of the behaviour of
the other agents.
A mechanism is \emph{probably strategy-proof} if the same holds  with
probability at least $1-\delta$, where $\delta\in(0,1)$ is given.
A mechanism is \emph{asymptotically strategy-proof} if $\UpiiT - \UiT \in \littleO(T)$
regardless of the behaviour of the other agents.
It is \emph{probably asymptotically strategy-proof} if this holds with
probability at least $1-\delta$.
A mechanism is \emph{probably asymptotically Bayes-Nash incentive-compatible} if
$\UpiiT - \UiT \in \littleO(T)$ with probability at least $1-\delta$ when all other agents are
being truthful.
In the above definition, an agent may adopt a non-truthful policy $\pi$ by
misreporting her threshold $\threshi$ at the beginning, or by misreporting
her load $\volit$ or the reward $\Xit$ on any round $t$.
For example, with the intention of getting more resources,
she could inflate the mechanism's estimate of her demand by
overstating her threshold or load, or understating her reward.
An agent could also
 be strategic over multiple rounds by adaptively using the information she gained when reporting
her loads and rewards. 


We focus on the above notions of incentive-compatibility since 
achieving dominant-strategy incentive-compatibility can be challenging
in multi-round
mechanisms~\citep{babaioff2014characterizing,babaioff2013multi,kandasamy2020mechanism}.
Typically, authors circumvent this difficulty by adopting Bayesian notions which assume
that agent values are drawn from \emph{known} prior beliefs.
However, such Bayesian assumptions may not be realistic, and even if it were, the prior beliefs
may not be known in practice~\citep{schummer2004almost}.
In this work, we rely on definitions of incentive-compatibility which hold with high probabilty
and/or asymptotically in order to make the problem tractable.
If a fair allocation mechanism is asymptotically strategy-proof,
the maximum utility an agent may gain by
not being truthful vanishes over time.
In many use cases, it is reasonable to assume that agents would be truthful if the benefit
of deviating is negligible.
Prior work has similarly explored different concepts of approximate
incentive-compatibility in various mechanism design
problems when achieving dominant truthfulness is not
possile~\citep{lipton2003playing,kojima2010incentives,roberts1976incentives,feder2007approximating,daskalakis2006note}.
Moreover,~\citet{kandasamy2020mechanism} and~\citet{nazerzadeh2008dynamic} study asymptotic
strategy-proofness when learning in auctions.

Finally, we note that in some use cases, we may wish to learn without any strategy-proofness
constraints while still satisfying fairness.
For instance, in Example~\ref{eg:webserving}, the SLOs may be set by the management, while an
organisation's central monitoring system might be able to directly observe the load and rewards.
Therefore, we will also consider mechanisms for online learning without any strategy-proofness
guarantees.
In particular, we find that when we relax the strategy-proofness requirements, the rates for
efficiency and fairness improve.


\subsection{Feedback Models}
\label{sec:fbmodels}

All that is left to do to complete the problem set up is to specify the feedback model,
i.e., a form for the payoffs $\{\payoffi\}_i$ and the (probabilistic) model for the rewards $\Xit$.
In this paper, we will consider the following three models.

\begin{enumerate}
\item \textbf{Deterministic Feedback:}
Our first model is the simplest of the three:
all agents deterministically observe the payoff for their allocation, i.e.,
$\Xit = \payoffi(\allocit/\volit)$ on all rounds.

\item \textbf{Stochastic Feedback with Parametric Payoffs:}
Here, the rewards are stochastic, where $\Xit$ has expectation
$\payoffi(\allocit/\volit)$ and is $\sigmait$ sub-Gaussian. 
For all users, $\payoffi$ has parametric form $\payoffi(a) = \payoffthetatruei(a)
=\mu(a\thetatruei)$,
where $\thetatruei>0$ is an unknown user-specific parameter and
$\mu:\RR_+\rightarrow\RR$ is a known increasing function.

\item \textbf{Stochastic Feedback with Nonparametric Payoffs:}
Here, $\Xit$ has expectation
$\payoffi(\allocit/\volit)$ and is $\sigmait$ sub-Gaussian. 
Moreover,
for all users $i$, $\payoffi$ is $(\Lf/\udmax)$--Lipschitz continuous, where
$\Lf$ is known.
That is,
\begin{align*}
\forall\,a_1, a_2\in [0,\udmax],\quad
\left|f(a_1) - f(a_2)\right| \leq 
            \; \frac{\Lf}{\udmax} \left|a_1 - a_2\right|.
\numberthis\label{eqn:nonparfbmodel}
\end{align*}
\end{enumerate}


There are no additional assumptions required for the first model.
For the second model,
we chose a parametric family of the above form since it is a straightforward way to model increasing
functions, as necessitated by our problem set up.
In Example~\ref{eg:webserving}, an appropriate choice for $\mu$ could be $\mu(x)=\tanh(x)$
or $\mu(x) = 1 - (1+x)^{-1}$, which are bounded and increasing.
Moreover, since they are concave, they can model situations which exhibit a diminishing
returns property when more resources are allocated, as is commonly the case in
practice~\citep{venkataraman2016ernest}.
In the second model, we will assume that we know a lower bound on the parameter $\thetatruei$
and the derivative of $\mu$.
Such regularity conditions are common in parametric models in the online learning
literature~\citep{filippi2010parametric,li2017provably,chaudhuri2015convergence}. 
This is stated formally in Assumption~\ref{asm:glm}.

\insertprethmspacing
\begin{assumption}
\label{asm:glm}
There exists known $\thetamin>0$ such that 
$\thetatruei\geq\thetamin$ for all $i$.
Moreover,
\[
\kappamudot \defeq \inf_{x\in[0,\udmax]} \frac{\ud\mu(x)}{\ud x} \geq 0.
\]
\end{assumption}

In the third model,
We will additionally require that for each user, the payoff increases sharply at
her demand, which we formalise by defining the near-threshold gradient (NTG)
of a non-decreasing function at a given point.

\insertprethmspacing
\begin{definition}[Near-threshold Gradient]
\label{def:ntg}
The near threshold gradient of a non-decreasing function
$f:\RR\rightarrow \RR$ at $\eta\in\RR$ with
$f(\eta)=\alpha$ is defined as,
\emph{
\begin{align*}
{\rm NTG}(f, \eta) &\defeq
\sup\bigg\{G \geq 0; \; \exists \epsilon>0 \text{ such that, }\;
\forall\,a\in(\eta-\epsilon, \eta+\epsilon),
|f(x)-\alpha| \geq \,\frac{G}{\udmax}|a-\eta| \bigg\}.
\end{align*}
}
Moreover, for $G \geq {\rm NTG}(f, \eta)$, define $\epsG$ as follows:
\emph{
\begin{align*}
\epsG &\defeq
\sup\bigg\{\epsilon \geq 0; \;
\forall\,a\in(\eta-\epsilon, \eta+\epsilon),
|f(x)-\alpha| \geq \,\frac{G}{\udmax}|a-\eta| \bigg\}.
\end{align*}
}
\end{definition}

\insertFigNTGIllus

If $f$ is differentiable at $\eta$, then 
${\rm NTG}(f, \eta) = f'(\eta)$.
Assumption~\ref{asm:ntg} below states that the NTG should be positive for all payoffs at their
respective demands.
To understand why such an assumption is necessary, assume that $\payoffi$ increases up to
the demand but is flat thereafter.
Then, any lower confidence bound for the payoff at any $a>\udtruei$ will necessarily be smaller
than $\threshi$ and
therefore, we will not be able to estimate this demand with any confidence.
We have illustrated this in Figure~\ref{fig:ntgillus}.

\insertprethmspacing
\begin{assumption}
\label{asm:ntg}
There exists $\ntg>0$ such that,
for all users $i$, ${\rm NTG}(\payoffi, \threshi) \geq \ntg$.
\end{assumption}

In this paper, we study the above models in order to delineate what can be achieved under
various assumptions and under various strategy-proofness constraints (see
Table~\ref{tb:resultssummary}).
In real world settings, the deterministic model can be unsuitable since feedback can be noisy. 
Empirically, we find that the nonparametric model most meaningfully reflects practical settings,
and consequently the corresponding algorithms outperform its parametric counterparts.
However, the proof intuitions are simpler to understand in the
deterministic and parametric models.

This completes the formulation of online learning in fair allocation.
Next, we present our mechanisms.

\insertAlgoMMFLearnSP
\insertAlgoMMFLearnNSP

\section{Mechanisms \& Theoretical Results}
\label{sec:methods}

Our mechanisms are outlined in Algorithms~\ref{alg:mmflearnsp} and~\ref{alg:mmflearnnsp}
where we have hidden the round index subscript (e.g. $t$) for simplicity.
In order to present a general framework for all models, we have abstracted out some of
components of the mechanism.
The most important of these is the definition of a user class \userclasss whose definition
depends on the feedback model.
Its main purpose is to estimate the unit demands of each user from past data.
\userclasss is instantiated for each user in the system.

Algorithm~\ref{alg:mmflearnsp} outlines a framework for online learning in this
environment, which, generally speaking, provides stronger strategy-proofness guarantees.
It operates over a sequence of brackets, indexed by $q$.
Each bracket begins with an exploration phase, during which the mechanism tries different
allocations for each user, collects the observed rewards, and reports them to the
\userclass.
Then, on round $t$, it computes \emph{upper (confidence) bounds} $\{\udubit\}_{i\in[n]}$ for the
unit loads
$\{\udtruei\}_{i\in[n]}$ of all agents, based on data collected only from
previous exploration phases;
for this, it uses the \ucgetudubs method of \userclass.
This is then followed by $r'(q)$ rounds during which the mechanism chooses the allocation
by invoking \mmfs as a subroutine;
when doing so, it sets the demand of agent $i$ to be $\volit\times\udubit$.
The number of exploration rounds remains fixed while $r'(q)$ increases with the
bracket index $q$.
In addition to the entitlements and the user class definition, the algorithm accepts
two function definitions as inputs which depend on the feedback model:
a function $r'$ which specifies the length of the second phase for each bracket
and a function \explphases which implements the exploration phase for the model.

By using an upper confidence bound as the reported demand, the mechanism errs on
the side of caution.
Choosing these allocations conservatively\footnote{%
This can be contrasted with optimistic methods in the bandit literature which may take actions
optimistically.}
is necessary to ensure both strategy-proofness and fairness.
Intuitively, if the mechanism promises not to under-estimate a user's demands, there is less incentive
for the user to inflate their resource requirements by not being truthful.
Similarly,
we also show that fairness is less likely to be violated if we do not
under-estimate user demands in \mmf.
As we will see, being conservative comes at a loss in efficiency,
reflected via worse rates for the loss.

\insertTableResultsSummary

Algorithm~\ref{alg:mmflearnnsp} outlines a framework for online learning which, generally
speaking, provides weaker strategy-proofness guarantees, but stronger efficiency and
fairness guarantees than Algorithm~\ref{alg:mmflearnsp}.
In round $1$, it allocates in proportion to each agent's entitlements.
In each subsequent round $t$, it computes recommendations $\{\udit\}_{i\in[n]}$ 
for the unit demands using data from previous rounds; 
for this, it invokes the \ucgetudrecs method of \userclass.
It then determines the allocations via \mmfs by setting the demand of agent $i$
to be $\volit\times\udit$.
Finally it collects and records the rewards and sub-gaussian constants (SGCs) from all agents.

Table~\ref{tb:resultssummary} summarises the main theoretical results in this paper.
For each of the three feedback models, we have presented the rates for
asymptotic efficiency, fairness, and strategy-proofness when using either
Algorithm~\ref{alg:mmflearnsp} or~\ref{alg:mmflearnnsp} along with model-specific definitions
for \userclass, $r'$, and \explphases which we will describe shortly.

In the next three subsections,
we present the algorithms for the three feedback models in detail.
When describing \userclass, we will use OOP-style pseudocode.
Therefore, \algorithmicmethod{} refers to a function that can be called using an instance of this
class, while \algorithmicprivmethod{} refers to a function that can only be called by other 
methods of the instant.
An \AlgAttribute{} refers to a variable belonging to the instant that can be accessed by all
methods.
When referring to an attribute or method from within a class, we will use the prefix \ucself.
Finally, \algorithmicfunction{} refers to function that can be defined and called without a class.
For brevity, we will assume that all input arguments and variables available to the mechanism
are also available to the model-specific functions and classes (e.g. $t, q, \udmax$).
Within each user class, we will hide the user index subscript  (e.g. $i$).
Moreover, the \AlgAttributes within a user class are round-indexed versions which will be updated
at the end of each round.
We will hide the round index subscript as we have done in Algorithms~\ref{alg:mmflearnsp}
and~\ref{alg:mmflearnnsp}; however our discussion and proofs will make the round index explicit.


\subsection{Deterministic Feedback}
\label{sec:detfbmodel}


With deterministic feedback, we will present mechanisms for learning under three different
strategy-proofness constraints: non-asymptotic strategy-proofness, asymptotic strategy-proofnes,
and without strategy-proofness.
As we will see,
the rate for the loss (efficiency) will improve as this strategy-proofness constraint is weakened.

\subsubsection{Learning with non-asymptotic strategy-proofness}
\label{sec:detfbssp}

In this setting, we will use Algorithm~\ref{alg:mmflearnsp} along with the
user class \userclasss and functions $r'$, \explphases as given
in Algorithm~\ref{alg:detfbspmodel}.
In the user class for user $i$, we maintain an attribute $\udubit$
(recall that we have suppressed user and time indices)
which is an upper bound on the user's unit demand $\udtruei$.
Each exploration consists of $n$ rounds, where each agent receives one round each;
in the $q$\ssth phase, it uses
$\udmax k/2^h$ as the unit demand for the agent on this round, where $h=\lceil \log_2(q+1)\rceil$,
and $k = 2q - 2^h + 1$.
That is, on round $t$ for such an agent, the allocation is
$\volit \times \udmax k/2^h$, where $\volit$ is the load of the agent for that round.
If the received feedback $X_i$ was larger than the threshold $\threshi$,
we update $\udubit$ to the minimum of the previous value and $\udmax k/2^h$
(line~\ref{lin:detfbspupdate}).

After the $\udubit$ values are updated for each agent in an exploration phase,
we use this value
for the remainder of the bracket in line~\ref{lin:spallocub} of Algorithm~\ref{alg:mmflearnsp}.
Finally, in this setting, we use $r'(q) = q$.
We have the following theorem for this mechanism.

\insertAlgoDetFbSPModel

\insertprethmspacing
\begin{theorem}
\label{thm:detfbsp}
Under the deterministic feedback model, Algorithm~\ref{alg:mmflearnsp}, when using the definitions
in Algorithm~\ref{alg:detfbspmodel}, satisfies the following.
Under truthful reporting from all agents, it is asymptotically efficient with
\emph{
\[
\LOTT \leq  10 n^{\nicefrac{3}{2}} \sqrtT \in \bigO(n^{\nicefrac{3}{2}}\sqrtT).
\]}
It is asymptotically fair with,
\emph{
\[
\UeiiT - \UiT \leq 2\sqrt{2}\Lipi \udmax \sqrtn \sqrtT \in \bigO(\sqrtn\sqrtT).
\]
}
for all agents $i$.
Finally,
it is (non-asymptotically) strategy-proof, i.e. for all agents $i$ and all policies $\pi$,
$\UpiiT - \UiT \leq 0$ for all $T\geq 1$.
\end{theorem}

While the above algorithm is asymptotically efficient, the rates are fairly slow for
deterministic observations, with the loss growing at rate $\bigO(\sqrtT)$.
This is not surprising, since we collect feedback only during the exploration phase,
where we essentially perform grid search for each agent separately.
This is due to the stringent strategy-proofness requirement---achieving a
(non-asymptotically) strategy-proof solution requires that the allocations used to estimate an
agent's demand do not depend on previous reports of any agent.
In what follows, we will demonstrate that the rate for the loss can be considerably improved
to have either logarithmic or no dependence on $T$ by considering weaker strategy-proofness
criteria.

\subsubsection{Learning with asymptotic strategy-proofness}
\label{sec:detfbsp}

In this section, we will study learning under deterministic feedback while satisfying asymptotic
strategy-proofness.
The user class \userclasss and the functions $r'$ and \explphases for this model are given
in Algorithm~\ref{alg:detfbmodel}.
We will describe these components individually.

The exploration phase in Algorithm~\ref{alg:mmflearnsp}
performs binary search for the $\udtruei$ value for each agent separately.
It maintains upper and lower bounds $(\udlbit, \udubit)$ for each agent initialised to
$(0, \udmax)$.
Each exploration phase consists of $2n$ rounds, where it performs two steps of binary search per
agent.
On a round $t$ for agent $i$ within this exploration phase,
the mechanism tries $\udit = (\udlbit+\udubit)/2$ for agent $i$.
All other agents $j$ do not receive an allocation during this round, i.e. $\udjt=0$.
If $\Xit<\threshi$, it updates the lower bound for the next round to $\udit$ and does not
change the upper bound, and vice versa if $\Xit\geq \threshi$.
After the two rounds for the agent, the $\udubit$ value computed as above is used
for the remainder of the bracket in line~\ref{lin:spallocub} of Algorithm~\ref{alg:mmflearnsp}.
Finally, for this model we use $r'(q) = \lfloor e^q\rfloor$.

We have the following theorem for this mechanism, whose proofs for the loss $\LOTT$ and fairness use
the simple fact that
binary search converges at an exponential rate.

\insertAlgoDetFbModel

\insertprethmspacing
\begin{theorem}
\label{thm:detsp}
Under the deterministic feedback model, Algorithm~\ref{alg:mmflearnsp}, when using the definitions
in Algorithm~\ref{alg:detfbmodel}, satisfies the following.
Under truthful reporting from all agents, it is asymptotically efficient with
\emph{
\[
\LOTT \leq 2n\log(eT) + 67 \volmax\udmax n^{2.39}  \in \bigO(n\log T).
\]
for all $T\geq 2n+3$.}
It is asymptotically fair with
\emph{
\[
\UeiiT - \UiT \leq 2n\Lipi\udmax\log(eT) \in \bigO(n\log T).
\]
}
for all agents $i$.
It is asymptotically strategy-proof with
\emph{
\[
\UpiiT - \UiT \leq  \Lipi \udmax \in \bigO(1).
\]
}
for all policies $\pi$ and all $T\geq 1$.
\end{theorem}

When compared to Theorem~\ref{thm:detfbsp}, the rate for efficiency and fairness have improved
significantly from $\bigO(\sqrtT)$ to $\bigO(\log T)$. 
However, the strategy-proofness guarantee is slightly weak:
the maximum gain in the sum of a user's utilities in $T$ rounds when deviating from
the truth, is at most a constant.

\subsubsection{Learning without Strategy-proofness Constraints}
\label{sec:detfbwsp}

We will use Algorithm~\ref{alg:mmflearnnsp} with the same \userclasss as in
Algorithm~\ref{alg:detfbmodel}.
%
Similar to above, the algorithm maintains upper and lower bounds
$(\udlbit, \udubit)$ for each agent initialised to $(0, \udmax)$.
In line~\ref{lin:nspudirec} of Algorithm~\ref{alg:mmflearnnsp},
it sets the recommendation to $\udit = (\udlbit+\udubit)/2$.
When it receives feedback, it updates the upper and lower bounds as follows.
\begin{align*}
\left(\udlbiitt{i}{t+1}, \udubiitt{i}{t+1}\right) \leftarrow
\begin{cases}
\left(\max(\allocit/\volit, \udlbi), \, \udubit\right) &\quad\text{if } \Xit < \threshi, \\
\left(\udlbit, \, \min(\allocit/\volit, \udubi)\right) &\quad\text{if } \Xit \geq \threshi.
\end{cases}
\label{eqn:detupdatensp}
\numberthis
\end{align*}
Here, $a\in\RR^n$ is the allocation vector chosen by \mmfs in line~\ref{lin:nspalloc}
when we use  $\{\udit\volit\}_{i=1}^n$ as the reported demands.
Despite some similarities to the strategy-proof case in computing $\udi$, there are
some important differences.
Crucially, as $\udi$ is only used to set the demands for \mmf, there is no guarantee that the agent
will be able to experience this allocation.
Consequently, we might not be able to estimate the demands of some agents.
For instance, if an agent's demands $\demtrueit = \volit\udtruei$ are large on all rounds,
then the agent might never be able to experience
this allocation due to contention for the limited resource.
Therefore, we cannot estimate $\udtruei$ accurately as $\Xit<\threshi$ on all rounds.
Fortunately, we can still effectively learn in this environment provided we 
estimate the demands of those agents whose demands are, loosely speaking, small.
We have the following theorem.

\insertprethmspacing
\begin{theorem}
\label{thm:detnsp}
Under the deterministic feedback model, Algorithm~\ref{alg:mmflearnnsp}, when using the definitions
in Algorithm~\ref{alg:detfbmodel}, satisfies the following.
Under truthful reporting from all agents, 
\emph{
\[\LOTT \leq 1 + 2n\volmax\udmax \in \bigO(n).
\]}
Therefore, it is asymptotically efficient.
Moreover, it is asymptotically fair.
Precisely, for any user $i$, and for all $T\geq 1$,
\emph{
\[
\UeiiT - \UiT \leq \Lipi\udmax \in \bigO(1).
\]}
\end{theorem}
When compared to Algoithm~\ref{alg:mmflearnsp}, Algorithm~\ref{alg:mmflearnnsp}
has better asymptotic bounds for efficiency and fairness, and in particular, it does
not grow with $T$.
On the flip side, Algorithm~\ref{alg:mmflearnnsp} is not strategy-proof.

\subsection{Stochastic Feedback with Parametric Payoffs}
\label{sec:glmfbmodel}

In this section, we will study  learning user demands under asymptotic strategy-proofness and
Bayes-Nash incentive-compatibility.
Algorithm~\ref{alg:glmfbmodel} defines \userclasss and the functions $r'$ and \explphases
for this model which we will use in both cases.

We begin by describing a procedure to construct
confidence intervals for the unit demands $\{\udtruei\}_i$ from past data.
Let $\{(\allocis, \Xis)\}_{s\in D_t}$
be a dataset of allocation--reward pairs for user $i$
where $D_t\subset \{1,\dots,t-1\}$ is a subset of the first $t-1$ time indices.
Then define,
\begin{align*}
\thetait = \argmin_{\theta \geq \thetamin}\,
\left| \sum_{s\in D_t}
\frac{\allocis}{\volis\sigmasqis}\left(
\mu\Big(\frac{\allocis\theta}{\volis}\Big) - \Xis\right)
\right|,
\hspace{0.3in}
\Asqit := \sum_{s\in D_t} \frac{\allocis^2}{\volis^2\sigmasqis}.
\label{eqn:thetaitdefn}
\numberthis
\end{align*}
Here, $\thetait$ can be interpreted as an estimate for $\thetatruei$ using the data in round indices
$D_t$.
Recall, from Section~\ref{sec:setup},
 that for the stochastic models we wish our results to hold with a target success probability of
at least $1-\delta$, where $\delta$ is an input to the mechanism.
We now define $\betat$ as follows,
\begin{align*}
\betat = (5/\kappamudot)\sqrt{\log(n\pi^2|D_t|^2/6\delta)}
\;\;\in\;\bigO(\log(n|D_t|/\delta))
\numberthis
\label{eqn:betatglm}
\end{align*}
We then define
$(\thetalbit, \thetaubit), (\udlbit, \udubit)$
as shown below, which are confidence intervals
for $\thetatruei$ and $\udtruei$ respectively (Lemma~\ref{lem:confinterval}).
We have:
\begin{align*}
(\thetalbit, \thetaubit) =
\left(\min\left\{\thetamin, \,\thetait - \frac{\betat}{\Ait}\right\},
        \;\thetait + \frac{\betat}{\Ait} \right),
\hspace{0.2in}
(\udlbit, \udubit) =
\left(\frac{\muinv(\threshi)}{\thetaubit}, \,\frac{\muinv(\threshi)}{\thetalbit} \right).
\label{eqn:udglmconfint}
\numberthis
\end{align*}

\insertAlgoGLMFbModel

\subsubsection{Learning with asymptotic strategy-proofness}
\label{sec:glmfbsp}

The exploration phase in Algorithm~\ref{alg:mmflearnsp} consists of just a single round where the
mechanism allocates in proportion to the entitlements,
while the latter phase in bracket $q$ is executed for
$r'(q) = \lfloor 5q^{\nicefrac{1}{2}}/6 \rfloor$ rounds.
The upper confidence bound $\udubi$ for $\udtruei$ in line~\ref{lin:spudiub} is
set to be $\udubit$ as computed in~\eqref{eqn:udglmconfint},
where we set $D_t$ to be the time indices corresponding to the exploration phase;
observe that in Algorithm~\ref{alg:mmflearnsp} and
Algorithm~\ref{alg:glmfbmodel}, the dataset $\Dcal_{it}$ is updated in the
\ucrecordfbs method which is called in \explphase.
The following theorem states the properties of Algorithm~\ref{alg:mmflearnsp} under this model.

\insertprethmspacing
\begin{theorem}
\label{thm:glmsp}
Assume that the rewards follow the parametric feedback model outlined in Section~\ref{sec:fbmodels}
and that it satisfies Assumption~\ref{asm:glm}.
Then Algorithm~\ref{alg:mmflearnsp}, when using the definitions
in Algorithm~\ref{alg:detfbmodel}, satisfies the following statements with 
 probability greater than $1-\delta$.
Under truthful reporting from all agents,
it is asymptotically efficient with,
\[
\LOTT \leq 3\Ttwth + 3Cn\betaT\Ttwth
\;\in\; \bigO\left(n\Ttwth\sqrt{\log(nT/\delta)}\right),
\]
for all $T > 0$.
Here, 
\emph{$C = \frac{ \sqrt{2} \udmax \volmax^2 \sigmamax}
    {\thetamin \left(\min_i \entitli\right)}$}.
Moreover,  $\UeiiT \leq \UiT$ for all users $i$ and for all $T\geq 1$.
Finally, for any user $i$ and for any policy $\pi$,
$\UpiiT - \UiT \leq 0$ for all $T\geq 1$.
Therefore, Algorithm~\ref{alg:mmflearnsp} is probably fair and
probably strategy-proof.
\end{theorem}

\subsubsection{Learning with asymptotic Bayes-Nash incentive-compatibility}
\label{sec:glmfbwsp}

In Algorithm~\ref{alg:mmflearnnsp}, on round $t$ we set the recommendation $\udi$ in
line~\ref{lin:nspudirec} to be $\udubit$ from~\eqref{eqn:udglmconfint},
where  we set $D_t=\{1,\dots,t-1\}$.
To state our theorem, we first define the following constants.
\begin{align*}
\Cone = \frac{\thetamin}{2\sigmamax\volmax^2},
\hspace{0.4in}
\Ctwo = \frac{1}{\sqrt{\log(1+\Cone^2)}},
\hspace{0.4in}
\Cthree = \left(\frac{\sigmamax\,\volmax}{\sigmamin \,\volmin \min_i \entitli}\right)^2.
\numberthis
\label{eqn:glmnspdefns}
\end{align*}
We have the following theorem.

\insertprethmspacing
\begin{theorem}
\label{thm:glmnsp}
Assume that the rewards follow the parametric feedback model outlined in Section~\ref{sec:fbmodels}
and that it satisfies Assumption~\ref{asm:glm}.
Then Algorithm~\ref{alg:mmflearnsp}, when using the definitions
in Algorithm~\ref{alg:detfbmodel}, satisfies the following statements with 
 probability greater than $1-\delta$.
Under truthful reporting from all agents,
it is asymptotically efficient with,
\[
\LOTT \;\leq\, 1 + \Ctwo n \betaT T^{\nicefrac{1}{2}} \sqrt{\log(C_3T)} \;
\in\bigO\left(nT^{\nicefrac{1}{2}}\sqrt{\log(nT/\delta)\log(T)}\right),
\]
for all $T\geq 0$.
Moreover,  $\UeiiT \leq \UiT$ for all users $i$ and for all $T\geq 1$;
i.e. it is probably fair.
Finally, assume all agents except $i$ are truthful.
Then, for all $\pi$, and for all $T>0$,
\emph{
\[
\UpiiT - \UiT 
    \leq \frac{\Lipi \Ctwo}{\volmin} (n-1) \betaT\sqrtT{\log(\Cthree T)}
\in\bigO\left(nT^{\nicefrac{1}{2}}\sqrt{\log(nT/\delta)\log(T)}\right).
\]
}
That is, Algorithm~\ref{alg:mmflearnnsp} is probably asymptotically Bayes-Nash incentive-compatible.
\end{theorem}

In the parametric model, Algorithm~\ref{alg:mmflearnnsp}
is asymptotically efficient with rate $\bigOtilde(T^{\nicefrac{1}{2}})$, which is better than the
$\bigOtilde(T^{\nicefrac{2}{3}})$ rate of Algorithm~\ref{alg:mmflearnsp}.
Moreover, unlike in Section~\ref{sec:detfbmodel},
Algorithm~\ref{alg:mmflearnnsp} also satisfies an 
asymptotic Bayes-Nash incentive-compatibility condition on an agent's truthful behaviour.
While this is weaker than asymptotic strategy-proofness, it describes an approximate Nash
equilibrium: an agent does not stand to
gain much by deviating from truthfully reporting their threshold, loads, and feedback, if
all other agents are being truthful.

\subsection{Stochastic Feedback with Nonparametric Payoffs}
\label{sec:nonparfbmodel}

In the nonparametric setting, we will use a branch-and-bound method for estimating the unit
demands.
For this, we will define an infinite binary tree for each user, with each node in the tree
corresponding to a sub-interval of $[0, \udmax]$.
As we collect data, we will expand the nodes in this tree,
and assign each point to multiple nodes of the expanded tree.
We will then use the data assigned to these nodes to construct upper and lower confidence intervals
for the values of the payoff in said interval.
%

We will index the nodes of our binary tree by a tuple of integers $(h, k)$ where $h\geq 0$ is the
height
of the node, and $k$, which lies between $1$ and $2^h$ denotes its position among all nodes at
height $h$.
The root node is $(0,1)$.
The left child of node $(h,k)$ is $(h+1, 2k-1)$ and its right child is $(h+1, 2k)$;
i.e. the children are at the next height and occupy the position in between the children of the
nodes immediately to the left and right of the parent.
It follows that the parent of a node $(h,k)$ is $(h-1, \lceil k/2\rceil)$.
Each node $(h,k)$ is associated with an interval $\Ihk\subset[0,\udmax]$,
defined as follows:
\begin{align*}
\Ihk = \left[\frac{\udmax(k-1)}{2^h}, \;\frac{\udmax k}{2^h}\right)
\;\;\text{if $1\leq k<2^h$},
\hspace{0.4in}
\Ihk = \left[\frac{\udmax(2^h -1)}{2^h}, \,\udmax\right]
\;\;\text{if $k=2^h$}.
\numberthis\label{eqn:Ihk}
\end{align*}
For example, the interval corresponding to the root $(0,1)$ is $[0,\udmax]$;
its children are $(1,1)$ and $(1,2)$ with corresponding intervals $[0, 1/2)$ and $[1/2, 1]$
respectively.
We see that the nodes at a given height partition $[0,\udmax]$,
i.e. for all $h$, $\bigcup_{k=1}^{2^h} \Ihk = [0,\udmax]$ and
for all $h$ and $k_1<k_2$, $\Ihhkk{h}{k_1}\cup \Ihhkk{h}{k_2} = \emptyset$.

We will now describe the \userclass, outlined in
Algorithms~\ref{alg:nonparfbmodelpart1}-\ref{alg:nonparfbmodelpart4}.
Recall that we have dropped the user and time subscripts in the algorithm.
For what follows, we define the following quantities.
\begin{align*}
\ttilde = 2^{\lceil\logtwo(t)\rceil},
\hspace{0.4in}
\betat = \sqrt{(4+2\log(2))\log\left(\frac{n\pi^2 t^3}{6\delta}\right)},
\hspace{0.4in}
\tauht = \frac{\betat^2}{L^2}4^h.
\numberthis
\label{eqn:betataudefn}
\end{align*}

\insertAlgoNonparFbModelOne
\insertAlgoNonparFbModelTwo
\insertAlgoNonparFbModelThree
\insertAlgoNonparFbModelFour

\userclasss will maintain an infinite binary tree for each user as described above.
It will incrementally expand nodes in the tree as it collects data, with the expanded nodes
reflecting the data it has received.
Let $\datait\subset\{1,\dots,t-1\}$ denote a subset of the round indices in the first $t-1$ rounds;
we will use feedback from agent $i$ in rounds $\datait$ to learn her demand.
Let $\Tcalit$ denote the sub-tree of expanded nodes for each user $i$ at the beginning
of round $t$.
Consider a round $t$, at the end of which user $i$ receives feedback; this could be all rounds
in Algorithm~\ref{alg:mmflearnnsp} or only during the exploration rounds in
Algorithm~\ref{alg:mmflearnsp}.
When the \userclasss receives a data point $(\normallocit, \Xit, \sigmait)$ in \ucrecordfbs
(line~\ref{lin:nonparucrecordfb}),
it assigns that data point to the nodes along a path $\Pit$,
where
\begin{align*}
\normallocit\in\Ihk, \text{ for all } (h,k)\in\Pit.
\numberthis \label{eqn:Pit}
\end{align*}
We will find it useful to view $\Pit$ as a set which contains nodes.
Accordingly, if user $i$ did not receive feedback during a round $t$, we let $\Pit = \emptyset$.
We now define 
$\VSithk$ to be the sum of squared inverse sub-Gaussian constants assigned to node $(h,k)$ and
$\fbarithk$ to be the sample mean of the data assigned to $(h,k)$ weighted by the
sub-Gaussian constants.
We have:
\begin{align*}
\VSithk = \sum_{s\in\datait} \frac{1}{\sigmais^2} \,
        \indfone\left((h,k)\in\Pis\right),
\hspace{0.25in}
\fbarithk = \frac{1}{\VSithk}\sum_{s\in\datait}
        \frac{\Xit}{\sigmais^2} \, \indfone\left((h,k)\in\Pis\right)
\numberthis \label{eqn:fbar}
\end{align*}
In \algnonpar, $\VSithk,\fbarithk$ are updated in the \ucassigntonodes method in
line~\ref{lin:ucassigntonode}.
In \ucrecordfb, when we traverse along the path $\Pit$~\eqref{eqn:Pit} assigning the node to
each point in that path, we stop either when we reach a leaf node
in $\Tcalit$, or if there is insufficient data at the curent node (line~\ref{lin:stopPit}).
The latter criterion is determined by $\VSithk<\tauht$, where $\tauht$ is as defined
in~\eqref{eqn:betataudefn}.
If the last node in $\Pit$ was a leaf node with sufficient data, i.e. if $\VSithk\geq \tauht$,
it expands that node and adds its children to the tree.

\userclasss maintains quantities $\Blbithk$ ($\Bubithk$) for each node in the
tree, which can be interpreted as a lower (upper) bound on
the infimum (supremum) of $\payoffi$ in the interval $\Ihk$.
To describe this bound, we first define upper and lower confidence bounds $\flbit,\fubit$ for each
node $(h,k)$ using only the data assigned to the node:
\begingroup
\allowdisplaybreaks
\begin{align*}
    \flbthk = \begin{cases}
\fbarithk - \betattilde\;\VSithk^{-\nicefrac{1}{2}} - L\cdot2^{-h} \quad &\text{if} \quad \VSithk>0, \\
-\infty \quad &\text{if} \quad \VSithk=0,
\end{cases}
\numberthis
\label{eqn:flbfub}
\\
    \fubthk = \begin{cases}
\fbarithk + \betattilde\;\VSithk^{-\nicefrac{1}{2}} + L\cdot2^{-h} \quad &\text{if}\quad \VSithk>0, \\
\infty \quad &\text{if} \quad \VSithk=0,
\end{cases}
\end{align*}
\endgroup
Here, $\betat, \ttilde$ are as defined in~\eqref{eqn:betataudefn}.
Above, 
$\betattilde\VSithk^{-\nicefrac{1}{2}}$ accounts for the stochasticity in the observed rewards,
while $L/2^h$ accounts for the variation in the function value in the interval $\Ihk$.
As we will show in our proofs, with high probability, $\payoffi(a) \in (\flbthk, \fubthk)$ for all
$a\in\Ihk$.

While $\flbthk,\fubthk$ provide us a preliminary confidence interval on the function values,
this can be refined by considering the bounds of its children and accounting for monotonicity of the
function.
The actual lower bounds $\Blbit$ for the function are computed as follows,
\begin{align*}
\Blbithk &= \begin{cases}
0 \hspace{2.5in} \text{if } \text{$(h,k')\notin\Tcalit$ for all $k'\leq k$},
\\
\max\left(\flbithk, \Blbiitthhkk{i,}{t-1}{h}{k}, \Blbiitthhkk{i}{t}{h+1}{2k-1}\right),
 \quad\quad \text{if } \text{$(h,k)\in\Tcalit$}, \\
\Blbiitthhkk{i}{t}{h}{k'}, \hspace{0.58in}\; \text{otherwise.
            Here,  $k'=\max\{k''<k; (h,k'') \in\Tcalit \}$.
            }
\numberthis \label{eqn:Blb}
\end{cases}
\end{align*}
In the first case above, if a node has not been expanded, nor have any nodes to its left
at the same height, we set $\Blbithk = 0$.
In the second case, if a node has been expanded, we set $\Blbithk$ to the maximum of
its $\flbithk$ value, the bound $\Blbiitthhkk{i,}{t-1}{h}{k}$ from the previous round,
and the bound $\Blbiitthhkk{i}{t}{h+1}{2k-1}$ of its left child.
The third case applies to nodes which have not been expanded, but if some node at the same
height to its left has been expanded;
in such cases, we set it to $\Blbiitthhkk{i}{t}{h}{k'}$ where $(h,k')$ is the right-most expanded
node to the left of $(h,k)$.
Observe that~\eqref{eqn:Blb} defines $\Blbit$ values for all nodes $(h,k)$ in the
tree:
for all nodes where there is no other expanded node at its height, $\Blbithk=0$ by the first line;
for the deepest expanded nodes, $\Blbithk$ is given by the second line;
then, the $\Blbithk$ values for the unexpanded nodes at the same height can be computed using the
first or the third lines depending on whether any nodes to the left of the node have been expanded or
not;
we can then proceed to the previous height and compute $\Blbit$ in a similar fashion, starting
with those nodes that have been expanded.

The upper bounds $\Bubit$ are computed in an analgous fashion:
\begin{align*}
\Bubithk &= \begin{cases}
1 \hspace{2.31in} \text{if } \text{$(h,k')\notin\Tcalit$ for all $k'\geq k$},
\\
\min\left(\fubithk, \Bubiitthhkk{i,}{t-1}{h}{k}, \Bubiitthhkk{i}{t}{h+1}{2k}\right),
 \quad\quad\quad \text{if } \text{$(h,k)\in\Tcalit$}, \\
\Bubiitthhkk{i}{t}{h}{k'}, \hspace{0.51in} \text{otherwise.
            Here,  $k'=\min\{k''>k; (h,k'') \in\Tcalit \}$.
            }
\numberthis \label{eqn:Bub}
\end{cases}
\end{align*}
In the first case, if a node has not been expanded, nor have any nodes to
its right at the same height, we set $\Bubithk = 1$.
In the second case, if a node has been expanded, we set $\Bubithk$ to the
minimum of its $\fubithk$ value, the bound $\Bubiitthhkk{i,}{t-1}{h}{k}$ from the previous round,
and the bound $\Bubiitthhkk{i}{t}{h+1}{2k}$ of its right child.
The third case applies to nodes which have not been expanded, but if some node at the same
height to its right has been expanded,
in such cases, we set it to $\Bubiitthhkk{i}{t}{h}{k'}$ where $(h,k')$ is the left-most expanded
node to the right of $(h,k)$.
Similar to above,~\eqref{eqn:Bub} defines $\Bubit$ for all nodes in the tree.
This completes the description of the confidence intervals for the payoffs.
In our analysis, we show that if $\flbit,\fubit$ trap the function,
so do $\Blbit,\Bubit$. Precisely,
\begin{align*}
&\payoffi(a) \in (\flbithk, \fubithk)\;\forall\,a\in\Ihk, \text{ for all $t$ and $(h,k)$}
\\
&\quad\implies\quad
\payoffi(a) \in (\Blbithk, \Bubithk)\;\forall\,a\in\Ihk, \text{ for all $t$ and $(h,k)$}.
\end{align*}

\userclasss computes and updates the confidence intervals in two different places as shown above.
First, whenever a new data point is received, it is assigned to nodes along a chosen path $\Pit$
changing the $\fbarithk$ values in~\eqref{eqn:fbar} and~\eqref{eqn:flbfub}.
Therefore, the $\Blb,\Bub$ values need to be updated, not only for those nodes in $\Pit$,
 but possibly also for
neighbouring
nodes at the same height, due to the third case in~\eqref{eqn:Blb} and~\eqref{eqn:Bub}.
This is effected via the \ucupdateboundsonpathtoroots method
(line~\ref{lin:ucupdateboundsonpathtoroot}) invoked in line~\ref{lin:callupdateboundsonpath}.
Second, whenever $t=\ttilde$, the value of $\betattilde$ used in~\eqref{eqn:flbfub} changes,
requiring that we update the confidence intervals throughout the entire tree.
This is effected via the \ucrefreshboundsintrees method (line~\ref{lin:ucrefreshboundsintree}).
For Algorithm~\ref{alg:mmflearnnsp}, we have shown its invocation 
explicitly in line~\ref{lin:callrefreshinucgetudrec} in the \ucgetudrecs method which is called for
each user by Algorithm~\ref{alg:mmflearnnsp} in each round.
For Algorithm~\ref{alg:mmflearnsp}, for brevity, we have not made this invocation explicit.
An implementation of Algorithm~\ref{alg:mmflearnsp} could, say, call \ucrefreshboundsintrees 
at the beginning of each round by checking if $t=\ttilde$.

We will now describe the mechanisms for learning in this model under asymptotic strategy-proofness
and without any strategy-proofness constraints.

\subsubsection{Learning with asymptotic strategy-proofness}
\label{sec:nonparfbsp}

We will use Algorithm~\ref{alg:mmflearnsp} in this setting along with the definitions in
\algnonpar.
We let $r'(q) = \lfloor 5nq^{\nicefrac{1}{2}} / 6 \rfloor$.
Each exploration phase consists of $n$ rounds, one per user.
In the round corresponding to user $i$, we call the $\ucgetudrecforub$ method
(line~\ref{lin:nonparucgetudrecforub}) of \userclass.
This method, invokes \ucubtraverse (line~\ref{lin:nonparucubtraverse}) which traverses a path of
nodes in the tree which either contain $\udtruei$ or is to the right of a node containing
$\udtruei$.
Concretely, starting from the root node it hops to the right child if the right child's lower
confidence bound $\Blbit$ is larger than $\threshi$ and to the left child otherwise;
intuitively, if it has enough evidence that $\udtruei$ is not in the right child, it chooses the
left child.
We proceed this way, and stop
either when we reach a leaf of $\Tcalit$ or if there are not sufficient
data points assigned to the node, quantified by the condition $\VSithk < \tauht$.
We update $(\hubit,\kubit)$ to be the value returned by \ucubtraverses and return
an arbitrary point in $\Ihk$.

The $(\hubit,\kubit)$ node, which is updated once every bracket during the exploration phase round
for user $i$, is used in calculating the value to be used as the reported demand in \mmf.
Precisely, when  Algorithm~\ref{alg:mmflearnsp} calls \ucgetudub, 
we return  $\udmax \kubit/2^{\hubit}$ which is the right-most point of the interval $\Ihkubit$.
Lemma~\ref{lem:nonparucgetudub} in Appendix~\ref{sec:proofs_nonpar}
shows that the point obtained in this manner
is an upper confidence bound on $\udtruei$.
It is worth observing that $\udmax k/2^h \notin \Ihk$, unless $k=2^h$
(see~\eqref{eqn:Ihk}).

The following theorem outlines the main properties of Algorithm~\ref{alg:mmflearnsp}
in the nonparametric setting.

\insertprethmspacing
\begin{theorem}
\label{thm:nonparsp}
Assume that the rewards follow the nonparametric feedback model outlined in
Section~\ref{sec:fbmodels} and that it satisfies Assumption~\ref{asm:ntg}.
Let $G\in(0,\ntg]$ be given and let $\epsG$ be as defined in Definition~\ref{def:ntg}.
Then Algorithm~\ref{alg:mmflearnsp}, when using the definitions
in Algorithms~~\ref{alg:nonparfbmodelpart1}--\ref{alg:nonparfbmodelpart4},
satisfies the following with probability greater than $1-\delta$.
Under truthful reporting from all agents, it is  asymptotically efficient with
\emph{\begin{align*}
\LOTT &\leq 3\nonth\Ttwth +
       C_1\frac{\Lf^{\nicefrac{1}{2}}\volmax\udmax\sigmamax}{G^{\nicefrac{3}{2}}}
        \betatwoT \nfoth\Ttwth +
\\
&\hspace{1in}
    C_2\nfith\Tonth\left( \frac{\Lf\sigmamax^2\udmax^3}{G^3\epsG^3}\betatwoT^2
+ \frac{\udmax^2\sigmamax^2}{G^2\epsG^2}\betatwoT^2
  \frac{\Lf^2\udmax}{G^2} + \frac{4\Lf\udmax}{G\epsG} + 1
    \right)
\\
&\in \bigO\left( 
       \frac{\log(nT/\delta)}{G^3\epsG^3}\nfith\Tonth
       + \frac{\sqrt{\log(nT/\delta)}}{G^{\nicefrac{3}{2}}} \nfoth T^{\nicefrac{2}{3}}
\right)
\end{align*}}
Here, $C_1,C_2$ are global constants.
Moreover, it is asymptotically fair with \emph{
\begin{align*}
\UeiiT - \UiT &\leq 3\Lipi\udmax\nonth\Ttwth
\;\in \bigO(\nonth\Ttwth).
\end{align*}
}
Finally, it is asymptotically strategy-proof with \emph{
\begin{align*}
\UeiiT - \UiT &\leq 3\Lipi\udmax\nmtwth\Ttwth
\;\in \bigO(\nmtwth\Ttwth).
\end{align*}
}
for all policies $\pi$ and all $T\geq 1$.
\end{theorem}

The $\nfoth\Ttwth$ rate in the dominant term for the loss is similar to the stochastic
parametric model.
However, there is also a fairly strong dependence on the near-threshold gradient:
in addition to the $G^{\nicefrac{-3}{2}}$ depenence on the leading term, there is also a
$G^{-3}\epsG^{-3}$ depenence on a lower order term.
The main reason for this strong dependence is that we need to translate confidence
intervals on the payoff
obtained via the rewards to a confidence interval on the unit demand.
As we outlined in Figure~\ref{fig:ntgillus}, this translation can be difficult if $\ntg$ is small.

The results for fairness and strategy-proofness are also weaker than the parametric model,
with the guarantees here holding only asymptotically.
An interesting observation here is that the asymptotic rate for strategy-proofness has an
$\nmtwth$ dependence on the number of agents.
This says that there is less opportunity for an agent to manipulate the outcomes when there are
many agents.
Next, we will look at learning in the nonparametric model without any strategy-proofness
constraints.

\subsubsection{Learning without Strategy-proofness Constraints}
\label{sec:nonparfbwsp}

We will use Algorithm~\ref{alg:mmflearnnsp} along with the definitions in \algnonpar.
The \ucgetudrecs method is given in line~\ref{lin:nonparucgetudrec} of \algnonpar.
To describe this method,
first define,
\begin{align*}
\Bithk = \min\left(\,\Bubithk - \threshi, \,\threshi-\Blbithk\right).
\numberthis\label{eqn:Bit}
\end{align*}
If $\Bithk$ is small, this is either because the upper bound $\Bubithk$ is close to
or smaller than the threshold $\threshi$, \emph{or} if the lower bound $\Blbithk$ is close to
or larger than the $\threshi$.
Intuitively, if $\Bithk$ is small we are more confident that $\udtruei\notin\Ihk$.
When we call the \ucgetudrecs method (line~\ref{lin:nonparucgetudrec}), it traverses a path along
this tree, where at each node, it chooses the child with the highest $\Bit$ value;
at each step, it refines its recommendation of $\udtruei$ in this manner.
It stops either when it has reached a leaf of $\Tcalit$ or if there is not sufficient
data assigned to the node, such that any finer estimate will not be meaningful.
It then returns an arbitrary point in the interval corresponding to the last node.

The following theorem outlines the main theoretical results for the nonparametric model when
using the above procedure along in Algorithm~\ref{alg:mmflearnnsp}.

\insertprethmspacing
\begin{theorem}
\label{thm:nonparnsp}
Assume that the rewards follow the nonparametric feedback model outlined in
Section~\ref{sec:fbmodels} and that it satisfies Assumption~\ref{asm:ntg}.
Let $G\in(0,\ntg]$ be given and let $\epsG$ be as defined in Definition~\ref{def:ntg}.
Then Algorithm~\ref{alg:mmflearnnsp}, when using the definitions
in Algorithms~~\ref{alg:nonparfbmodelpart1}--\ref{alg:nonparfbmodelpart4}
satisfies the following with probability greater than $1-\delta$.
Under truthful reporting from all agents, it is  asymptotically efficient with
\emph{\begin{align*}
\LOTT 
&\leq
C_1 n\left(
       \frac{\Lf^{\nicefrac{1}{2}}\volmax\udmax\sigmamax}{G^{\nicefrac{3}{2}}}
        \betatwoT T^{\nicefrac{1}{2}}
+
\frac{\Lf\volmax\udmax^3\sigmamax^2}{G^3\epsG^3}\betatwoT^2 
+
\frac{\volmax\udmax^2\sigmamax^2}{G^2\epsG^2}\betatwoT^2 
+ C_3
 \right)
\\
&\in \bigO\left( 
       \frac{n\log(nT/\delta)}{G^3\epsG^3}
        +
       \frac{\sqrt{\log(nT/\delta)}}{G^{\nicefrac{3}{2}}}n T^{\nicefrac{1}{2}}
    \right)
\end{align*}}
Moreover, it is asymptotically fair with \emph{
\begin{align*}
\UeiiT - \UiT &\leq
C_2\Lipi\,\left(
       \frac{\Lf^{\nicefrac{1}{2}}\udmax\sigmamax}{G^{\nicefrac{3}{2}}}
        \betatwoT T^{\nicefrac{1}{2}}
+
\frac{\Lf\volmax\udmax^3\sigmamax^2}{G^3\epsG^3}\betatwoT^2 
+
\frac{\volmax\udmax^2\sigmamax^2}{G^2\epsG^2}\betatwoT^2 
+  C_3
\right)
\\
&\in \bigO\left( 
       \frac{\log(nT/\delta)}{G^3\epsG^3}
        +
       \frac{\sqrt{\log(nT/\delta)}}{G^{\nicefrac{3}{2}}}T^{\nicefrac{1}{2}}
\right)
\end{align*}}
Above, $C_1, C_2$ are global constants and
\emph{$C_3= L^2/G^2 + L\udmax/G + 1$}.
\end{theorem}

As was the case for the previous models, we find that the rate for the efficiency and fairness
are better when the strategy-proofness constraints are relaxed, improving from
$\bigOtilde(\Ttwth)$ to $\bigOtilde(\sqrtT)$.
The dependence on the near threshold gradient in the dominant term is similar to
Theorem~\ref{thm:nonparsp}, although now the $G^{-3}\epsG^{-3}$ is coupled with a
$n\log(T)$ whereas in Theorem~\ref{thm:glmsp}, it was $\nfith\Tonth$.

\insertFigTreeIllus

This completes the description and results for the nonparametric model.
It is worth emphasising that while the procedure is seemingly long, it is computationally
very efficient in practice as most of the steps are relatively
simple. 
Moreover, the height of the tree does not grow too rapidly since we wait for $\VSithk$ to grow
larger than $\tauht$ which increases exponentially in $h$~\eqref{eqn:betataudefn}.
In fact,
the above method is being deployed in a real-time scheduling system, where
sub-second response times are necessary.
The method was designed with this important practical consideration of being able to determine the
allocations fast.
In Section~\ref{sec:experiments}, we have presented some empirical results on run time.

In Figure~\ref{fig:treeillus}, we have illustrated the expansion of the tree and the construction
of the confidence intervals in Algorithm~\ref{alg:mmflearnnsp}.
We show two different users whose demands
are small and large relative to their entitlement.
For the first user, the tree is expanded
deep around $\udtruei$ enabling us to accurately estimate  this user's demand.
In the latter case, the demand is large; due to resource contention and fairness constraints,
we are not able to allocate many resources to this user and accurately estimate her demand.
We have also illustrated confidence intervals for the payoff, computed via the
\ucgetconfinterval{} method (line~\ref{lin:ucgetconfinterval}).
Observe that the confidence intervals are monotonic.
For instance, in the bottom figure, one would
expect the confidence intervals in the $(0.75, 1)$ interval to be large due to the lack of data.
However, we are able to use monotonicity to clip the lower confidence bound using data from
allocations in the range $(0, 0.75)$.
Monotonicity of the confidence intervals is necessary for the correctness of our algorithms---see
Remark~\ref{rem:monotonicity}.
%

\subsection{Proof Sketches \& Discussion}
\label{sec:pfsketch}

We conclude this section with a brief discussion on high-level proof techniques and relations to
previous work in the online learning and bandit literature.

To control the $T$-period loss, we bound the instantaneous loss $\lot(\demtrueit, \alloct)$
using a two-way argument.
First, if there are any unallocated resources $\lotur>0$, we bound $\lot$ by 
the unmet demand $\lotud$;
this is usually the easier case since it captures instances when there are more resources
available than the sum of demands.
For example, when we use an upper bound for
$\udtruei$ in Algorithm~\ref{alg:mmflearnsp}, 
we show that
when $\lotur>0$, we also have
 $\lotud=0$ as all agents have met their demands for that round.
Even in Algorithm~\ref{alg:mmflearnnsp} when we do not use an upper bound, this term can be shown
to be small.
However, when $\lotur(\demtrueit)=0$, this means that there is a scarcity of resources.
This is the harder case to analyse since
the mechanism can risk an inefficient allocation by over-allocating for some agents.
Generally speaking, this term will vanish if our estimate of the demands converge to the
true demands, for which we use the properties of our allocation scheme.
However, as noted before, 
the demands of all agents
cannot be estimated accurately in Algorithm~\ref{alg:mmflearnnsp} due to resource contention and
fairness constraints.
For instance, in Figure~\ref{fig:treeillus}, we cannot allocate enough resources for the agent in
the bottom figure to accurately estimate her demand.
Therefore, this requires a
more careful analysis which argues that we only need to estimate the
demands of agents whose demands are small in order to achieve an efficient allocation.
This argument needs to be made carefully as we need to account for the changing demands at each
round.

Our proofs for fairness and strategy-proofness rely on several useful properties of \mmf, which we
prove in Appendix~\ref{sec:pfmmf}.
For example, in Algorithm~\ref{alg:mmflearnsp}, since we use upper bounds on the agents'
demands in line~\ref{lin:spallocub}, it guarantees that an agent cannot gain by inflating the
mechanism's estimate of her demand.
One interesting observation here is that the fairness guarantees for the parametric model
hold non-asymptotically---albeit probabilisitically---whereas the guarantees for
the deterministic model hold only asymptotically.
This is a consequence of the parametric assumption, as it allows us to estimate an agent's payoff by
simply estimating her parameter $\thetai$.

While our algorithms bear some superficial similarities to optimistic
bandit methods which use upper confidence bounds on the reward~\citep{auer03ucb},
there are no immediate connections as we do not maximise any function of the
rewards.
In fact, it can be argued that an optimistic strategy for resource allocation would allocate resources
assuming that satisfying each user's demand was as easy as possible given past data. Consequently,
it would allocate resources based on a lower (confidence) bound of the demand.
In contrast, our use of the upper bounds stems from the strategy-proofness
and fairness requirements.
That said, the construction of the confidence intervals for the parametric and nonparametric models
borrows ideas from prior work in the bandit literature using generalised linear
models~\citep{filippi2010parametric,rusmevichientong2010linearly,dani2008stochastic},
and nonparametric models~\citep{jones93direct,bubeck2010x,azar2014online,shang2018adaptive,%
sen2018multi,sen2019noisy,grill2015black}.

It is worth highlighting some of the differences in the nonparametric setting when compared to the
above nonparametric bandit optimisation work.
First, bandit analyses are concerned with minimising the cumulative regret which compares the
optimal payoff \emph{value} to the queried payoff \emph{value}.
However, here, our loss is given in terms of the allocations (see~\eqref{eqn:LOTTdefn}) and not the
payoff value achieved by the allocations.
This requires us to translate payoff values we have observed to an estimate on the demand,
which can be difficult\footnote{%
For example, in optimisation, finding the optimal value is typically is easier than finding the
optimal point.
}.
Second, in typical bandit settings, we may query the function at any point we wish.
However, in our setting, we may not be able to do so due to contention on limited resources.
For example, in our algorithm, we may not receive feedback for the chosen
recommendation; rather, the allocation is chosen by \mmfs based on the recommendation which may be
smaller than the recommendation itself.
In the design of the algorithm, we need to ensure
that recommendations are carefully chosen so that \mmfs does not repeatedly choose the same
allocation for the user; ensuring that the confidence intervals are monotonic is critical for doing
so.
Moreover, as we stated before,
the analysis needs to handle the fact that we may never be able to estimate a user's demand
due to this phenomenon.
Additionally, we also need to account for the discrepancy between the recommended point and the
evaluated point when recording feedback.
%
In addition to these two main differences, there are a number of other differences that arise due to
the differences in the problem set up.
Since our goal here is to estimate the demands, the criterion for traversing a tree to
select a recommendation (lines~\ref{lin:nonparucgetudrecforub},~\ref{lin:nonparucubtraverse}
and lines~\ref{lin:nonparucgetudrec},~\ref{lin:ucgetBval} for Algorithms~\ref{alg:mmflearnsp}
and~\ref{alg:mmflearnnsp} respectively)
is different from
prior tree-based bandit work~\citep{bubeck2010x,azar2014online}.
Moreover, the
computation of the lower/upper confidence bounds in~\eqref{eqn:Blb},~\eqref{eqn:Bub} are
markedly different from the usual way they are computed in the optimisation literature.
Navigating these challenges requires new design and analysis techniques.

\section{Experiments}
\label{sec:experiments}

In this section, we present our experimental evaluation on synthetic experiments and
a prediction-serving task.
We compare the following classes of methods in this evaluation.
\begin{itemize}
\item Entitlement based allocation: on all rounds we allocate in proportion to entitlements.
\item The three methods in Section~\ref{sec:detfbmodel} using deterministic feedback, i.e.
Algorithms~\ref{alg:mmflearnsp} \&~\ref{alg:detfbspmodel},
Algorithms~\ref{alg:mmflearnsp} \&~\ref{alg:detfbmodel},
and~\ref{alg:mmflearnnsp} \&~\ref{alg:detfbmodel},
\item The two methods in Section~\ref{sec:glmfbmodel} using parametric feedback,
i.e. Algorithms~\ref{alg:mmflearnsp} \&~\ref{alg:glmfbmodel}
and Algorithms~\ref{alg:mmflearnnsp} \&~\ref{alg:glmfbmodel}, when using $\mu(x) = \tanh(x)$.
\item The same parametric models when using an algebraic function
$1 - (1+x)^{-1}$ for $\mu(x)$.
\item The two methods in Section~\ref{sec:nonparfbmodel} using nonparametric feedback,
i.e. Algorithms~\ref{alg:mmflearnsp} \&~\ref{alg:nonparfbmodelpart1}-\ref{alg:nonparfbmodelpart4}
and Algorithms~\ref{alg:mmflearnnsp} \&~\ref{alg:nonparfbmodelpart1}-\ref{alg:nonparfbmodelpart4}.
\end{itemize}

\subsection*{Synthetic Experiments}

Our synthetic experimental set up simulates the web-serving scenario outlined in
Example~\ref{eg:webserving}.
We have agents for whom the unit load is between $10^{-4}$ and $10^{-6}$ and on each
round, the load is chosen uniformly randomly in $[5000, 15000]$.
We perform three experiments, with 5, 10, and 15 agents respectively,
and where the rewards are drawn from payoff functions $\payoffi$ which have
the parametric form in Section~\ref{sec:glmfbmodel}.
For the first two synthetic experiments, we use
$\mu(x)=\tanh(x)$ and $\mu(x)=1-(1+x)^{-1}$ with 
the parameter $\thetatruei$ set based on the unit load for each agent's model.
In the third synthetic experiment, we use $\mu(x) = 1/(1+e^{-(x-b)})$ (logistic
function) where $b$ is chosen so that it is $0.6 \times$ the unit demand for each user;
observe that while the first two experiments conform to the parametric model, the third does not.
For \detsps and \detnsp, we directly use the payoff $\payoffi(a) = \mu(\thetatruei a)$
as the feedback,
whereas for the stochastic methods we use stochastic feedback.

\insertFigExpResults

The results are given in Figures~\ref{fig:tanh5}-\ref{fig:logistic15}.
As expected, assigning proportional to entitlements on each round performs poorly and has linear
loss.
As indicated in our analysis, Algorithm~\ref{alg:mmflearnnsp} does better than
Algorithm~\ref{alg:mmflearnsp} for the same feedback model.
While the parametric models outperform entitlement-based allocation in all experiments,
it can suffer when the model is misspecified.
In contrast, the nonparametric models perform well across all the experiments as it does not make
strong assumptions about the payoffs.
%
Finally, the deterministic methods do better than the stochastic methods since they observe the
payoff without noise.

In Figure~\ref{fig:logistic15}, while some of the methods based on Algorithm~\ref{alg:mmflearnnsp}
perform worse than simply allocating in proportion to the entitlements, we see that the
loss grows sublinearly.
The large loss in the initial rounds is due to the large exploration phases when there are many
agents.
It is also worth observing that in Figure~\ref{fig:tanh5}, the parametric method using a
$\tanh$ function for $\mu$ performs worse than the nonparametric method, even though the true payoff
is a $\tanh$ function.
This could be due to the fact that the confidence intervals may be somewhat conservative (see
below).

We wish to highlight that the tree-based procedure for the nonparametric model is computationally
efficient in practice.
For example, in Algorithm~\ref{alg:mmflearnnsp}, in the second synthetic
experiment with 10 users,
the average time taken to obtain a recommendation for a user was
$\sim 0.0011$s after 100 rounds (i.e. 100 data points in the tree),
$\sim 0.0037$s after 1000 rounds,
and $\sim 0.0384s$ after 10000 rounds.
The procedure becomes more expensive in later rounds since the tree is expanded as we collect more
data.
However, since we expand a node only after the sum of inverse variances exceeds $\tauht$---which grows
at rate $4^h$ (see~\eqref{eqn:betataudefn})---it does not expand very fast.

\subsection*{A Prediction-serving task}

We evaluate our approach on latency-sensitive prediction serving~\citep{crankshaw2018inferline},
which is used in a applications such as Amazon Alexa.
Here, each user deploys a queued serving system that takes query inputs and
returns prediction responses.
In this setting, although the application owner
knows how to quantify strict performance requirements from the application, the appropriate resource
allocation is far less clear a priori, due to the complexity of the system.

In our experiment we consider fives users sharing 100 virtual machines, with equal entitlement to
this resource. Each
user specifies a 100ms, 0.95 percentile latency SLO target for their application response time,
meaning that they allow at most 5\% of queries in incoming traffic to have response time greater than
100ms.
At the end of each round, users provide feedback on the fraction of queries that were completed on
time.
For the arrival traffic of three of the users we use the Waikato network
dataset~\citep{mcgregor2000nlanr}, with
different time-of-day and day-of-week regions for the different users.
For one user, we use data from the Twitter streaming API~\citep{twitter2020},
and for the one user, we use the Wikipedia traffic
data~\citep{wikipedia2020}.
We use a query-level execution simulator from~\citep{crankshaw2018inferline}, which uses
power-of-2-choices load balancing to mimic real deployment conditions.

The results are shown in Figure~\ref{fig:predserve}; we obtained ground truth by exhaustively
profiling the performance of each user for all values for the number of virtual machines, which enables
us to numerically compute the unit demands.
As expected, all methods outperform allocating resources in proportion to the entitlements, with
the nonparametric model performing the best.
The parametric models perform poorly in this experiment, which could be
attributed to a mismatch between the model and the problem.

\subsection*{Some Implementation Details}

For the parametric methods, we solved for $\thetait$ in~\eqref{eqn:thetaitdefn}, by computing
$\thetamlit$ as outlined in~\ref{sec:proofs_glm}, and then clipping it at $\thetamin$.
However, in practice we found $\thetamlit>\thetamin$ in almost all cases.
We solved for $\thetamlit$, using the Newton-Raphson method.

In all our experiments, we used $\delta=10^{-3}$.
We found the theoretical value for $\betat$~\eqref{eqn:betatglm} as described in
Section~\ref{sec:glmfbmodel} for the parametric model to be too
conservative; therefore, we divided it by $5$.
This value was tuned using a hold-out set of synthetic experiments that were not included
in Figure~\ref{fig:synthetic}.
In the bandit literature,
it is common to tune upper confidence bounds in a similar
manner~\citep{filippi2010parametric,kandasamy15addBO,srinivas10gpbandits}.

\section{Conclusion}
\label{sec:conclusion}

We studied mechanisms for a multi-round fair allocation problem when agents do not know their
resource requirements, but can provide feedback on an allocation assigned to them.
We proposed three feedback models for this problem, and described mechanisms for each model that
achieved varying degrees of strategy-proofness.
In all cases, we provided upper bounds on the asymptotic rate for efficiency, fairness, and
strategy-proofness, and observed that as we relaxed the strategy-proofness constraints, the rates
for efficiency and fairness improved.
These insights are backed up by empirical evaluations on range of synthetic and real
benchmarks.

One avenue for future work is to explore hardness results for this problem.
In particular, while the $\bigOtilde(\sqrtT)$ rates for the loss are not surprising for the
stochastic models when using Algorithm~\ref{alg:mmflearnnsp}, it is worth
exploring lower bounds for asymptotic fairness and strategy-proofness, and more interestingly
how fairness and strategy-proofness constraints affect the rates for the loss.

\vspace{0.1in}

\paragraph{Acknowledgements:}
We would like to thank Matthew Wright for providing feedback on an initial draft of this manuscript.

\vspace{0.5in}

\appendix

\textbf{\Large Appendix}

\vspace{0.2in}

This Appendix is organised as follows.
In Appendix~\ref{sec:pfmmf}, we prove Theorem~\ref{thm:mmf}.
We also establish some properties about \mmfs which will be useful in subsequent proofs.
In Appendix~\ref{sec:proofs_intermediate}, we will state and prove some intermediate results
that will be useful throughout our analyses of the learning problem in all three models.
Appendix~\ref{sec:proofs_det} analyses the deterministic feedback model,
Appendix~\ref{sec:proofs_glm} analyses the stochastic parametric model, and
Appendix~\ref{sec:proofs_nonpar} analyses the nonparametric model.


\section{Properties of Max-min Fairness and Proof of Theorem~\ref{thm:mmf}}
\label{sec:pfmmf}

In this section we state and prove some properties about Max-min fairness (\mmf) that will be used
in our proofs.
Recall that \mmfs is outlined in Algorithm~\ref{alg:mmf}.
We will let $r$ and $e$ be the variables in Algorithm~\ref{alg:mmf}, which are initialised
in line~\ref{lin:reinit} and then updated in line~\ref{lin:reupdate}.

\insertprethmspacing
\begin{property}
\label{pro:lessthanentitl}
If agent $i$ reports $\demi < \entitli$, her allocation is $\demi$.
\end{property}
\begin{proof}
In Algorithm~\ref{alg:mmf}, $\demj<\entitlj$ for all agents $j$ before user $i$ in the sorted order;
therefore, $r/e$ is increasing until it reaches agent $i$.
Since $r/e=1$ at the beginning,
we have $r/e\geq 1$ whenever we reach the \textbf{if} condition
in line~\ref{lin:mmfifcondn} for all users up to $i$.
Therefore, $\demj/\entitlj < r/e$ for all agents $j$ until $i$.
Hence, the for loop does not break,
and $i$  is allocated $\demi$ in line~\ref{lin:allocdemand}.
\end{proof}

\insertprethmspacing
\begin{property}
\label{pro:morethanentitl}
If user $i$ reports $\demi \geq \entitli$, her allocation is at least $\entitli$.
\end{property}
\begin{proof}
We will first show, by way of induction, that in Algorithm~\ref{alg:mmf}, $r\geq e$ each time
we visit line~\ref{lin:mmfifcondn}.
For the first user, this is true since $r=e=1$.
Now assume that it is true when the \textbf{if} condition is satisfied for a user $j$. Then,
$\demj/\entitlj < r/e$. We therefore have,
\begin{align*}
\frac{r-\demj}{e-\entitlj} \;>\;
\frac{r-r\entitlj/e}{e-\entitlj} \;=\;
\frac{r}{e}\cdot\frac{1-\entitlj/e}{1-\entitlj/e} \;\geq\; 1.
\label{eqn:mmfifcondn}
\numberthis
\end{align*}
Therefore, the statement is true when we visit line~\ref{lin:mmfifcondn} the next time.

Now, in  Algorithm~\ref{alg:mmf}, if a user was allocated her resource in line~\ref{lin:allocdemand},
then $\alloci = \demi \geq \entitli$.
Therefore, say she was allocated in line~\ref{lin:allocsaturate}.
When the condition in line~\ref{lin:mmfifcondn} is violated for some user, $r/e\geq 1$
by the above argument in~\eqref{eqn:mmfifcondn}.
Therefore, for all users $j$  who are assigned in~\ref{lin:allocsaturate},
$\allocj = \entitlj \frac{r}{e} \geq \entitlj$.
\end{proof}

\insertprethmspacing
\begin{property}
\label{pro:allocnomorethandemand}
A user's allocation is never more than her reported demand.
\end{property}
\begin{proof}
Let user $j$'s demand be $\demj$ and consider any user $i$.
If she is allocated in line~\ref{lin:allocdemand}, then she is allocated $\demi$.
If she is allocated in line~\ref{lin:allocsaturate}, then she is allocated
$r\frac{\entitli}{\entitl}$ which is at most $\demi$ by the \textbf{if} condition in
line~\ref{lin:mmfifcondn} and the fact that agents are sorted in ascending order of $\demj/\entitlj$.
\end{proof}

\insertprethmspacing
\begin{property}
\label{pro:alloclessthandemand}
Suppose we have a demand vector $\dem$ and allocation vector $\alloc$ returned by \emph{\mmf}, where,
for user $i$, $\alloci < \demi$.
Keeping all other reported demands constant, user $i$'s allocation would have been the same 
for all of her reports $\demi'\geq\alloci$.
\end{property}
\begin{proof}
If $\alloci < \demi$, then she was allocated $\alloci = r\entitli/e$ at
line~\ref{lin:allocsaturate} in Algorithm~\ref{alg:mmf}.
Therefore, for any bid greater than or equal to $\alloci$, while her ranking may have changed, she
would still have been allocated the same amount in line~\ref{lin:allocsaturate}, as the
\textbf{if} condition in line~\ref{lin:mmfifcondn} is not satisfied.
\end{proof}

\insertprethmspacing
\begin{property}
\label{pro:allocreportlessthandemand}
Fix the reported demands of all agents except $i$.
Let the allocations of agent $i$ when reporting $\demi, \demi'$ be
$\alloci$ and  $\alloci'$ respectively.
If $\demi<\demi'$, then $\alloci\leq \alloci'$ with equality holding only when the agent is
allocated in line~\ref{lin:allocsaturate} when reporting both $\demi$ and $\demi'$.
\end{property}
\begin{proof}
First consider the case where
the agent was allocated in line~\ref{lin:allocsaturate} when reporting $\demi$.
Then, $\alloci\leq\demi$.
Using the same argument used in the proof of Property~\ref{pro:alloclessthandemand}, we have
that her allocation would have been the same for all demands larger than $\demi$, including,
in particular $\demi'$. Therefore, $\alloci=\alloci'$.

Second, consider the case where
she was allocated in line~\ref{lin:allocdemand} when reporting $\demi$.
If
she would have been allocated in line~\ref{lin:allocdemand}, had she reported $\demi'$ instead of
$\demi$, then $\alloci=\demi<\demi'=\alloci'$.
Suppose
instead, that she would have been allocated in line~\ref{lin:allocsaturate} when reporting $\demi'$.
If $r', e'$ are the values of $r, e$  in Algorithm~\ref{alg:mmf} when
the \textbf{if} condition is first not satisfied,
 she would have received $r' \entitli/e'$ when reporting $\demi'$.
We will show, by way of contradiction, that $\demi<r' \entitli/e'$ which proves the
property.

To show the contradiction,
assume instead $\demi \geq r' \entitli/e'$.
Then, when $i$ reports $\demi$, the \textbf{if} condition is violated when either user $i$ or an
agent before user $i$ in the ordering reaches the condition.
Then, she will have been allocated at line~\ref{lin:allocsaturate} which contradicts the premise
of the case, which is that $i$ was allocated in line~\ref{lin:allocdemand} when reporting $\demi$.
\end{proof}

\insertprethmspacing
\begin{property}
\label{pro:reportlessthandemand}
Fix the reported demands of all agents except $i$.
Let the utilities of agent $i$ when reporting $\demi, \demtruei$ be
$\tilde{u}_i$ and  $\tilde{u}^\star_i$ respectively.
If $\demi<\demtruei$, then $\tilde{u}_i\leq\tilde{u}^\star_i$.
\end{property}
\begin{proof}
This follows from monotonicity of the utility and Property~\ref{pro:allocreportlessthandemand},
i.e. the fact that an agent's allocation cannot decrease when she increases
her bid, when all other bids are unchanged.
\end{proof}

\insertprethmspacing
\begin{property}
\label{pro:reportmorethandemand}
Fix the reported demands of all agents except $i$.
Let the utilities of agent $i$ when reporting $\demi, \demtruei$ be
$\tilde{u}_i$ and  $\tilde{u}^\star_i$ respectively.
If $\demi>\demtruei$, then $\tilde{u}_i=\tilde{u}^\star_i$.
\end{property}
\begin{proof}
Assume the allocations were $\alloctruei,\alloci$ when the user reports $\demtruei,\demi$
respectively.
By Property~\ref{pro:allocnomorethandemand}, $\alloctruei \leq \demtruei$.
First assume that the agent was allocated $\alloctruei < \demtruei$.
By Property~\ref{pro:alloclessthandemand}, we have
$\alloci=\alloctruei$ which implies $\tilde{u}_i=\tilde{u}^\star_i$.
Now say the agent was allocated $\alloctruei = \demtruei$, and it changes to $\alloci$ when
she switches to $\demi$.
By Property~\ref{pro:allocreportlessthandemand},
the allocation by \mmfs does not decrease when an agent increases her demand.
Therefore, $\alloci\geq\alloctruei$.
The claim follows from the fact that $\utili(a) = \utili(\alloctruei)$ for all $a\geq \alloctruei$.
\end{proof}

We can now use the above properties to prove Theorem~\ref{thm:mmf}.

\insertprethmspacing
\textbf{\textit{Proof of Theorem~\ref{thm:mmf}.}}
\textbf{Efficiency:}
Assume all users report truthfully.
Recall that a mechanism is efficient if
$\lot(\demtrue, \alloc) = \min(\lotur(\alloc) + \lotor(\demtrue,\alloc), \lotud(\demtrue, \alloc)) =
0$ where $\lotur, \lotor, \lotud$ are as defined in~\eqref{eqn:lotdefn}.
Since \mmfs never allocates a user more than her reported demand
(Property~\ref{pro:allocnomorethandemand}), $\lotor(\demtrue,\alloc)=0$.
Therefore, if $\lotur(\alloc) = 0$, then $\lot(\demtrue, \alloc) = 0$.
If $\lotur(\alloc) > 0$, this means Algorithm~\ref{alg:mmf} never entered
line~\ref{lin:allocsaturate} and therefore $\alloci=\demtruei$ for all users $i$.
Hence, $\lot(\demtrue,\alloc) \leq \lotud(\demtrue,\alloc) = 0$.

\textbf{Fairness:}
Assume user $i$ reports truthfully.
If $\demtruei < \entitli$, then by Property~\ref{pro:lessthanentitl},
$\utili(\entitli) = \utili(\demtruei) = \utili(\alloci)$.
If $\demtruei \geq \entitli$, then by Property~\ref{pro:morethanentitl} and the fact that
$\utili$ is non-decreasing, $\utili(\entitli) \leq \utili(\alloci)$.

\textbf{Strategy-proofness:}
This follows from Properties~\ref{pro:reportlessthandemand} and~\ref{pro:reportmorethandemand}.
\qedwhite


\section{Some Intermediate Results}
\label{sec:proofs_intermediate}

In this section, we will state some intermediate lemmas that will be useful in the
proofs of Theorems~\ref{thm:detsp}--\ref{thm:nonparnsp}.
We begin with some notation.

\textbf{Instantaneous allocations and losses:}
We will let $\normallocit=\allocit/\volit$ be the allocation per unit demand for agent $i$
at round $t$.
Let $\lott=\lot(\demtrueit, \alloct)$ be the loss at round $t$ and
$\loturt=\lotur(\alloct)$,
$\lotort=\lotor(\demtrueit, \alloct)$,
and
$\lotudt=\lotud(\demtrueit, \alloct)$ be the unallocated resources, over allocated resources, and
unmet demand respectively at round $t$~\eqref{eqn:lotdefn}.
Recall, $\lott = \min(\loturt+\lotort, \lotudt)$.

\textbf{Upper/lower (confidence) bounds on the demands:}
We will use $\udubit$ and $\udlbit$ to denote upper and lower (confidence) bounds on
agent $i$'s unit demands $\udtruei$ at round $t$,
i.e. they satisfy $\udlbit \leq \udtruei \leq \udubit$ (with high probability).
Moreover, unless otherwise specified,
 $\demubit=\volit\udubit$ and $\demlbit=\volit\udlbit$ will denote upper and
lower (confidence) bounds on
agent $i$'s demand $\demtrueit$ at round $t$.

\textbf{Exploration phase \& bracket indices in Algorithm~\ref{alg:mmflearnsp}:}
We will use $\Expl$ to denote the round indices belonging to
the exploration phase in Algorithm~\ref{alg:mmflearnsp}.
We will use $\qt$ to denote the bracket index round $t$ belongs to and
$T_q$  to denote the number of rounds completed by $q$ brackets.
Then,
\begin{align*}
\numberthis
\label{eqn:qtdefn}
T_{\qt-1} < t \leq T_{\qt}.
\end{align*}

Finally, it is worth recalling the definition of a sub-Gaussian distribution. 
A random variable $X$ is $\sigma$ sub-Gaussian if,
\begin{align*}
\text{for all $\lambda>0$, } \hspace{0.2in}
\EE\left[ e^{\lambda (X-\EE X)} \right]
\leq \exp\left(\frac{\lambda^2\sigma^2}{2}\right).
\label{eqn:subgauss}
\numberthis
\end{align*}



\subsection{Bounds on the Loss}
\label{app:allocintermediate}

The following lemma will be useful in bounding the loss in round $t$ for
Algorithm~\ref{alg:mmflearnsp}.

\insertprethmspacing
\begin{lemma}
\label{lem:ubinstloss}
Suppose on round $t$, a multi-round mechanism chose its allocations via \emph{\mmfs} by using an upper
bound $\demubit$ for $\demtrueit$ for all $i\in\{1,\dots,n\}$ as the reported demand,
i.e. $\demubit\geq\demtrueit$.
Moreover, let $\demlbit$ be a lower bound for
$\demtrueit$ for all $i$,
i.e. $\demlbit\leq\demtrueit$.
Then,
\[
\lott \leq \sum_{i=1}^n (\demubit - \demtrueit) \leq \sum_{i=1}^n (\demubit - \demlbit).
\]
\end{lemma}
\begin{proof}
The second inequality follows from the first since $\demtrueit\geq\demlbit$.
To prove the first inequality, first let $\loturt>0$.
This means, in Algorithm~\ref{alg:mmf}, all agents were allocated in line~\ref{lin:allocdemand}
and therefore $\allocit=\demubit>\demtrueit$ for all $i$.
Hence $\lott \leq \lotudt=0$.
The statement is true since $\demubit\geq\demtrueit$.

Now let $\loturt=0$.
We bound $\lott \leq \lotort=\sum_i (\allocit - \demtrueit)^+$.
Since the mechanism reports $\demubit$ as the demand to \mmf, and since \mmfs does not
allocate more than the reported demand (Property~\ref{pro:allocnomorethandemand}),
we have $\allocit\leq\demubit$.
\end{proof}

The following lemma will be useful in bounding the loss in round $t$ for
Algorithm~\ref{alg:mmflearnnsp}.

\insertprethmspacing
\begin{lemma}
\label{lem:nsploss}
Suppose on round $t\geq 2$ of Algorithm~\ref{alg:mmflearnnsp}, we chose the allocations
via \emph{\mmfs}
by using a reported demand $\demit$ for all $i\in\{1,\dots,n\}$. 
Then,
\[
\LOTT 
\leq 1 + \sum_{i=1}^n \sum_{t=2}^T (\allocit-\demtrueit)^+ \,+\,
    \sum_{i=1}^n \,\sum_{t= 2}^T \indfone(\allocit=\demit)(\demtrueit-\allocit)^+.
\]
\end{lemma}
\begin{proof}
First consider the loss at round $t$, where, recall
$\lott = \min(\loturt+ \lotort, \lotudt)$.
If $\loturt > 0$, we will bound $\lott\leq\lotudt$ and if $\loturt=0$, we will bound
$\lott\leq\lotort$.
Intuitively, if the resource is scarce we will show that we have not been inefficient by
over-allocating to users, and if there are excess resources, we will show that we have not
under-allocated to anyone.
This leads us to,
\begingroup
\allowdisplaybreaks
\begin{align*}
\LOTT &= \sum_{t=1}^T\lott \leq 1 + \sum_{t\geq 2, \loturt=0} \lotort +
                            \sum_{t\geq 2, \loturt>0} \lotudt
\\ 
&= 1 + \sum_{t=2, \loturt=0}^T \sum_{i=1}^n(\allocit-\demtrueit)^+ \,+\,
    \sum_{t\geq 2, \loturt>0} \sum_{i=1}^n(\demtrueit-\allocit)^+ \\
&\leq 1 + \sum_{i=1}^n \sum_{t=2}^T (\allocit-\demtrueit)^+ \,+\,
    \sum_{i=1}^n \,\sum_{t= 2, \,\allocit=\demit}^T \hspace{-0.1in}(\demtrueit-\allocit)^+
\end{align*}
\endgroup
Here, the second step uses the definitions for $\loturt,\lotudt$.
The last step uses two relaxations.
First, we remove the constraint $\loturt=0$ in the first summation.
In the second summation, we use the fact that
if there are unallocated resources, it can only be when \mmfs allocates to all users their requested
demand;
therefore,
$\loturt>0$ implies $\allocit=\demit$ for all $i$.
\end{proof}

\subsection{Bounds on Fairness}
\label{sec:intermediatefairness}

The following two lemmas will be useful in the proofs of our fairness results for
Algorithm~\ref{alg:mmflearnsp}.
Recall from Section~\ref{sec:onlinefairdivision},
user $i$'s utility $\utili$ is $\Lipi$-Lipschitz continuous.

\insertprethmspacing
\begin{lemma}
\label{lem:ubfairness}
Suppose on round $t$, the allocations $\alloct$ are set via \mmfemph, and that for agent $i$ we used
a reported demand $\demit=\udit\volit$.
Say she received an allocation $\allocit$.
Let $\udubit$ be such that $\udubit\geq \max(\udtruei,\udit)$.
Then, regardless of the behaviour of the other agents, we have for agent $i$,
\begin{align*}
\utili(\entitli/\volit) - \utili(\allocit/\volit)
\leq \Lipi(\udubit - \udit).
\end{align*}
\end{lemma}
\begin{proof}
Denote $\demubit=\udubit\volit$ and $\demtrueit=\udtruei\volit$.
First observe that if $\demit \geq \entitli$, by Property~\ref{pro:morethanentitl},
we have
$\utili(\entitli/\volit) - \utili(\allocit/\volit) \leq 0$ and the statement is true since
the RHS is positive.
Now say $\demit<\entitli$.
By Property~\ref{pro:lessthanentitl}, $\allocit=\demit$ and therefore,
\begin{align*}
\utili(\entitli/\volit) - \utili(\allocit/\volit) 
&\leq \utili(\entitli/\volit) - \utili(\demtrueit/\volit) + \utili(\demtrueit/\volit) -
    \utili(\demit/\volit) \\
&\leq \Lipi(\demtrueit/\volit - \demit/\volit)^+
 \leq \Lipi(\udubit - \udit).
\end{align*}
Above, we have used
$\utili(\entitli/\volit) - \utili(\demtrueit/\volit) \leq 0$ since $\udtruei$ maximises $\utili$ and
the fact that $\utili$ is $\Lipi$ Lipschitz and increasing.
The last step uses the fact that $\udubit\geq\udtruei$.
\end{proof}

\insertprethmspacing
\begin{lemma}
\label{lem:spfairness}
Suppose we use upper bounds $\{\udubit\}_{i=1}^n$ on the unit demand in line~\ref{lin:spallocub} of
Algorithm~\ref{alg:mmflearnsp} on round $t$ and chose the allocations via \mmfemph.
Assume agent $i$ was truthful.
Let $\qT$ be as defined in~\eqref{eqn:qtdefn} and let $r$ denote the number of rounds per
exploration phase.
Then, regardless of the behaviour of the other agents, we have for agent $i$,
\emph{
\[
\UeiiT - \UiT \leq r\Lipi\udmax\qT.
\]
}
\end{lemma}
\begin{proof}
We decompose the sum of utilities into $t\in\Expl$ and $t\notin\Expl$ and apply
Lemma~\ref{lem:ubfairness} to obtain,
\begin{align*}
\UeiiT - \UiT &= \sum_{t\in\Expl} (\utili(\entitli/\volit) - \utili(\allocit/\volit))
     + \sum_{t\notin\Expl} \left(\utili(\entitli/\volit) - \utili(\allocit/\volit)\right) \\
    &\leq \sum_{t\in\Expl} (\utili(\entitli/\volit) - \utili(0))
    \leq \sum_{t\in\Expl} (\utili(\udtruei) - \utili(0))
    \leq r\Lipi\udmax\qT.
\end{align*}
Above, $\utili(\entitli/\volit) - \utili(\allocit/\volit) \leq 0$ for $t\notin\Expl$
by applying Lemma~\ref{lem:ubfairness} with $\udit$ and $\udubit$ both set to the value $\udubi$ in
line~\ref{lin:spallocub} of Algorithm~\ref{alg:mmflearnsp} (which is an upper bound on $\udtruei$).
The third step uses that $\udtruei$ maximises $\utili$.
The last step uses Lipschitzness of agent $i$'s utility, the fact that $\udtruei\leq\udmax$,
 and that in $T$ rounds there will
have been $r\qT$ exploration rounds.
\end{proof}

The following lemma will be useful in the proofs of our fairness results for
Algorithm~\ref{alg:mmflearnnsp}.

\insertprethmspacing
\begin{lemma}
\label{lem:nspfairness}
Suppose on round $t$ of Algorithm~\ref{alg:mmflearnnsp}, the allocations $\alloct$ are set via
\mmfemph, and that for agent $i$ we used
a reported demand $\demit=\udit\volit$.
Then, regardless of the behaviour of the other agents, we have for agent $i$,
\begin{align*}
\UeiiT - \UiT
\leq \sum_{t=1}^T 
    \indfone(\normallocit=\udit\,\wedge\,\normallocit<\udtruei) \cdot \left(\utili(\udtruei) - \utili(\udit)\right)
\end{align*}
\end{lemma}
\begin{proof}
Observe that $\utili(\entitli/\volit) < \utili(\allocit/\volit)$ only when
$\allocit < \min(\entitli, \demtrueit)$;
if $\allocit\geq \entitli$, then the agent's utlity will be larger than when just using her
entitlement  since $\utili$ is non-decreasing.
Further, $\utili(\udtruei) = \utili(\eta)$ for all $\eta \geq \udtruei$.
Accounting for this, we can bound $\UeiiT -\UiT$ as shown below.
\begingroup
\allowdisplaybreaks
\begin{align*}
\UeiiT - \UiT &=
\sum_{t=1}^T \left(\utili\bigg(\frac{\entitli}{\volit}\bigg) -
    \utili\bigg(\frac{\allocit}{\volit}\bigg)\right) \\
&\leq \sum_{t:\allocit\leq\min(\entitli,\demtrueit)} 
    \left(\utili\bigg(\frac{\entitli}{\volit}\bigg) -
          \utili\bigg(\frac{\demtrueit}{\volit}\bigg) +
          \utili\bigg(\frac{\demtrueit}{\volit}\bigg) -
          \utili\bigg(\frac{\udit\volit}{\volit}\bigg)\right)
\\
&\leq \sum_{t:\allocit\leq\min(\entitli,\demtrueit)} 
    \left(\utili(\udtruei) - \utili(\udit)\right)
\\
&\leq \sum_{t=1}^T
    \indfone(\normallocit=\udit\,\wedge\,\normallocit<\udtruei) \cdot \left(\utili(\udtruei) - \utili(\udit)\right)
\end{align*}
\endgroup
Above, the second step adds and subtracts $\utili(\udtruei)=\utili(\demtrueit/\volit)$
while also observing $\allocit=\demit=\udit\volit$ by Property~\ref{pro:lessthanentitl}.
The third step uses the fact that
$\utili(\demtrueit/\volit) = \utili(\udtruei) \geq \utili(a)$ for all $a$.
In the fourth step, we have used Properties~\ref{pro:lessthanentitl},~\ref{pro:morethanentitl},
and~\ref{pro:allocnomorethandemand} to conclude that
$\udit=\normallocit$ when $\normallocit\leq\entitli$.
\end{proof}

\subsection{Strategy-proofness}

Our next lemma will be useful in establishing strategy-proofness for
Algorithm~\ref{alg:mmflearnsp}.
For this, consider any agent $i$ and fix the behaviour of all other agents.
Let $\UiT, \UpiiT$ respectively denote the sum of utilities when $i$ participates truthfully and
when she is following any other (non-truthful) policy $\pi$.
Let $\{\allocit\}_{t}, \{\allocpiit\}_{t}$ respectively denote the sequence of allocations
for agent $i$ when she is adopting these strategies.
Let $\{\alloctdit\}_t$ denote the allocations when agent $i$
follows $\pi$ from rounds $1$ to $t-1$, but her allocation for round $t$ is based on her
true unit demand;
in Algorithm~\ref{alg:mmflearnsp}, this means we use $\udubit = \udtruei$ in
line~\ref{lin:spallocub}.
Note that in general, $\alloctdit$ depends on previous allocations when following $\pi$,
since it will have also affected the allocations and consequently the estimates of the unit demands
for the other agents, which in turn will affect the allocation $\alloctdit$ chosen by \mmfs at
round $t$ for agent $i$.

Our next Lemma is useful for establishing strategy-proofness results
for  Algorithm~\ref{alg:mmflearnsp}.
Recall that in the exploration phase, the allocations for one agent are not affected by
reports from other agents.

\insertprethmspacing
\begin{lemma}
\label{lem:spsp}
Consider Algorithm~\ref{alg:mmflearnsp} and assume that allocations for any agent in the
exploration phase are chosen independent of the reports by other agents.
Suppose $\udubi$ in line~\ref{lin:spallocub} is an upper bound on $\udtruei$.
Then, for all policies $\pi$ and all $T\geq 1$,
\[
\UpiiT - \UiT
\leq \sum_{t=1}^T \indfone(t \in \Expl)
        \left(\utili(\allocpiit/\volit) - \utili(\allocit/\volit) \right).
\]
\end{lemma}
\begin{proof}
We first decompose $\UpiiT - \UiT$ as follows,
\begin{align*}
\UpiiT - \UiT
&= \sum_{t\in \Expl} \left(\utili(\allocpiit/\volit) - \utili(\allocit/\volit) \right)
 + \sum_{t\notin \Expl} \left(\utili(\allocpiit/\volit) - \utili(\allocit/\volit) \right)
\\
&\leq \sum_{t\in \Expl} \left(\utili(\allocpiit/\volit) - \utili(\allocit/\volit) \right)
 + \sum_{t\notin \Expl} \left(\utili(\alloctdit/\volit) - \utili(\allocit/\volit) \right)
\end{align*}
Here, the last step uses Properties~\ref{pro:reportlessthandemand}
and~\ref{pro:reportmorethandemand} to conclude that an agent's utility is
maximised when reporting her true demand.
Agent $i$'s allocation when using her true demand is $\alloctdit$ and therefore
$\utili(\alloctdit/\volit) \geq \utili(\allocpiit/\volit)$.

We now argue that each term in the second sum is $0$ which will establish the claim.
Recall that Algorithm~\ref{alg:mmflearnsp}
only uses values in the exploration phase to determine allocations,
and moreover the allocations in the exploration phase for agent $j\neq i$
are chosen independent of the reports by the agent $i$.
Therefore, the value used for $\udubj$ in line~\eqref{lin:spallocub} is the same for all agents
$j\neq i$ when computing $\alloctdit$ and $\allocit$.
By Property~\ref{pro:reportmorethandemand},
we have $\utili(\alloctdit/\volit) - \utili(\allocit/\volit)=0$ since $\udubi$ is an upper
bound on $\udtruei$.
\end{proof}




%

\subsection{Bounding the number of Exploration Phases in Algorithm~\ref{alg:mmflearnsp}}
\label{app:explphases}

Lemmas~\ref{lem:qTdetsp},~\ref{lem:qTdet}, and~\ref{lem:qTglm}
will help us lower and upper bound the number of brackets $\qT$~\eqref{eqn:qtdefn}
after $T$ rounds in Algorithm~\ref{alg:mmflearnsp}.

\insertprethmspacing
\begin{lemma}
\label{lem:qTdetsp}
Suppose we execute Algorithm~\ref{alg:mmflearnsp} with $n$ exploration phase rounds
in each bracket followed by $r'(q) = nq$ rounds for the latter phase.
Then, 
$\frac{1}{\sqrt{3}}\nmontw\sqrtT\leq \qT \leq 2\sqrt{2}\nmontw\sqrtT$.
\end{lemma}
\begin{proof}
The claim can be easily verified for $T\leq 2n$ so that $\qT = 1$.
Therefore, let $T\geq 2n+1$ and 
for brevity, write $q=\qT$. 
We have,
$
T_{q-1} < T \leq \Tq, 
$
where,
$
T_{m} = mn + \sum_{t=1}^m t.
$
Letting $S_m = \sum_{t=1}^m t^d$ and bounding the sum of an increasing function
by an integral we have,
\begin{align*}
\int_0^m t^d\ud t < S_m < \int_1^{m+1} t^d\ud t
\hspace{0.2in}
\implies
\hspace{0.2in}
\frac{m^{d+1}}{d+1} < S_m < \frac{ (m+1)^{d+1} - 1}{d+1}.
\label{eqn:Smbound}
\numberthis
\end{align*}
(We will use~\eqref{eqn:Smbound} again below in Lemma~\ref{lem:qTglm} with different $d$ values.)
Setting $d=1$ leads to the following bounds on $T$,
\begin{align*}
\numberthis\label{eqn:qTlbdetssp}
T &\leq T_q \leq qn + \sum_{s=1}^q ns \leq qn + \frac{n}{2}\big((q+1)^{2} -1\big)
 \leq nq^{2} + \frac{n}{2}(2q)^{2}
 \leq 3n q^{2}, \\
\numberthis\label{eqn:qTubdetssp}
T &\geq T_{q-1} \geq n(q-1) + \sum_{t=1}^{q-1} nt \geq
     \frac{n}{2}(q-1)^{2}
    \geq  \frac{n}{2}\left(\frac{q}{2}\right)^{2}
    = \frac{n}{8}q^{2}.
\end{align*}
In~\eqref{eqn:qTlbdetssp}, we have used the upper bound in~\eqref{eqn:Smbound} with $m=q$, and the
facts $q\leq q^{2}$, $q+1 \leq 2q$.
In~\eqref{eqn:qTubdetssp}, we have used the lower bound in~\eqref{eqn:Smbound} with $m=q-1$, and
that $q-1\geq q/2$ when $q\geq2$ which is true
when $T\geq 2n+1$.
We therefore have,
$q \leq 2\sqrt{2}\nmontw T$ and
$q \geq \frac{1}{\sqrt{3}}\nmontw \sqrtT$.
\end{proof}

\insertprethmspacing
\begin{lemma}
\label{lem:qTdet}
Suppose we execute Algorithm~\ref{alg:mmflearnsp} with $2n$ exploration phase rounds
in each bracket followed by $r'(q) = \lfloor e^q\rfloor$ rounds for the latter phase.
Then,
\[
\log\left(\frac{T}{2n+e}\right)
\leq \qT \leq
1+\log\left( T\right).
\]
\end{lemma}
\begin{proof}
The statement can be verified easily for $T\leq2n+2$  since $\qT=1$.
Therefore let $T\geq 2n+3$.
For brevity, write $q=\qT$.
Using the notation in~\eqref{eqn:qtdefn}, we first observe
$
T_{q-1} < T \leq \Tq, 
$
where
$
T_{m} = 2nq + \sum_{t=1}^m \lfloor e^t \rfloor.
$
Letting $S_m = \sum_{t=1}^m e^t$ and bounding the sum of an increasing function
by an integral we have,
\begin{align*}
\int_0^m e^t\ud t < S_m < \int_1^{m+1} e^t\ud t
\hspace{0.2in}
\implies
\hspace{0.2in}
 (e^m - 1) < S_m < e(e^m - 1).
\end{align*}
This leads to the following bounds on $T$,
\begingroup
\allowdisplaybreaks
\begin{align*}
T &\leq T_q \leq 2nq + \sum_{t=1}^q e^t \leq 2nq + e(e^q - 1) \leq (2n+e) e^q,
\\
T &\geq T_{q-1} \geq 2n(q-1) + \sum_{t=1}^{q-1} (e^t - 1)
    \geq (2n-1)(q-1) + e^{q-1} -1 \geq e^{q-1}.
\end{align*}
\endgroup
In the last step, we have used the fact that $(q-1)(2n-1)-1\geq 0$ when $q\geq2$ which is true when
$T\geq 2n+3$.
Inverting the above inequalities yields the claim.
\end{proof}

\insertprethmspacing
\begin{lemma}
\label{lem:qTglm}
Suppose we execute Algorithm~\ref{alg:mmflearnsp} with $a$ exploration phase rounds
in each bracket followed by $r'(q) = \lfloor 5aq/6\rfloor$ rounds for the latter phase.
Then, 
$\frac{1}{2}a^{\nicefrac{-2}{3}}\Ttwth\leq \qT \leq 3a^{\nicefrac{-2}{3}}\Ttwth$.
\end{lemma}
\begin{proof}
The proof will follow along similar lines to the proof of Lemma~\ref{lem:qTdetsp}.
The claim can be easily verified for $T\leq 2a$ so that $\qT=1$.
Therefore, let $T\geq 2a+1$.
For brevity, write $q=\qT$, $c = 5/6$, $d=1/2$.
We have,
$
T_{q-1} < T \leq \Tq, 
$
where,
$
T_{m} = am + \sum_{t=1}^m \lfloor ct^d \rfloor.
$
Using~\eqref{eqn:Smbound}, we have the following bounds on $T$,
\begingroup
\allowdisplaybreaks
\begin{align*}
\numberthis\label{eqn:qTlb}
T &\leq T_q \leq aq + ac\sum_{t=1}^q t^d \leq aq + \frac{ac}{d+1}\big((q+1)^{d+1} -1\big)
\\
 &\leq aq^{d+1} + \frac{ac}{d+1}(2q)^{d+1}
 \leq c_1 a q^{\nicefrac{3}{2}}, \\
\numberthis\label{eqn:qTub}
T &\geq T_{q-1} \geq a(q-1) + \sum_{t=1}^{q-1} (act^d - 1) \geq
    (a-1)(q-1) + \frac{ac}{d+1}(q-1)^{d+1}
\\
    &\geq \frac{ac}{d+1}\left(\frac{q}{2}\right)^{d+1}
    = c_2aq^{\nicefrac{3}{2}}.
\end{align*}
\endgroup
Here $c_1 = 1 + 2^{\nicefrac{5}{2}}c/3$ and $c_2 = c/(3\sqrt{2})$.
In~\eqref{eqn:qTlb}, we have used the upper bound in~\eqref{eqn:Smbound} with $m=q$, and the
facts $q\leq q^{\nicefrac{3}{2}}$, $q+1 \leq 2q$.
In~\eqref{eqn:qTub}, we have used the lower bound in~\eqref{eqn:Smbound} with $m=q-1$, and
that $q-1\geq q/2$ when $q\geq2$ which is true
when $T\geq 2a+1$.
We therefore have,
$q \leq c_2^{-\nicefrac{2}{3}}a^{\nicefrac{-2}{3}}\Ttwth$ and
$q \geq c_1^{-\nicefrac{2}{3}}a^{\nicefrac{-2}{3}}\Ttwth$.
Substituting $c=5/6$ yields the desired result.
\end{proof}

\subsection{Technical Lemmas}
\label{app:techlemmas}

We conclude this Section with some technical results.
The first, given in Lemma~\ref{lem:mds}, is a concentration result that will help us establish
confidence intervals in the second and third models.
It is a corollary of the following result from~\citet{de2004self}.

\insertprethmspacing
\begin{lemma}[\citet{de2004self}, Corollary 2.2]
\label{lem:dela}
Let $A, B$ be random variables such that $A\geq 0$ a.s. and
$\EE\left[e^{\nu B - \frac{\nu^2 A^2}{2}}\right]$ $\leq 1$ for all $\nu\in\RR$.
Then, $\forall\,c\geq 2$, $\PP\left(|B|>cA\sqrt{2+\log(2)}\right) \leq e^{-\nicefrac{c^2}{2}}$.
\end{lemma}

\insertprethmspacing
\begin{lemma}
\label{lem:mds}
Let $\{\filtrs\}_{s\geq 0}$ be a filtration, and $\{\gamma_s\}_{s\geq 1}$,
$\{\sigma_s\}_{s\geq 1}$ be predictable processes such that $\gamma_s\in\RR$ and $\sigma_s>0$.
Let $\{z_s\}_{s\geq 1}$ be a real-valued martingale difference sequence adapted to
$\{\filtrs\}_{s\geq 0}$.
Assume that $z_s$ is conditionally $\sigma_s$ sub-Gaussian, i.e.
$
\forall\,\lambda\geq 0, \quad
\EE\left[e^{\lambda z_s}\Big|\filtrsmo \right] \leq
\exp\left(\frac{\lambda^2\sigma_s^2}{2}\right).
$
Then, for $t\geq 1$ and $\delta\leq 1/e$, the following holds
with probability greater than $1-\delta$.
\[
\left|\sum_{s=1}^{t} \frac{\gamma_s z_s}{\sigma_s^2} \right|
\leq
\sqrt{(4+2\log(2))\log\left(\frac{1}{\delta}\right)
\sum_{s=1}^{t} \frac{\gamma_s^2}{\sigma_s^2}.
}
\]
\end{lemma}
\begin{proof}
We will apply Lemma~\ref{lem:dela} with
$
A^2\leftarrow\sum_{s=1}^{t} \frac{\gamma_s^2}{\sigma_s^2},
$
$
B\leftarrow\sum_{s=1}^{t} \frac{\gamma_s z_s}{\sigma_s^2}
$.
We first need to verify that the condition for the Lemma holds.
Let $\nu\in\RR$ be given.
Writing,
$Q_s =  \frac{\nu\gamma_s z_s}{\sigmasqis}
- \frac{\nu^2}{2}\frac{\gamma_s^2}{\sigma_s^2}$,
we have  $\nu B - \frac{\nu^2 A^2}{2} = \sum_{s=1}^{t}Q_s$.
Then, by the sub-Gaussian property,
\begin{align*}
\EE\left[e^{Q_s}\big|\Fcal_{s-1}\right] =
\exp\left(-\frac{\nu^2}{2}\frac{\gamma_s^2}{\sigma_s^2} \right)
\EE\left[
\exp\left(\frac{\nu\gamma_s}{\sigma_s^2} z_s \right)\bigg|\,\Fcal_{s-1}\right]
\leq 1.
\end{align*}
Here, we have used the fact that $\gamma_s$ and $\sigma_s$ are $\Fcal_{s-1}$ measurable.
This leads us to,
\begin{align*}
\EE\left[e^{\sum_{s=1}^{t}Q_s}\right] =
\EE\left[e^{\sum_{s=1}^{t-1}Q_s}\,\EE\left[e^{Q_{t}}\big|\Fcal_{t-1}\right]\right]
\leq
\EE\left[e^{\sum_{s=1}^{t-1}Q_s}\right]
\leq
\,\dots\,
\leq  1.
\end{align*}
The claim follows by applying Lemma~\ref{lem:dela} with $c=\sqrt{2\log(1/\delta)}$.
The $\delta<1/e$ condition arises from the $c\geq 2$ condition in Lemma~\ref{lem:dela}.
\end{proof}

Finally, we will also need the following inequality.

\insertprethmspacing
\begin{lemma}
\label{lem:logbound}
Let $c>0$. Then, for all $x\in[0, c]$, $x \leq \frac{c}{\log(1+c)}\log(1+x)$.
\end{lemma}
\begin{proof}
Write $f(x) = \frac{c}{\log(1+c)}\log(1+x) - x$.
We need to show $f(x)\geq 0$ for all $x\in[0,c]$.
This follows by observing that $f$ is concave and that $f(0) = f(c) = 0$.
\end{proof}

\section{Proofs of Results in Section~\ref{sec:detfbmodel}}
\label{sec:proofs_det}

In this section, we present our proofs for the deterministic feedback model.

\subsection{Proof of Theorem~\ref{thm:detfbsp}}
\label{sec:detsp}

\textbf{Efficiency:}
We will first decompose the $T$-period loss as follows,
\begin{align*}
\LOTT
\,= \sum_{t\in\Expl} \lott + \sum_{t\notin\Expl} \lott
\,\leq n\qT + \sum_{t\notin\Expl} \lott
\leq 2\sqrt{2}\sqrtn\sqrtT + \sum_{t\notin\Expl} \sum_{i=1}^n (\demubit - \demlbit).
\numberthis
\label{eqn:LTsspdet}
\end{align*}
Above, we have used the fact that in $T$ rounds there will have been
$n\qT$ exploration rounds and then used Lemma~\ref{lem:qTdetsp} to bound $\qT$.
To bound the second sum, we have used Lemma~\ref{lem:ubinstloss}.
Here,  $\demubit=\volit\udubit$ is an upper bound on the unit demand, computed
at the end of the exploration phase in the \ucrecordfb{} method in line~\ref{lin:ucrecordfbssp}.
Similarly, we will define $\demlbit=\volit\udlbit$ to be a lower bound on the unit demand
which is computed as follows.
It is initialised to $\udlbiitt{i}{0}=0$ at the beginning.
At the end of an exploration phase round $t$ for user $i$ it is updated to 
$\udlbiitt{i}{t} = \min(\udmax k/2^h, \udlbiitt{i,}{t-1})$ if $X_i$ was less than 
or equal to $\threshi$.
Here, $h, k$ are as defined in \explphase{} of Algorithm~\ref{alg:detfbspmodel}.

Consider any $t\notin \Expl$ and let $\qt$ be as defined in~\eqref{eqn:qtdefn}.
In the exporation phases up to the $t$\ssth round, we will have evaluated unit demands
for all agents at least at values $\{\udmax k/2^h; h = 1,\dots, \lfloor \log_2(\qt+1) \rfloor, 
\; k=1,\dots,2^h\}$.
Therefore, assuming all agents were truthful, we will have constrained $\udtruei$ to within the
interval $[\udlbit, \udubit)$ of width at most
$\udmax/2^{\lfloor \log_2(\qt+1) \rfloor}$. We have:
\[
\udubit - \udlbit \leq \frac{\udmax}{2^{\lfloor \log_2(\qt+1) \rfloor}}
\leq \frac{\udmax}{2^{\log_2(\qt+1) - 1}}
\leq \frac{2\udmax}{\qt + 1}
\leq \frac{2\udmax}{\qt}
\leq \frac{2\udmax\sqrt{3n}}{t^{\nicefrac{1}{2}}}.
\]
Here, the last step uses Lemma~\ref{lem:qTdetsp}.
We can now use~\eqref{eqn:LTsspdet} to bound the loss as follows,
\begin{align*}
\LOTT 
&\leq 2\sqrt{2}\sqrtn\sqrtT + \sum_{t\notin\Expl} \sum_{i=1}^n \volit(\udubit - \udlbit)
\leq 2\sqrt{2}\sqrtn\sqrtT + 2 n \udmax\volmax\sqrt{3n} \,\sum_{t=1}^T \frac{1}{t^{\nicefrac{1}{2}}}
\\
&\leq 2\sqrt{2}\sqrtn\sqrtT + 2\sqrt{3} \udmax\volmax n^{\nicefrac{3}{2}} \left(2 \sqrt{T}\right)
\leq 10 n^{\nicefrac{3}{2}} \sqrt{T}.
\end{align*}
Here, the second step uses the conclusion from the previous display.
The third step bounds $\sum_{t=1}^T (1/\sqrt{t}) \leq 2\sqrtT$.
The last step uses $\volmax\udmax\leq 1$ by our
assumptions, $\sqrtn \leq n^{\nicefrac{3}{2}}$, and that $2\sqrt{2} + 4\sqrt{3} \leq
10$.

\textbf{Fairness:}
This follows by applying Lemma~\ref{lem:spfairness} with $r=n$ and the upper bound for
$\qT$ in Lemma~\ref{lem:qTdetsp}.

\textbf{Strategy-proofness:}
This follows from Lemma~\ref{lem:spsp} and noting that when $t\in\Expl$, the allocations do not
depend on the strategy adopted by any of the users. Therefore, $\allocpiit=\allocit$.

\subsection{Proof of Theorem~\ref{thm:detsp}}
\label{sec:detsp}


\textbf{Efficiency:}
We will first decompose the $T$-period loss as follows similar to~\eqref{eqn:LTsspdet},
\begin{align*}
\LOTT
\,= \sum_{t\in\Expl} \lott + \sum_{t\notin\Expl} \lott
\,\leq 2n\qT + \sum_{t\notin\Expl} \lott
\leq 2n\log(eT) + \sum_{t\notin\Expl} \sum_{i=1}^n (\demubit - \demlbit).
\label{eqn:LTspdet}
\numberthis
\end{align*}
Above, we have used the fact that in $T$ rounds there will have been
$2n\qT$ exploration rounds and then used Lemma~\ref{lem:qTdet} to bound $\qT$.
To bound the second sum, we have used Lemma~\ref{lem:ubinstloss}.
Here,  $\demubit=\volit\udubit$ and $\demlbit=\volit\udlbit$ are upper and lower bounds on the
demands, where $\udubit, \udlbit$ are the upper and lower bounds on the unit demand, computed
at the end of the exploration phase in the \ucrecordfb{} method in line~\ref{lin:ucrecordfbsp}.
Since $\volmax\udmax\leq 1$, the binary search procedure always 
halves the width of the current upper and lower bounds.
We therefore have, for $t\notin\Expl$,
\[
(\demubit - \demlbit) = \volit(\udubit-\udlbit)
\leq  \frac{\volmax\udmax}{2^{2\qt-1}} 
\leq  \frac{2\volmax\udmax}{4^{\log(\frac{t}{2n+e})}} 
\leq  \frac{2(2n+e)^{\log(4)}\volmax\udmax}{t^{\log(4)}}.
\]
The second step observes that in $t$ rounds, there will have been $\qt$ brackets and
therefore $2\qt$ exploration rounds for agent $i$.
The width $\udubit - \udlbit$ is $\udmax$ at round $t=1$ and is halved after each iteration of
binary search.
We then use the lower bound in Lemma~\ref{lem:qTdet} which is positive when $T\geq 2n+3$,
which is the case when $t\notin \Expl$.
The last inequality uses the fact that $a^{\log b} = b^{\log a}$ for $a,b>0$.
By summing over all agents and time steps, we have
\[
\sum_t\sum_{i=1}^n (\demubit - \demlbit) 
\,\leq  2\cdot 5^{\log(4)} \volmax\udmax \frac{\log(4)}{\log(4)-1} n^{1+\log(4)}
\,\leq 67 \volmax\udmax n^{2.39}.
\]
The claim on asymptotic efficiency follows by combining this with~\eqref{eqn:LTspdet}.

\textbf{Fairness:}
This follows by applying Lemma~\ref{lem:spfairness} with $r=2n$ and the upper bound for
$\qT$ in Lemma~\ref{lem:qTdet}.

\textbf{Strategy-proofness:}
Denote $\normallocit = \allocit/\volit$ and for a policy $\pi$, $\normallocpiit =
\allocpiit/\volit$.
Let $\udubit, \udlbit$ be the upper and lower bounds on the unit demand, as computed
at the end of the exploration phase in the \ucrecordfb{} method in line~\ref{lin:ucrecordfbsp}.
Then, using Lemma~\ref{lem:spsp} we obtain,
\begingroup
\allowdisplaybreaks
\begin{align*}
\UpiiT - \UiT &\leq \sum_{t\in \Expl} (\utili(\normallocpiit) - \utili(\normallocit))
= \sum_{t\in \Expl, \allocit>0} (\utili(\normallocpiit) - \utili(\udtruei) + \utili(\udtruei) - \utili(\normallocit))
\\
&= \sum_{t\in \Expl, \allocit>0} (\utili(\udtruei) - \utili(\normallocit))
\leq \Lipi \sum_{t\in \Expl, \allocit>0} (\udtruei - \normallocit)^+
\\
&= \frac{\Lipi}{2} \sum_{t\in \Expl, \allocit>0} (\udubit - \udlbit)
\leq \frac{\Lipi\udmax}{2} \sum_{\ell=1}^\infty \frac{1}{2^{\ell-1}}
=\Lipi\udmax.
\end{align*}
\endgroup
Here, in the second step we have added and subtracted $\utili(\udtruei)$ and moreover
restricted the summation to rounds where user $i$ receives a non-zero allocation; recall that
in \explphases of Algorithm~\ref{alg:detfbmodel}, each user receives two rounds of allocations.
In the third step we have used the fact that $\udtruei$ maximises $\utili$.
In the fourth step we have used the Lipschitz property of $\utili$.
The fifth step uses the fact that at a round $t$ in the exploration phase for
agent $i$, $\normallocit = (\udubit+\udlbit)/2$ and that $\udtruei \leq \udubit$.
The proof is completed by observing that the exploration phase performs binary search for
$\udtruei$.
Therefore, $\udubit - \udlbit$ is initially equal to $\udmax$ and then halved at each round.
\qedwhite

\subsection{Proof of Theorem~\ref{thm:detnsp}}
\label{sec:detnsp}

\textbf{Efficiency:}
We will use Lemma~\ref{lem:nsploss} to bound the $T$-period loss.
We will bound the first sum in the RHS of Lemma~\ref{lem:nsploss}
by bounding $\sum_{t=2}^T  (\allocit - \demtrueit)^+$ for user $i$.
Recall that we denote $\normallocit = \allocit/\volit$.
Let $\sidxtt{1},\sidxtt{2},\dots,$ denote the round indices after the second round
 in which, $\allocit > \demtrueit$.
We have:
\begin{align*}
\normallocisrpo - \udtruei \leq
    \frac{1}{2}\left(\udubiitt{i}{\sidxrpo} + \udlbiitt{i}{\sidxrpo}\right) - \udtruei 
    =
    \frac{1}{2}\left(\normallocisr - \udtruei\right)
    \leq \;\dots\; \leq \frac{1}{2^{r}} (\normallocisone - \udtruei)
    \leq \frac{\udmax}{2^{r+1}}.
\end{align*}
Above,
the first step uses the fact that the allocations are chosen by \mmf, by setting the demand of
user $i$ to be $\udit\volit =
\frac{1}{2}(\udubit + \udlbit)\volit$ and the fact
that \mmfs does not allocate more than the reported demand
(Property~\ref{pro:allocnomorethandemand});
therefore, $\normallocisrpo \leq \udiitt{i}{\sidxrpo}  
=(\udubiitt{i}{\sidxrpo} + \udlbiitt{i}{\sidxrpo}) /2$.
The second step uses that since we did not over-allocate between rounds
$\sidxr$ and $\sidxrpo$, the upper bound $\udubiitt{i}{\sidxrpo}$ at round $\sidxrpo$ is the
unit allocation $\normallocisr$ at round $\sidxr$;
additionally, $\udlbit \leq \udtruei$ for all $t$.
By repeatedly applying this argument, we arrive at
$(\normallocisone - \udtruei)/2^r$.
The last step follows from the observation that when we first over-allocate,
we have $\normallocisone - \udtruei \leq \frac{1}{2}\udmax$.
Therefore,
\begin{align*}
\sum_{t=2}^T (\allocit -\demtrueit)
\leq \sum_{r=1}^\infty \volisr(\normallocisr -\udtruei)
\leq \volmax \sum_{r=1}^\infty (\normallocisr -\udtruei)
\leq \volmax\udmax.
\numberthis
\label{eqn:LTspdetfirstsum}
\end{align*}
By summing over all agents we can bound the first sum by $n\volmax\udmax$.

Now let us turn to the second summation in Lemma~\ref{lem:nsploss}.
We will consider a user $i$ and bound $\sum_{t=2, \allocit=\demit}^T (\demtrueit-\allocit)^+$.
Letting $\sidxtt{1},\sidxtt{2},\dots,$ denote the round indices where $\allocit=\demit$, we have
\begin{align*}
\udubiitt{i}{\sidxrpo} -\udlbiitt{i}{\sidxrpo}
\leq
\udubiitt{i,}{\sidxr+1} -\udlbiitt{i,}{\sidxr+1}
\leq
\frac{1}{2}\left(\udubiitt{i}{\sidxr} -\udlbiitt{i}{\sidxr}\right)
\leq \dots \leq
\frac{\udmax}{2^{r+1}}.
\label{eqn:widthhalves}
\numberthis
\end{align*}
The first step above uses that both $\udubit, \udlbit$ are respectively non-increasing and
non-decreasing with $t$ and that $\sidxrpo \geq \sidxr + 1$.
The second step uses the fact that since
$\allociitt{i}{\sidxr}=\demiitt{i}{\sidxr}=\udiitt{i}{\sidxr}\voliitt{i}{\sidxr}$ in round
$\sidxr$, the
gap between the upper and lower bounds would have halved at the next round.
Via a similar argument to~\eqref{eqn:LTspdetfirstsum} and summing over all agents,
we can bound the second sum in Lemma~\ref{lem:nsploss} by $n\volmax\udmax$.
Therefore, $\LOTT \leq 1+2n\volmax\udmax$.

\textbf{Fairness:}
We will use Lemma~\ref{lem:nspfairness} to bound $\UeiiT - \UiT$.
Letting $\sidxtt{1},\sidxtt{2},\dots,$ denote the round indices where
$\normallocit=\udit$ and $\normallocit<\udtruei$, we have,
\begin{align*}
\UeiiT - \UiT
&\leq \sum_{t=1}^T 
    \indfone(\normallocit=\udit\,\wedge\,\normallocit<\udtruei) \cdot \left(\utili(\udtruei) -
\utili(\udiitt{i}{t})\right)
=\sum_{r=1}^{\infty} \left(\utili(\udtruei) - \utili(\udiitt{i}{\sidxr})\right)
\\
&\leq \Lipi \sum_{r=1}^{\infty} (\udtruei - \udiitt{i}{\sidxr})^+
\leq \frac{\Lipi}{2} \sum_{r=1}^{\infty} (\udubiitt{i}{\sidxr} - \udlbiitt{i}{\sidxr})
\leq \Lipi\udmax.
\end{align*}
Here, the second we have used that $\utili$ is Lipschitz and non-decreasing.
The third
step simply observes $\udtruei\leq\udubit$ and $\udit=\frac{1}{2}(\udubit + \udlbit)$ for all $t$.
The last step uses a similar argument to~\eqref{eqn:widthhalves}.
\qedwhite

\insertprethmspacing
\begin{remark}\emph{%
Our proofs used the fact that $\volit>0$ when arguing that $\udubit - \udlbit$ is halved at the
end of each exploration round in Appendix~\ref{sec:detsp}.
Had an agent experienced no load during one of her
exploration rounds, this would not have been true.
This assumption can be avoided by more careful book-keeping where the mechanism checks if
an agent has non-zero traffic and if not, delaying that exploration round to a future round.
}
\label{rem:detvolitgz}
\end{remark}

\section{Proofs of Results in Section~\ref{sec:glmfbmodel}}
\label{sec:proofs_glm}

This section presents our analysis for the stochastic parametric model.
In Appendix~\ref{app:glmcconfinterval}, we will first show that the confidence intervals
in~\eqref{eqn:udglmconfint} traps the true demands with probability larger than $1-\delta$ in
all rounds.
In Appendices~\ref{app:glmsp} and~\ref{app:glmnsp}, we will establish our results for
Algorithms~\ref{alg:mmflearnsp} and~\ref{alg:mmflearnnsp} respectively assuming that
this is true.

\subsection{Confidence Intervals for $\udtruei$}
\label{app:glmcconfinterval}

The following lemma shows that the probability that the confidence intervals
in~\eqref{eqn:udglmconfint}
do not capture the true parameter and demands is bounded by $\delta$.
The lemma applies to both Algorithms~\ref{alg:mmflearnsp} and~\ref{alg:mmflearnnsp}.

\insertprethmspacing
\begin{lemma}
\label{lem:confinterval}
Assume the feedback model described in Section~\ref{sec:glmfbmodel} and let
$(\thetalbit, \thetaubit)$ and $(\udlbit, \udubit)$ be as defined in~\eqref{eqn:udglmconfint}.
Then, with probability greater than $1-\delta$,
for all $i\in\{1,\dots, n\}$ and all $t\geq 0$,
\[
\thetatruei \in (\thetalbit, \thetaubit),
\hspace{0.3in}
\udtruei \in (\udlbit, \udubit).
\]
\end{lemma}
\begin{proof}
Consider any user $i$.
First,
we will show that the sum of deviations of $\Xit$ from the true payoff scale with $\Ait$.
We will use Lemma~\ref{lem:mds} with
$z_s\leftarrow \Xis - \payoffi(\normallocis)$,
$\gamma_s\leftarrow\allocis/\volis$,
$\sigma_s\leftarrow\sigmais$, and
let $\filtrs$ be the sigma-field generated by the data from all users up to round $s$.
Accordingly, $\{\gamma_s\}$ is predictable since the recommendations
$\{\udubiitt{j}{s}\times\voliitt{j}{s}\}_{j=1}^n$
for all users are chosen based on their past data and then the allocations are chosen based on these
recommendations via \mmf; i.e. $\gamma_s$ is $\filtrsmo$-- measurable.
$\sigma_s$ is predictable by our assumptions.
Finally, $z_s$ is $\filtrs$-measurable and $\EE[z_s|\filtrtt{s-1}] = 0$.
Now define $\Bit$ as shown below and recall the definition of $\Ait$ from~\eqref{eqn:thetaitdefn}.
\[
\Bit = \sum_{s=1}^{t-1}\frac{\allocis}{\volit\sigmasqis} \left(\Xis - \payoffi(\normallocis)\right),
\hspace{0.8in}
\Asqit := \sum_{s\in D_t} \frac{\allocis^2}{\volis^2\sigmasqis}.
\]
By applying Lemma~\ref{lem:dela}
for a given $\delta'\in(0, 1/e)$, the following holds with probability at least
$1-\delta'$:
\begin{align*}
 |\Bit| \leq \Ait\sqrt{(4+2\log(2))\log(1/\delta')}.
\label{eqn:deltapintermediate}
\numberthis
\end{align*}

Next, we will bound the estimation error $|\thetait - \thetatruei|$
(see~\eqref{eqn:thetaitdefn}) for the parameter $\thetatruei$ in terms of $\Bit$.
This part of the analysis is based on prior work on generalised linear
models~\citep{filippi2010parametric}.
We first define,
\[
\git(\theta) = \sum_{s=1}^{t-1} \frac{\allocis}{\volis\sigmasqis}
    \mu\left( \frac{\allocis}{\volis}\,\theta\right).
\]
By an application of the mean value theorem we have,
\[
|\thetait - \thetatruei|
\leq \frac{1}{\gdotit(\thetait')} |\git(\thetait) - \git(\thetatruei)|
\leq \frac{1}{\kappamudot\Asqit} |\git(\thetait) - \git(\thetatruei)|
\]
Here $\thetait'$ is between $\thetait$ and $\thetatruei$.
We have upper bounded its integral using Assumption~\ref{asm:glm} to
obtain $\gdotit(\thetait') = 
\sum_{s=1}^{t-1} \frac{\allocis^2}{\volis^2\sigmasqis}
    \dot{\mu}( \frac{\allocis}{\volis}\,\thetait')
\geq \kappamudot\Asqit$.

Next, let $\thetamlit$ be the unique $\theta\in\RR$ satisfying
$\git(\theta) = \sum_{s=1}^{t-1} \frac{\allocis}{\volis\sigmasqis} \Xis$.
Readers familiar with the literature on generalised linear models will recognise $\thetamlit$ as
the maximum quasi-likelihood estimator~\citep{chen1999strong}.
We can therefore rewrite $\thetait$ in~\eqref{eqn:thetaitdefn}
as $\thetait = \argmin_{\theta>\thetamin} | \git(\theta) - \git(\thetamlit)|$.
We therefore have,
\begin{align*}
|\thetait - \thetatruei|
&\leq \frac{1}{\kappamudot\Asqit} |\git(\thetait) - \git(\thetatruei)|
\leq \frac{1}{\kappamudot\Asqit}\left( |\git(\thetait) - \git(\thetamlit)| +
                                |\git(\thetamlit) - \git(\thetatruei)| \right)
\\
&\leq \frac{2}{\kappamudot\Asqit} |\git(\thetamlit) - \git(\thetatruei)|
= \frac{2}{\kappamudot\Asqit} |\Bit|.
\label{eqn:thetadiffbound}
\numberthis
\end{align*}
The first step uses the conclusion from the previous display.
The third step uses the fact that $\thetait$ minimises $| \git(\theta) - \git(\thetamlit)|$;
therefore, $| \git(\thetait) - \git(\thetamlit)| \leq | \git(\thetatruei) - \git(\thetamlit)|$.
The last step uses the definition of $\Bit$.

Finally,
we will apply~\eqref{eqn:deltapintermediate} with $\delta'\leftarrow 6\delta/(\pi^2n|D_t|^2)$ which is
less than $1/e$ if $n\geq 2$.
By~\eqref{eqn:thetadiffbound}, a union bound, and observing $\sum t^{-2} = \pi^2/6$, we have,
with probability greater than $1-\delta$,
\[
\text{for all $i, t$,} 
\quad
|\thetait - \thetatruei|
\leq \frac{2}{\kappamudot\Asqit} \sqrt{4+2\log(2)\log(1/\delta')} \Ait \leq \frac{\betat}{\Ait}.
\]
This establishes the result for the confidence intervals for $\thetatruei$.
The result for the confidence intervals for $\udtruei$ follows from the fact that 
 $\udtruei=\mu^{-1}(\alpha)/\thetatruei$ is a decreasing function of
$\thetatruei$.
\end{proof}

\subsection{Proof of Theorem~\ref{thm:glmsp}}
\label{app:glmsp}

Throughout this proof we will assume that for all
$i,n$, $\demtrueit\in(\demlbit,
\demubit)$. By Lemma~\ref{lem:confinterval}, this is true with probability at least $1-\delta$.

\textbf{Efficiency:}
We will first decompose the loss as follows,
\begin{align*}
\LOTT
\,\leq \sum_{t\in\Expl} \lott + \sum_{t\notin\Expl} \lott
\,\leq \qT + \sum_{t\notin\Expl} \lott
\leq 3\Ttwth + \sum_{t\notin\Expl} \sum_{i=1}^n (\demubit - \demlbit).
\end{align*}
We have used the fact that in $T$ rounds there will have been
$\qT$ exploration rounds and then used Lemma~\ref{lem:qTglm} with $a=1$ to upper bound $\qT$.
To bound the second sum, we have used Lemma~\ref{lem:ubinstloss}.
Here,  $\demubit=\volit\udubit$ and $\demlbit=\volit\udlbit$ are upper and lower confidence
bounds on the
instantaneous demands, where $\udubit, \udlbit$ are the upper and lower confidence bounds on the
unit demand as given in~\eqref{eqn:udglmconfint}.
We can bound this further via,
\begin{align*}
(\demubit - \demlbit) &=
\volit(\udubit - \udlbit) = \volit \muinv(\threshi)\left(\frac{1}{\thetalbit} -
\frac{1}{\thetaubit}\right)
= \frac{\volit\muinv(\threshi)}{\thetalbit\thetaubit} \left(\thetaubit - \thetalbit \right)
\\
&\leq \frac{\muinv(\threshi)\volmax}{\thetamin\thetatruei} \frac{2\betat}{\Ait} 
\leq
\frac{2\sqrt{2}\udtruei\volmax^2\sigmamax}{\thetamin \entitli} \betat t^{-\nicefrac{1}{3}}
\leq
2 C\betat t^{-\nicefrac{1}{3}}
\end{align*}
Here, the first three steps substitutes the expressions for $\thetalbit,\thetaubit,\udlbit,\udubit$
from~\eqref{eqn:thetaitdefn}.
The fourth step uses $\thetaubit \geq \thetatruei$, $\thetalbit\geq \thetamin$, and
that $\thetaubit-\thetalbit = 2\betat\Ait$.
The fifht step uses $\udtruei = \muinv(\threshi)/\thetatruei$ and
an upper bound for $\Ait$ that we will show below.
The last step simply uses the definition of $C$ given in the theorem.
To obtain the upper bound on $\Ait$, we 
we use Lemma~\ref{lem:qTglm} and the fact that at all exploration rounds,
agent $i$ receives an allocation $\entitli$.
We have:
\[
\Asqit = \sum_{s\in D_t} \frac{\allocis^2}{\volis^2\sigmasqis}
\geq \frac{\entitli^2}{\volmax^2\sigmamax^2} \qt
\geq \frac{\entitli^2}{2\volmax^2\sigmamax^2} t^{\nicefrac{2}{3}}.
\]
This leads us the following bound on $\LOTT$,
\[
\LOTT
\leq 3\Ttwth + 2C\betaT\sum_{i=1}^n \sum_{t\notin\Expl}  t^{-\nicefrac{1}{3}}
\leq 3\Ttwth + 3Cn\betaT\Ttwth.
\]

\textbf{Fairness:}
This follows as a consequence of Lemma~\ref{lem:ubfairness}.
Precisely,
\begin{align*}
\UeiiT - \UiT &= \sum_{t\notin \Expl} \left(\utili(\entitli/\volit) - \utili(\allocit/\volit)\right)
\leq 0.
\end{align*}
In the first step, we use the fact that 
during the exploration phase rounds $t\in\Expl$, each agent gets her entitlement $\allocit=\entitli$.
During the other rounds, we invoke \mmfs using an upper bound on agent $i$'s demand.
In Lemma~\ref{lem:ubfairness} with set both $\udubit$ and $\udit$ to this upper bound, and
consequently, the RHS is $0$.

\textbf{Strategy-proofness:}
This follows from Lemma~\ref{lem:spsp} and the following two observations:
first, under the event specified in Lemma~\ref{lem:confinterval}, $\udubi$ in
line~\eqref{lin:spallocub} of Algorithm~\ref{alg:mmflearnsp}
is an upper bound on $\udtruei$;
second, the allocations in the exploration phase
do not depend on reports of any agents from previous rounds, therefore $\allocpiit =\allocit$.
\qedwhite

\subsection{Proof of Theorem~\ref{thm:glmnsp}}
\label{app:glmnsp}

We will begin with three intermediate lemmas.
The first will help us establish the result for asymptotic efficiency, the second will help us  with
 asymptotic BNIC,
and the third is a technical lemma that will help us in controlling both loss terms.
In all three lemmas, we will let $\Ecal$ denote the event   that for all
$i,n$ $\demtrueit\in(\demlbit,
\demubit)$. By Lemma~\ref{lem:confinterval}, $\PP(\Ecal)\geq 1-\delta$.

\insertprethmspacing
\begin{lemma}
Let $\Cone$ be as defined in~\eqref{eqn:glmnspdefns}.
The following bound holds on $\LOTT$ under event $\Ecal$.
\label{lem:LOTTbound}
\[
\LOTT \leq 1 + \frac{\betaT}{\Cone}\sum_{i=1}^n\sum_{t=2}^T \min\left( \Cone,
\frac{\allocit}{\volit\sigmait\Ait}\right).
\]
\end{lemma}
\begin{proof}
We will divide this proof into three steps. In the first two steps, we will consider a single round
$t\geq 2$.
Recall the notation from the beginning of Appendix~\ref{sec:proofs_intermediate}.

\textbf{Step 1:}
First, we will argue that $\lott \leq \min(1, \lotort)$.
Observe that if $\loturt=0$, then $\alloct = \demubit \geq \demtrueit$ and therefore
$\lott \leq \lotudt = 0$ and the claim is true.
When $\loturt = 0$, then we can bound $\lott \leq \loturt+ \lotort=\lotort$.
The above claim follows by observing that $\lott\leq 1$ trivially since the total amount of the
resource is $1$.

\textbf{Step 2:}
Next, consider any user $i$. We will argue that $\max(\allocit-\demtrueit, 0) \leq
\frac{2\betat\volmax}{\thetamin}\frac{\allocit}{\Ait}$.
Let $\{\alloctruejt\}_{j=1}^n$ denote the allocations returned by \mmfs if we
were to invoke with the true demands $\{\demtruejt\}_{j=1}^n$ for round $t$.
Recall that $\alloctrueit\leq \demtrueit$.

First assume $\alloctrueit<\demtrueit$.
Then, by Property~\ref{pro:alloclessthandemand}, increasing $i$'s reported demand to
$\demubit (>\demtrueit)$ while keeping the demands of all other agents at $\{\demtruejt\}_{j\neq i}$
does not increase $i$'s allocation.
Moreover, increasing the demands of all other agents cannot increase $i$'s allocation and
hence $\allocit \leq \alloctrueit < \demtrueit$; therefore, the statement is true since the RHS is
positive.

Now assume $\alloctrueit=\demtrueit$.
If $\allocit\leq\alloctrueit=\demtrueit$, the statement is trivially true.
Therefore, let $\allocit > \alloctrueit=\demtrueit>\demlbit$.
As \mmfs does not allocate more than the reported demand, we also have $\allocit<\demubit$.
This results in the following bound.
\[
\max(\allocit-\demtrueit, 0)
\leq \demubit - \demlbit
\leq \frac{\muinv(\alpha)\volit}{\thetalbit\thetaubit}(\thetaubit-\thetalbit)
\leq \frac{\volmax\demlbit}{\thetalbit}\frac{2\betat}{\Ait}
\leq \frac{2\volmax\betat}{\thetamin}\frac{\allocit}{\Ait}.
\]

\textbf{Step 3:}
We now combine the results of the previous two steps to bound $\LOTT$ as follows.
\begin{align*}
\LOTT &= \sum_{t=1}^T \lott
    \leq 1 + \sum_{t=2}^T \lott
    \leq 1 + \sum_{t=2}^T \min(1, \lotort)
    = 1 + \sum_{t=2}^T \min\left(1, \sum_{i=1}^n \max(\allocit-\demtrueit, 0) \right)
\\
    &\leq 1 + \sum_{t=2}^T \sum_{i=1}^n\min\left(1, \max(\allocit-\demtrueit, 0) \right)
    \leq 1 + \sum_{t=2}^T \sum_{i=1}^n\min\left(1, \frac{2\volmax\betat}{\thetamin}
                                            \frac{\allocit}{\Ait} \right)
\\
    &\leq 1 + \sum_{i=1}^n\sum_{t=2}^T
    \min\left(1, \frac{2\volmax\betat\sigmait\volit}{\thetamin}
        \frac{\allocit}{\sigmait\volit\Ait} \right)
    \leq 1 + \betaT\sum_{i=1}^n\sum_{t=2}^T
    \min\left(1, \frac{1}{\Cone} \frac{\allocit}{\sigmait\volit\Ait} \right)
\\
   &\leq 1 + \frac{\betaT}{\Cone}\sum_{i=1}^n\sum_{t=2}^T
    \min\left(\Cone, \frac{\allocit}{\sigmait\volit\Ait} \right)
\label{eqn:LOTTboundstepthree}
\numberthis
\end{align*}
Here, the third step applies the results from step 1.
The fifth step simply uses $\min(a, \sum_i b_i) \leq \sum_i\min(a, b_i)$ when $b_i\geq 0$ for all $i$.
The sixth step applies the result from step 2.
The eighth step uses the fact that $\betat\geq 1$ for all possible choices of
$n,\delta, |D_t|$.
\end{proof}

\insertprethmspacing
\begin{lemma}
\label{lem:UpiiTbound}
Consider any agent $i\in\{1,\dots,n\}$ and assume all other agents are submitting truthfully.
Let $\UpiiT, \UiT$ be as defined in Theorem~\ref{thm:glmnsp}.
The following bound holds on $\UpiiT - \UiT$ under event $\Ecal$.
\emph{
\[
\UpiiT - \UiT \leq \frac{\Lipi\betaT}{\Cone\volmin} \sum_{j\neq i}\sum_{t=2}^{T}
         \min\left( \Cone, \frac{\allocjt}{\voljt\sigmajt\Ajt}\right).
\]
}
\end{lemma}
\begin{proof} We will divide this proof into three steps.

\textbf{Step 1:}
We will decompose $\UpiiT - \UiT$ as follows.
Let $\utilpiit$ denote the utility at time $t$ when agent $i$ is following $\pi$.
Let $\utildagit$ denote the utility at time $t$ when agent $i$ is follows $\pi$ until round $t-1$ and
then on round $t$ we invoke \mmfs (line~\ref{lin:nspalloc}, Algorithm~\ref{alg:mmflearnnsp}) with
the true demand $\demtrueit=\udtruei\volit$ for agent $i$, and the upper bounds
$\demubjt = \udubjt\voljt$ for all other agents $j$.
Let $\utilstarit$ denote the utility of agent $i$ when we invoke \mmfs with the true demands
for all agents.
We then have,
\begin{align*}
\UpiiT - \UiT = \sum_{t=1}^n \left( \utilpiit - \utilit \right)
    = \sum_{t=1}^n \left( \utilpiit -\utildagit + \utildagit - \utilstarit + \utilstarit - \utilit \right)
    \leq \sum_{t=1}^n \left(\utilstarit - \utilit \right).
\end{align*}
The last step is obtained via two observations.
First, $\utilpiit - \utildagit \leq 0$ since, by
Properties~\ref{pro:reportlessthandemand} and~\ref{pro:reportmorethandemand},
reporting true demands is a weakly dominant strategy for agent $i$.
Second, $\utildagit-\utilstarit\leq 0$ since the reported demands used for all other agents $j\neq
i$ when invoking \mmf,
is an upper bound on their true demands under $\Ecal$ as they are being truthful.
When comparing the reported demands in $\utildagit$ and $\utilstarit$, we see that
agent $i$'s reported demand stays
the same but the  reported demands of other agents decrease, which cannot
decrease the allocation and consequently the utility of agent $i$.

Next, using the Lipschtiz properties of $\utili$, we can write
$\utilstarit - \utilit = \utili(\alloctruei/\volit) - \utili(\allocit/\volit)
\leq \Lipi\max(\alloctruei/\volit - \allocit/\volit, 0)
\leq \frac{\Lipi}{\volmin}\max(\alloctruei - \allocit, 0)$.
Moreover, since increasing all user's demands can only increase the sum of allocations, we have
$\sum_{j=1}^n\alloctruejt \leq \sum_{j=1}^n\allocjt$.
Therefore, $\alloctrueit - \allocit \leq \sum_{j\neq i}(\allocjt - \alloctruejt)$
and hence $\max(\alloctrueit - \allocit, 0) \leq 
\sum_{j\neq i}\max(\allocjt - \alloctruejt, 0)$.
This leads us to the following bound.
\begin{align*}
\UpiiT - \UiT 
    \leq \frac{\Lipi}{\volmin}\sum_{t=1}^T \max\left(\alloctrueit - \allocit , 0\right)
    \leq \frac{\Lipi}{\volmin}\sum_{t=1}^T \sum_{j\neq i}
        \max\left(\allocjt - \alloctruejt, 0 \right).
\end{align*}

\textbf{Step 2:}
Consider any user $j\neq i$. Here, we will argue that $\max(\allocjt-\alloctruejt, 0) \leq
\frac{2\betat}{\thetamin}\frac{\allocjt}{\Ajt}$,
where, recall $\{\alloctrueiitt{k}{t}\}_{k=1}^n$
is the allocation returned by \mmfs if we use the true demands
$\{\demtrueiitt{k}{t}\}_{k=1}^n$ for all agents $k$.
First observe that if $\allocjt\leq \alloctruejt$, the statement is trivially true as the RHS is
positive.

Next, if $\allocjt>\alloctruejt$, we argue that $\alloctruejt=\demtruejt$; we will prove this via
its contrapositive.
Observe that $\alloctrueiitt{j}{t}\leq \demtrueiitt{j}{t}$ by properties of \mmf.
Let us assume that $\alloctruejt<\demtruejt$.
Then, by Property~\ref{pro:alloclessthandemand}, increasing $j$'s reported demand to
$\demubjt (>\demtruejt)$ while keeping the reported demands of all other agents at
$\{\demtrueiitt{k}{t}\}_{k\neq j}$ does not increase her allocation.
Next, increasing the demands of all other agents cannot increase $j$'s allocation and
hence $\allocjt \leq \alloctruejt$. This proves the above statement.

By substituting $\demtruejt$ for $\alloctruejt$, we obtain the following bound.
\[
\max(\allocjt-\alloctruejt, 0)
\leq \demubjt - \demlbjt
\leq \frac{\volit\muinv(\alpha)}{\thetalbjt\thetaubjt}(\thetaubjt-\thetalbjt)
\leq \frac{\volmax\demlbjt}{\thetalbjt}\frac{2\betat}{\Ajt}
\leq \frac{2\volmax\betat}{\thetamin}\frac{\allocjt}{\Ajt}.
\]
Above, we have used the facts $\allocjt\leq \demubjt$, $\demtruejt\geq\demlbjt$,
and the expressions for
$\demubjt,\demlbjt,\thetaubit,\thetalbjt$.

\textbf{Step 3:}
Combining the results of the two previous steps, we obtain the following bound,
\begin{align*}
\UpiiT - \UiT
    &\leq \frac{\Lipi}{\volmin}\sum_{t=1}^n \sum_{j\neq i}
         \min\left( 1, \frac{2\volmax\betat}{\thetamin}\frac{\allocjt}{\Ajt} \right)
    \leq  \frac{\Lipi\betaT}{\Cone\volmin} \sum_{j\neq i}\sum_{t=2}^{T}
         \min\left( \Cone, \frac{\allocjt}{\voljt\sigmajt\Ajt}\right).
\end{align*}
Here, the first step observes that $\max(\allocjt-\alloctruejt, 0) \leq 1$
and the last step is obtained by repeating the calculations in~\eqref{eqn:LOTTboundstepthree}.
\end{proof}

\insertprethmspacing
\begin{lemma}
\label{lem:minbound}
Let $c>0$.
Then, $\sum_{t=1}^T \min\left(c, \frac{\allocit^2}{\sigmait^2\volit^2\Asqit}\right) \leq
       \frac{c}{\log(1+c)}\log(\Cthree T)$,
where $\Cthree$ is as given in~\eqref{eqn:glmnspdefns}.
\end{lemma}
\begin{proof}
We first simplify $\Asqiitt{i}{T}$ as follows.
\begin{align*}
\Asqiitt{i}{T} = \Asqiitt{i,}{T-1} + \frac{\allocit^2}{\volit^2\sigmasqit}
        = \Asqiitt{i,}{T-1}\left( 1 + \frac{\allocit^2}{\Asqiitt{i,}{T-1}\volit^2\sigmasqit}\right)
        = \dots
        = \Asqiitt{i}{2} \prod_{s=1}^{T-1} \left(1 + \frac{\allocit^2}{\Asqiitt{i}{s}\volit^2\sigmasqit}\right)
\end{align*}
Now, observing that $\Asqiitt{i}{T} = \sum_{t=1}^{T-1} \frac{\allocit^2}{\sigmasqit\volit^2} \leq
\frac{T-1}{\sigmamin^2\volmin^2}$ and that
$\Asqiitt{i}{2} = \frac{\entitli^2}{\sigmasqiitt{i}{1}\voliitt{i}{1}^2} \geq
\frac{\entitli^2}{\sigmamax^2\volmax^2}$,
we obtain,
\begin{align*}
 \sum_{t=1}^T\log\left( 1 + \frac{\allocit^2}{\sigmait^2\volit^2\Asqit} \right)
= \log\left(\frac{\Asqiitt{i}{T}}{\Asqiitt{i}{2}}\right)
\leq \log(\Cthree T).
\end{align*}
By applying Lemma~\ref{lem:logbound}, we obtain
\begin{align*}
\sum_{t=1}^T \min\left(c, \frac{\allocit^2}{\sigmait^2\volit^2\Asqit}\right) \leq
    \frac{c}{\log(1+c)} \sum_{t=1}^T\log\left( 1 + \frac{\allocit^2}{\sigmait^2\volit^2\Asqit} \right)
   \leq \frac{c}{\log(1+c)} \log\left(\Cthree T\right).
\end{align*}
\end{proof}

We are now ready to prove the theorem.

\textbf{Proof of Theorem~\ref{thm:glmnsp}.}
In this proof,
denote $Q_i = \sum_{t=2}^T \min\left( \Cone, \frac{\allocit}{\volit\sigmait\Ait}\right)$,
and recall the definitions of $\Cone,\Ctwo,\Cthree$ from~\eqref{eqn:glmnspdefns}.
By an application of the Cauchy-Schwarz inequality and Lemma~\ref{lem:minbound}
we obtain,
\[
Q^2_i \leq (T-1) \sum_{t=2}^T \min\left( \Cone^2, \frac{\allocit^2}{\volit^2\sigmasqit\Asqit}\right)
\leq (T-1) \Cone^2\Ctwo^2\log(\Cthree T).
\]

\textbf{Efficiency:} 
The claim for asymptotic efficiency follows by applying Lemma~\ref{lem:LOTTbound}.
\begin{align*}
\LOTT &\leq 1 + \frac{\betaT}{\Cone}\sum_{i=1}^n Q_i 
    \leq 1 + \Ctwo n \betaT\sqrt{(T-1) \log(\Cthree T)}.
\end{align*}

\textbf{Fairness:}
Similar to the argument in Appendix~\ref{app:glmsp},
this follows as a direct consequence of Lemma~\ref{lem:ubfairness}.
On round $1$, each agent gets her entitlement.
On each subsequent round, we invoke \mmfs using an upper bound on the agent's demand.

\textbf{Strategy-proofness:}
The claim for asymptotic Bayes-Nash incentive compatibility follows by applying Lemma~\ref{lem:UpiiTbound}.
\begin{align*}
\UpiiT - \UiT &\leq \frac{\Lipi\betaT}{\Cone\volmin}\sum_{i=1}^n Q_i 
    \leq \frac{\Lipi \Ctwo}{\volmin} (n-1) \betaT\sqrt{(T-1) \log(\Cthree T)}.
\end{align*}
\qedwhite

\insertprethmspacing
\begin{remark}\emph{%
In our proofs, if $\sigmait=0$, many of the above quantities become undefined, resulting in a
degeneracy.
However, if
at any instant $\sigmait=0$, it means we will have observed $\payoffthetatruei$ at
$\allocit/\volit$ exactly, and will know $\thetatruei$.
Since, $\payoffthetatruei$ is completely determined by $\thetatruei$, we can compute the agent's
true unit demand and use it from thereon.
We also require $\volit>0$ for the strategy-proof case for a similar reason as in
Remark~\ref{rem:detvolitgz}. This can be avoided by more careful book-keeping which postpones an
agent's exploration round if  $\volit>0$.
}
\end{remark}

\section{Proofs of Results in Section~\ref{sec:nonparfbmodel}}
\label{sec:proofs_nonpar}

In this section, we analyse the stochastic model with nonparametric payoffs.
We will first define some quantities and notation that will be used in our proofs.

Let $\lhk = (k-1)2^{-h}\udmax$ and $\rhk = k2^{-h}\udmax$
denote the left and right and points of the interval
$\Ihk$~\eqref{eqn:Ihk}. 
Then, for a user $i$, let $\Deltaihk$ be,
\begin{align*}
\Deltaihk = \begin{cases}
\payoffi(\lhk) - \threshi, \quad\quad&\text{if } \lhk>\threshi, \\
\threshi - \payoffi(\rhk), \quad\quad&\text{if } \rhk<\threshi, \\
0, \quad\quad&\text{otherwise.}
\end{cases}
\numberthis \label{eqn:Deltaihk}
\end{align*}
For an interval $(h,k)$ with $\Deltaihk>0$, we define
\begin{align*}
\uithk = \frac{4\betat^2}{\left(\Deltaihk-L2^{-h}\right)^2}.
\numberthis \label{eqn:uithk}
\end{align*}
In our analysis, we will consider a desirable event $\Ecal$.
We will first show that $\PP(\Ecal)>1-\delta$ and then show that our bounds will hold when
$\Ecal$ is true, thus proving our theorems.
In order to define $\Ecal$, we first define $\Ecalit(h,k)$ below.
Recall that $\flbit,\fubit$ are defined in~\eqref{eqn:flbfub}.
\begin{align*}
\Ecalit(h,k) = \begin{cases}
\fubithk > \threshi\, \wedge \, \flbithk < \threshi, \quad\quad&\text{if }\udtruei\in\Ihk,\\
\flbithk > \threshi, \quad&\text{if } \!\begin{aligned}[t]
                                        &\Ihk\subset(\udtruei,\udmax],\;
                                        \payoffi(\lhk)-\threshi>\frac{L}{2^h}, \\
                                        &\VSithk>\uithk,
                                       \end{aligned} \\
\fubithk < \threshi, \quad&\text{if } \!\begin{aligned}[t]
                                        &\Ihk\subset[0,\udtruei), \;\;
                                        \threshi - \payoffi(\rhk)>\frac{L}{2^h}, \\
                                        &\VSithk>\uithk,
                                       \end{aligned} \\
{\rm True} \quad&\text{otherwise.}
\end{cases}
\numberthis
\label{eqn:Ecalit}
\end{align*}
We then define $\Ecalit, \Ecali, \Ecal$ as shown below.
\begin{align*}
\Ecalit = \bigcap_{h=0}^{\infty}\, \bigcap_{k=1}^{2^h}\, \Ecalit(h,k),
\hspace{0.4in}
\Ecali = \bigcap_{t=1}^{\infty}\,  \Ecalit,
\hspace{0.4in}
\Ecal = \bigcap_{i=1}^n\, \Ecali.
\numberthis \label{eqn:Ecal}
\end{align*}

Recall that $\Pit$ is the path  chosen by \ucrecordfbs for user $i$ in round $t$.
We will let $(\Hit, \Kit)$ denote the last node in $\Pit$.
It is worth observing that $\Pit$ is also the path chosen by \ucubtraverses when we are
in the exploration phase in Algorithm~\ref{alg:mmflearnsp}, or 
by \ucgetudrecs when $\udit = \normallocit$ in Algorithm~\ref{alg:mmflearnnsp}.
Next, we define $\Nithk,\Npithk$ as follows.
\begin{align*}
\Nithk = \sum_{s=1}^{t-1} \indfone((h,k)\in\Pis),
\hspace{0.4in}
\Npithk = \sum_{s=1}^{t-1} \indfone((H_{it}, K_{it}) = (h,k)).
\numberthis \label{eqn:Nithk}
\end{align*}
Here, $\Nithk$ is the number of points assigned to $(h,k)$ while $\Npithk$ only counts the
point if it was the last node in the path chosen by \ucrecordfb.
The following relations should be straightforward to verify.
\begin{align*}
&\Npithk \leq \Nithk,
\hspace{0.7in}
\sum_{(h,k)\in\Tcalit} \hspace{-0.0in} \Npithk = t-1,
\\
&\sigmamin^2\VSithk \leq \Nithk \leq \sigmamax^2\VSithk.
\end{align*}
We will find it useful to define $\githk$ as shown below, which is similar to $\Bit$, but is defined
using the $\flbit,\fubit$ quantities.
\begin{align*}
\githk = \min\left(\fubithk-\threshi, \threshi-\flbithk\right)
\numberthis \label{eqn:githk}
\end{align*}

Recall that the intervals $\{\Ihk\}_{k=1}^{2^h}$ at each height partitions $[0, \udmax]$.
Hence, at each height $h$, there is a unique interval $\kthrh$ which contains the unique demand:
\begin{align*}
\text{for all $h$, }\quad \udtruei\in\Ihhkk{h}{\kthrh}.
\numberthis
\label{eqn:thresholdnodes}
\end{align*}
We will refer to this sequence $(0,1), (1,\kthrhh{1}), (2, \kthrhh{3}), \dots$ as the
\emph{threshold nodes}.

Observe that we use $\betattilde$ in the expressions for the confidence
intervals~\eqref{eqn:flbfub}, where $\ttilde$ is as given in~\eqref{eqn:betataudefn}.
The following statements are straightforward to verify.
\begin{align*}
t \leq \ttilde \leq 2t,
\hspace{1in}
\betattilde \leq \betatt{2t} \in \bigO\left(\sqrt{\log(nt/\delta)}\right).
\numberthis
\label{eqn:trelationttilde}
\end{align*}

Finally, in this proof, when we say that a node $(h,k)$ is an ancestor of $(h',k')$,
we mean that $(h,k)$ could be $(h', k')$, or its parent, or its parent's parent, etc.
We say that $(h',k')$ is a descendant of $(h,k)$ if $(h,k)$ is an ancestor of $(h',k')$.

We can now proceed to our analysis, which will be organised as follows.
In Appendix~\ref{sec:boundingPEc}, we will bound $\PP(\Ecal^c)$ for both
Algorithms~\ref{alg:mmflearnsp} and~\ref{alg:mmflearnnsp}.
In Appendices~\ref{sec:nonparintermediateresults} and~\ref{sec:nonparintermediatedefns},
we will establish some intermediate lemmas and definitions that will
be used in both algorithms.
Then, in Appendices~\ref{sec:nonparspproofs} and~\ref{sec:nonparnspproofs}, we will prove
our main results for 
Algorithms~\ref{alg:mmflearnsp} and~\ref{alg:mmflearnnsp} respectively.

\subsection{Bounding $\PP(\Ecal^c)$}
\label{sec:boundingPEc}

In this section, we will prove the following result which applies to both,
Algorithms~\ref{alg:mmflearnsp} and~\ref{alg:mmflearnnsp}.

\insertprethmspacing
\begin{lemma}
\label{lem:PEc}
Let $\Ecal$ be as defined in~\eqref{eqn:Ecal}.
Then $\PP(\Ecal^c)\leq \delta$.
\end{lemma}

The first step of the proof of this lemma considers nodes $(h,k)$ which satisfy the first case
in~\eqref{eqn:Ecalit}. We will first show the following result.

\insertprethmspacing
\begin{lemma}
\label{lem:Pgoodsets}
Consider user $i$ and let $(h,k)$ be such that $\udtruei\in\Ihk$.
Then,
\[
\forall\;t\geq 2, \quad
\PP\left(\fubithk<\threshi \vee \flbithk>\threshi\right) \leq \frac{6\delta}{n\pi^2t^3}.
\]
\end{lemma}
\begin{proof}
First observe that if $\VSithk=0$, the statement is true by the definition of $\flbithk,\fubithk$ as
the event inside $\PP()$ holds with probability $0$.
Therefore, assume $\VSithk>0$ going forward.
We first observe:
\begingroup
\allowdisplaybreaks
\begin{align*}
\fubithk < \threshi
&\iff \fbarithk + \frac{\betattilde}{\sqrt{\VSithk}} + \frac{\Lf}{2^h} < \threshi,
\\
&\iff \sum_{s=1}^{t-1}\frac{\indfone((h,k)\in\Pis)}{\sigmais^2} \Xis
    + \betattilde\sqrt{\VSithk} < \left(\threshi - \frac{\Lf}{2^h}\right) \VSithk,
\\
&\iff
\sum_{s=1}^{t-1}\frac{\indfone((h,k)\in\Pis)}{\sigmais^2}\left(\Xis-\payoffi(\normallocis)\right)
    + \betattilde\sqrt{\VSithk} \,<\,
    \\
    &\hspace{2in} \sum_{s=1}^{t-1}\frac{\indfone((h,k)\in\Pis)}{\sigmais^2}
        \left(\threshi - \frac{L}{2^h}-\payoffi(\normallocis)\right),
\\
&\implies
\sum_{s=1}^{t-1}\frac{\indfone((h,k)\in\Pit)}{\sigmais^2}\left(\Xis-\payoffi(\normallocis)\right)
    < - \betattilde\sqrt{\sum_{s=1}^{t-1}\frac{\indfone((h,k)\in\Pis)}{\sigmais^2}}.
\end{align*}
\endgroup
Here, the last step simply uses condition~\eqref{eqn:nonparfbmodel} to conclude,
$
\normallocis\in\Ihk \implies |\udtruei-\normallocis| \leq \frac{\udmax}{2^h}
\implies
\threshi-\payoffi(\normallocis) < \frac{\Lf}{2^h}.
$
By a similar argument, we can show,
\begin{align*}
\flbithk > \threshi
\implies
\sum_{s=1}^{t-1}\frac{\indfone((h,k)\in\Pit)}{\sigmais^2}\left(\Xis-\payoffi(\normallocis)\right)
    > \betattilde\sqrt{\sum_{s=1}^{t-1}\frac{\indfone((h,k)\in\Pis)}{\sigmais^2}}.
\end{align*}
We will now apply Lemma~\ref{lem:mds} with
$\gamma_s \leftarrow \indfone((h,k)\in\Pis)$,
$\sigma_s\leftarrow\sigmais$,
$z_s\leftarrow(\Xis-\payoffi(\normallocis))$.
We will let $\filtrs$ be the sigma-field generated by the data from all users up to round $s$.
Accordingly, $\{\gamma_s\}$ is predictable since the recommendations
$\{\udiitt{j}{s}\times\voliitt{j}{s}\}_{j=1}^n$
for all users are chosen based on their past data, the allocations are chosen based on these
recommendations via \mmf, and finally $\Pis$ depends on this allocation and $i$'s past data in the
tree; i.e. $\gamma_s$ is $\filtrsmo$--measurable.
Moreover, $\sigma_s$ is predictable by our assumptions.
Finally, $\EE[z_s|\filtrtt{s-1}] = 0$.
By combining the two previous displays, we have,
\begin{align*}
&\PP(\fubithk < \threshi \vee \flbithk > \threshi)
\\
&\hspace{0.2in}\leq
\PP\left(
\left|
\sum_{s=1}^{t-1}\frac{\indfone((h,k)\in\Pit)}{\sigmais^2}\left(\Xis-\payoffi(\normallocis)\right)
\right|
    > \betattilde\sqrt{\sum_{s=1}^{t-1}\frac{\indfone((h,k)\in\Pis)}{\sigmais^2}}.
\right)
\\
&\hspace{0.2in}
\leq\; \frac{6\delta}{n\pi^2\ttilde^3}
\leq\; \frac{6\delta}{n\pi^2 t^3}.
\end{align*}
The first step uses Lemma~\ref{lem:mds} and
the last step follows from the observation $\ttilde\geq t$.
\end{proof}

Next, we consider nodes $(h,k)$ which satisfy the second
case in~\eqref{eqn:Ecalit}, for which we have the following result.

\insertprethmspacing
\begin{lemma}
\label{lem:Pbadgtr}
Consider user $i$ and let $(h,k)$ be such that \emph{$\Ihk\subset(\udtruei, \udmax]$} and
$\Deltaihk=\payoffi(\lhk)-\threshi> L/2^h$.
Let $\uithk$ be as defined in~\eqref{eqn:uithk}.
Then,
\[
\forall\;t\geq 2, \quad
\PP\left(\flbithk>\threshi \,\wedge\, \VSithk\geq\uithk\right) \leq \frac{6\delta}{n\pi^2t^3}.
\]
\end{lemma}
\begin{proof}
Consider any $t$.
First, the condition on $\VSithk$ implies the following,
\begingroup
\allowdisplaybreaks
\begin{align*}
\VSithk>\frac{4\betat^2}{\left(\Deltaihk-\Lf 2^{-h}\right)^2}
&\;\implies\;
\payoffi(\lhk) - \threshi - \frac{\Lf}{2^h} >  \frac{2\betat}{\sqrt{\VSithk}}
\\
&\;\implies\;
\forall a\in\Ihk, \;\;
\threshi - \frac{\Lf}{2^h}- \payoffi(a) <  -\frac{2\betat}{\sqrt{\VSithk}}.
\numberthis \label{eqn:Withkcondnconcln}
\end{align*}
\endgroup
We now observe:
\begingroup
\allowdisplaybreaks
\begin{align*}
&\VSithk\geq\uithk \,\wedge\,
\flbithk < \threshi
\\
&\hspace{0.2in}
\iff
    \VSithk\geq\uithk
\quad \wedge \quad 
\fbarithk - \frac{\betattilde}{\sqrt{\VSithk}} < \frac{\Lf}{2^h} + \threshi 
\\
&\hspace{0.2in}\iff
    \VSithk\geq\uithk \quad \wedge \quad 
\sum_{s=1}^{t-1}\frac{\indfone((h,k)\in\Pit)}{\sigmais^2}\left(\Xis-\payoffi(\normallocis)\right)
    - \betattilde\sqrt{\VSithk}
\\
&\hspace{3in}
    \,<\,
    \sum_{s=1}^{t-1}\frac{\indfone((h,k)\in\Pit)}{\sigmais^2}
        \left(\threshi + \frac{L}{2^h}-\payoffi(\normallocis)\right),
    \\&\hspace{0.1in}
\\
&\hspace{0.21in}\implies
\sum_{s=1}^{t-1}\frac{\indfone((h,k)\in\Pit)}{\sigmais^2}\left(\Xis-\payoffi(\normallocis)\right)
    < - \betattilde\sqrt{\sum_{s=1}^{t-1}\frac{\indfone((h,k)\in\Pis)}{\sigmais^2}}.
\\
&\hspace{0.21in}\implies
\left|\sum_{s=1}^{t-1}\frac{\indfone((h,k)\in\Pit)}{\sigmais^2}\left(\Xis-\payoffi(\normallocis)\right)\right|
    > \betat\sqrt{\sum_{s=1}^{t-1}\frac{\indfone((h,k)\in\Pis)}{\sigmais^2}}.
\end{align*}
\endgroup
In the third step, we have used~\eqref{eqn:Withkcondnconcln} along with fact that when
$(h,k)\in\Pit$, then $\allocit\in\Ihk$ as in each step in the while loop of \ucrecordfb, we
choose the child which contains $\normallocit$.
In the last step we have considered the absolute value of the LHS and used the fact that
 $\ttilde\geq t$.
The claim follows by applying Lemma~\ref{lem:mds} with the same
$\gamma_s$, $\sigma_s$ , $z_s$, and $\filtrs$
as we did in the proof of Lemma~\ref{lem:Pgoodsets}.
\end{proof}

Next, we consider nodes $(h,k)$ which satisfy the third
case in~\eqref{eqn:Ecalit}.
The proof of the following lemma follows along similar lines to that of Lemma~\ref{lem:Pbadgtr}.

\insertprethmspacing
\begin{lemma}
\label{lem:Pbadles}
Consider user $i$ and let $(h,k)$ be such that $\Ihk\subset[0, \udtruei)$ and
$\Deltaihk=\threshi - \payoffi(\rhk)> L/2^h$.
Let $\uithk$ be as defined in~\eqref{eqn:uithk}.
Then,
\[
\forall\;t\geq 2, \quad
\PP\left(\fubithk<\threshi \,\wedge\, \VSithk\geq\uithk\right) \leq \frac{6\delta}{n\pi^2t^3}.
\]
\end{lemma}

We are now ready to prove Lemma~\ref{lem:PEc}.

\textbf{\textit{Proof of Lemma~\ref{lem:PEc}.}}
Recall the definitions in~\eqref{eqn:Ecalit} and~\eqref{eqn:Ecal}.
By the union bound, we first write,
\[
\PP(\Ecali^c) \leq \sum_{t=1}^\infty \sum_{(h,k)} \PP(\Ecalit^c(h,k)).
\]
We first note that in $t$ rounds, at most $t$ nodes will have been expanded.
For any node that has not been expanded, $\PP(\Ecalit^c(h,k)) = 0$;
this is because, for the first case in~\eqref{eqn:Ecalit}, we have
$\fubithk=\infty$ and $\flbithk=-\infty$ by definition~\eqref{eqn:flbfub};
moreover, $\VSithk=0$ for unexpanded nodes, and therefore the second and third cases do not occur.
Therefore, there are at most $t$ non-zero terms in the inner summation above.

Next, note that any node $(h,k)$ for which $\PP(\Ecalit^c(h,k))$ is non-zero, satisfies
$\PP(\Ecalit^c(h,k))\leq 6\delta/(n\pi^2t^3)$
by Lemmas~\ref{lem:Pgoodsets},~\ref{lem:Pbadgtr}, and~\ref{lem:Pbadles}.
Therefore,
\[
\PP(\Ecali^c)
\leq \sum_{t=1}^\infty \sum_{(h,k)\in\Tcalit} \frac{6\delta}{n\pi^2t^3}
\leq \sum_{t=1}^\infty \frac{6\delta}{n\pi^2t^2}
\leq \frac{\delta}{n}.
\]
The last step above uses the identity $\sum t^{-2} = \pi^2/6$.
The claim follows by a final application of the union bound over the
$n$ users $\PP(\Ecal^c) \leq \sum_{i=1}^n\PP(\Ecali^c) \leq \delta$.
\qedwhite

\subsection{Some Intermediate Results}
\label{sec:nonparintermediateresults}

In this section, we will prove some technical lemmas that will be used in the proofs
of both Theorems~\ref{thm:nonparsp} and~\ref{thm:nonparnsp}.
The first shows that both $\Blbithk, \Bubithk$ are non-decreasing with $k$ for a given $h$.

\insertprethmspacing
\begin{lemma}
\label{lem:BubBlbmonotonicity}
Let $h\geq 0$ and $k_1,k_1\in\NN$ such that $0\leq k_1 < k_2 \leq 2^h$.
Then, the following hold
\[
\Blbithkone \leq \Blbithktwo,
\hspace{0.3in}
\Bubithkone \leq \Bubithktwo.
\]
\end{lemma}
\begin{proof}
We will prove the first result. The proof of the second result follows analogously.
Recall that we update the bounds at two different places in
Algorithms~\ref{alg:nonparfbmodelpart1}--\ref{alg:nonparfbmodelpart4}.
First, the \ucrefreshboundsintrees method (line~\ref{lin:ucrefreshboundsintree})
recomputes the lower confidence bounds $\flbit(h,k)$ and $\Blbit(h,k)$  for all expanded
nodes $(h,k)\in\Tcalit$.
This is invoked by \ucgetudrecs (line~\ref{lin:nonparucgetudrec}) when $t=\ttilde$, i.e.
at the beginning of rounds $1, 2, 4, 8,$ etc.
Second, the \ucupdateboundsonpathtoroots (line~\ref{lin:ucupdateboundsonpathtoroot})
method, which is invoked when we add a new data
point in \ucrecordfbs (line~\ref{lin:nonparucrecordfb}), recomputes $\flbit$ for the nodes the data
point was assigned to, and updates $\Blbit$ for nodes whose values may have been affected.
We will show that in the first case, the refresh operation ensures that $\Blbit$ is
non-decreasing,
and in the second case, the updates preserve monotonicity.

First, consider the \ucrefreshboundsintrees method.
The $\bcheckmax$ variable keeps track of the maximum $\Blbithk$ value as we update
the expanded nodes $(h,k)$ in increasing order of $k$ at height $h$.
When we reach a node $(h,k')$, we ensure
$\Blbithkp\geq\bcheckmax$ which ensures monotonicity.

Second, consider the \ucupdateboundsonpathtoroots method.
Assume that node $k$ at height $h$ is updated at round $t$, and that $\Blbiitt{i,}{t-1}$ is
monotonic at all heights $h$ at round $t-1$.
Observe that the update ensures that $\Blbithk \geq \Blbiitthhkk{i,}{t-1}{h}{k}$ and that
the values $\Blbiitthhkk{i,}{k-1}{h}{k'}$ for $k'<k$ do not change from round $t-1$ to round $t$.
Therefore, since monotonicity is
preserved at round $t-1$, we have $\Blbithkp \leq \Blbithk$ for all $k' \leq k$.
Moreover, the \ucupdateboundsfornodesatsamedepths method updates
$\Blbithkp\leftarrow\max(\Blbithkp, \Blbithk)$ for $k'>k$.
Both these updates ensure monotonicity at the end of round $t$.
\end{proof}

Our second technical result in this section expresses $\Bithk$ as a function
of $\githk$ and the $\Bit$ values of its children.

\insertprethmspacing
\begin{lemma}
\label{lem:fgB}
Let $\git$ be as defined in~\eqref{eqn:githk} and $\Bit$ be as defined in~\eqref{eqn:Bit}.
Then, for $(h,k)\in\Tcalit$,
\[
\Bithk = \min\left(\,\githk, \;\Bitmohk,\;\max\left(\Bithpotkmo, \Bithpotk\right)\, \right).
\]
\end{lemma}
\begin{proof}
First, recall the expressions for $\Blbit,\Bubit$ in~\eqref{eqn:Blb},~\eqref{eqn:Bub}, when
$(h,k)\in\Tcalit$:
\begin{align*}
\Blbithk &= \max(\flbithk, \Blbitmohk, \Blbiitthhkk{i}{t}{h+1}{2k-1}),
\\
\Bubithk &=
\min(\fubithk, \Bubitmohk, \Bubiitthhkk{i}{t}{h+1}{2k}).
\numberthis
\label{eqn:BlbBubexpanded}
\end{align*}
We now expand $\Bithk$ as follows.
\begingroup
\allowdisplaybreaks
\begin{align*}
\Bithk &= \min\left(\,\Bubithk - \threshi, \,\threshi-\Blbithk\right)
\\
&= \min\big(\min(\fubithk, \Bubitmohk, \Bubithpotk)-\threshi, \\
&\hspace{0.60in} \threshi-\max(\flbithk, \Blbitmohk, \Blbithpotkmo)\big)
\\
&=\min\Big(\fubithk-\threshi, \threshi-\flbithk, \Bubitmohk-\threshi, \threshi-\Blbitmohk, \\
&\hspace{0.6in} \Bubithpotk-\threshi,  \threshi-\Blbithpotkmo\Big)
\\
&=\min\Big(\githk,\, \Bitmohk,\,  \max(\Bubithpotkmo, \Bubithpotk)-\threshi,  \\
  &\hspace{2.03in} \threshi-\min(\Blbithpotkmo,\Blbithpotk)\Big)
\\
&=\min\Big(\githk, \Bitmohk, \\
&\hspace{0.6in}
\max\Big(\min\big(\Bubithpotkmo-\threshi, \threshi-\Blbithpotkmo\big) \\
  &\hspace{1.02in} \min\big(\Bubithpotk-\threshi, \threshi-\Blbithpotk\big)\Big)\Big)
\\
&=\min\big(\githk, \Bitmohk, \max(\Bithpotkmo, \Bithpotk) \big).
\end{align*}
\endgroup
In the first step, we have substituted the expressions from the previous display.
The second step follows from the fact
$\min(\min(a,b,c), \min(d, e,f)) = \min(a,b,c,d,e,f)$.
In the third step, we have first used the definitions for $\githk$ and $\Bitmohk$.
Moreover,
we have used monotonicity of $\Blbit,\Bubit$ (Lemma~\ref{lem:BubBlbmonotonicity})
to write $\Blbithpotkmo=\min(\Blbithpotkmo,\Blbithpotk)$ and
$\Bubithpotk=\max(\Bubithpotkmo,\Bubithpotk)$.
The fourth step uses  $\min(\max(a,b),\max(c,d))=\max(\min(a,c),\min(b,d))$.
The last step uses the definition of $\Bit$.
\end{proof}

Our next result shows that for all threshold nodes~\eqref{eqn:thresholdnodes},
under $\Ecal$, the lower and upper bounds $\Blbit,\Bubit$ trap the threshold value $\threshi$.

\insertprethmspacing
\begin{lemma}
\label{lem:Bitthreshnodes}
Consider any user $i$ and let $h\geq 0$. Let $\kthrh$ be as defined above.
Under $\Ecal$, for all $t\geq1$ and $h\geq0$,
we have 
\[
\Blbithkthr<\threshi,
\hspace{0.2in}
\Bubithkthr>\threshi,
\hspace{0.2in}
\Bithkthr>0,
\]
\end{lemma}
\begin{proof}
We will first prove that $\Blbithkthr<\threshi$.
%
Define,
\[
\Blbpithk = \max\left(\,\flbithk, \min\left(\Blbithpotkmo, \Blbithpotk\right)\, \right).
\]
Let $t\geq 1$ be given.
We will show, via induction, that $\Blbpithkthr <\threshi$ for all $h\geq 0$.
Let $(h_t, \kthrhh{h_t})$ be the deepest expanded threshold node at round $t$.
As the base case, we have that $\Blbpit(h,\kthrh)<\threshi$ for all $h>h_t$ since
$\Blbpit(\ell,m) = 0<\threshi$ for any unexpanded node $(\ell, m)$.
Now, assume that $\Blbpithpokthr<\threshi$ for some $h$.
We therefore have,
\begin{align}
\Blbpithkthr \geq \max\left(\,\flbit(h,\kthrh), \Blbpithpokthr\, \right) <\threshi.
\label{eqn:Blbpithpokthrlessthanthreshi}
\end{align}
Here, the first step simply uses the definition for $\Blbpithk$ from the previous display,
observing that $(h+1,\kthrhh{h+1})$ is either $(h+1,2\kthrh-1)$ or  $(h+1,2\kthrh)$.
In the second step, we have used $\flbit(h,\kthrh)<\threshi$ under $\Ecal$~\eqref{eqn:Ecal},
and that $\Blbpithpokthr<\threshi$ by the inductive assumption.

Now, observe that,
\begin{align*}
\Blbithkthr &= \max\left( \flbit(h,\kthrh),
    \Blbiitthhkk{i,}{t-1}{h}{\kthrh}, \Blbiitthhkk{i,}{t}{h+1}{2\kthrh-1} \right)
\\
 &= \max\left( \flbit(h,\kthrh),
       \Blbiitthhkk{i,}{t-1}{h}{\kthrh}, \min(\Blbiitthhkk{i,}{t}{h+1}{2\kthrh-1},
                                             \Blbiitthhkk{i,}{t}{h+1}{2\kthrh}) \right)
\\
 &= \max\left( \Blbpithkthr, \Blbiitthhkk{i,}{t-1}{h}{\kthrh}\right)
\end{align*}
Here, the first step is simply the definition for $\Blbit$~\eqref{eqn:Blb},
and the second step uses monotonicity of
$\Blbiitthhkk{i}{t}{h}{\cdot}$ (Lemma~\ref{lem:BubBlbmonotonicity}).
We can now prove the claim via induction over the rounds $t$.
As the base case, $\Blbiitthhkk{i}{1}{h}{k} = 0 <\threshi$ for all
nodes $(h,k)$ at round $t=1$; therefore, it is also true for all $(h,\kthrh)$.
Now, assume $\Blbiitthhkk{i,}{t-1}{h}{\kthrh}<\threshi$ as the inductive hypothesis.
We then have,
by~\eqref{eqn:Blbpithpokthrlessthanthreshi},
$\Blbithkthr=\max(\Blbpithkthr, \Blbiitthhkk{i,}{t-1}{h}{\kthrh})<\threshi$.

The proof of the second result follows along similar lines and the third result follows from the
first two, as $\Bithk=\min(\Bubithk-\threshi, \threshi-\Blbithk) > 0$.
\end{proof}
%

\subsection{Some Intermediate Definitions}
\label{sec:nonparintermediatedefns}

In this section, we will define a few constructions that we will use in both proofs.

\textbf{Definitions $\hG, \kidown,\kiup, \lidown,\ridown, \liup, \riup$:}
Let $\ntg $ be as given in Assumption~\ref{asm:ntg}.
Let $G\in(0,\ntg]$ be given, and $\epsG$ be as defined in Definition~\ref{def:ntg}.
We define $\hG = \min\{h; \udmax 2^{-\hG} \leq G\epsG/(4L)\}$.
It is straightforward to verify,
\begin{align*}
\frac{4\Lf \udmax}{G\epsG} \leq 2^{\hG} < \frac{8\Lf \udmax}{G\epsG}.
\numberthis
\label{eqn:hG}
\end{align*}
Next, for user $i$, we will consider two nodes $(\hG,\kidown)$, $(\hG, \kiup)$ at height $\hG$ of
the tree such that the following hold. Let $\IhGkidown=[\lidown,\ridown)$, $\IhGkiup=[\liup,\riup)$
be the corresponding intervals.
We have:
\begin{align*}
\lidown \leq \udtruei-\epsG < \ridown < \udtruei-\epsG/2,
\hspace{0.3in}
\udtruei+\epsG/2 < \liup < \udtruei+\epsG \leq \ridown.
\numberthis
\label{eqn:kiupdown}
\end{align*}
We can find such $\kidown,\kiup$ by our definition of $\hG$~\eqref{eqn:hG}.

\textbf{Definitions $\Icalh,\Jcalh$:}
Let $\Icalh$ be the nodes $(h,k)$ at height $h$ which satisfy the following conditions:
\[
\Deltaihk \leq 2\frac{\Lf}{2^h},
\hspace{0.6in}
\exists\,a\in\Ihk,\quad a>\udtruei,
\hspace{0.6in}
\forall\,a\in\Ihk,\quad a<\liup.
\]
We will let $\Jcalh$ be the nodes $(h,k)$ at height $h$ such that,
$(h,k)\notin\Icalh$, $\Ihk\cap(\udtruei, \liup) \neq
\emptyset$, and whose parent is in $\Icalhmo$.
Next, we will bound the sizes of $\Icalh$ and $\Jcalh$.
For any $a\in\Ihk$ where $(h,k))\in\Icalh$,
\[
a - \udtruei \leq (a - \lhk) + (\lhk - \udtruei) \leq
\frac{\udmax}{2^h} + \frac{\udmax}{G}(\payoffi(\lhk) - \threshi)
\leq \frac{\udmax}{2^h} + \frac{2L\udmax}{G2^h}
\defeq {\rm width}_h.
\]
Here, we have used the NTG condition and that the maximum width of any $\Ihk$ is
$\udmax 2^{-h}$.
The size of $\Icalh$ is bounded by the number of intervals of size $\udmax2^{-h}$ in an
interval of size ${\rm width}_h$ and the leftmost interval which contains $\udtruei$.
Using the fact that $\Lf\geq G$ yields the following bound in~\eqref{eqn:IcalhJcalh} for $|\Icalh|$.
Moreover, since the parent of $\Jcalh$ is in $\Icalhmo$, we can also bound $\Jcalh$.
We have:
\begin{align*}
|\Icalh| \leq 1 + \frac{{\rm width}_h}{\udmax 2^{-h}}
\leq 2 + \frac{2L}{G} \leq \frac{4L}{G},
\hspace{0.9in}
|\Jcalh| \leq 2|\Icalhmo| \leq \frac{8L}{G}.
\numberthis \label{eqn:IcalhJcalh}
\end{align*}
Next, for any $a\in\Ihk$, where $(h,k)\in\Icalh$, we can bound $\payoffi(a)-\threshi$,
\begin{align*}
\numberthis \label{eqn:Icalhpayoffbound}
\payoffi(a)-\threshi = \payoffi(a) - \payoffi(\lhk) + \payoffi(\lhk) - \threshi
\leq 2 L 2^{-h} + \Deltaihk
\leq 4 L 2^{-h}.
\end{align*}
Here, we have used~\eqref{eqn:nonparfbmodel} to conclude
$|\payoffi(a) - \payoffi(\lhk)| \leq 2L2^{-h}$.
Similarly, since the parents of nodes in $\Jcalh$ are in $\Icalhmo$, we have
for all $a\in\Ihk$, where $(h,k)\in\Jcalh$,
$\payoffi(a)-\threshi \leq 8 L 2^{-h}$.

\textbf{Definitions $\Icalph,\Jcalph$:}
Similar to above, we
let $\Icalph$ be the nodes $(h,k)$ at height $h$ which satisfy the following three conditions,
\[
\Deltaihk \leq 2\frac{\Lf}{2^h},
\hspace{0.4in}
\exists\,a\in\Ihk\quad a<\udtruei,
\hspace{0.4in}
\forall\,a\in\Ihk\quad a>\ridown.
\]
We will let $\Jcalph$ be the nodes $(h,k)$ at height $h$ such that
$(h,k)\notin\Icalh$, $\Ihk\cap(\ridown, \udtruei) \neq
\emptyset$, and whose parent is in $\Icalphmo$.
By following the same argument to~\eqref{eqn:IcalhJcalh}, we can show
$|\Icalph|  \leq \frac{4L}{G}$,
and
$|\Jcalph| \leq 2|\Icalphmo| \leq \frac{8L}{G}$.
Moreover, by following a similar argument to~\eqref{eqn:Icalhpayoffbound}, we can show
that for any $a\in\Ihk$, where $(h,k)\in\Icalph$,
$\threshi-\payoffi(a) \leq 4 L 2^{-h}$.
Similarly, since the parents of nodes in $\Jcalph$ are in $\Icalphmo$, we have
for all $a\in\Ihk$, where $(h,k)\in\Jcalph$,
$\threshi-\payoffi \leq 8 L 2^{-h}$.

\subsection{Proof of Theorem~\ref{thm:nonparsp}}
\label{sec:nonparspproofs}

In this section, we will prove Theorem~\ref{thm:nonparsp}.
Recall that in Algorithm~\ref{alg:mmflearnsp},
we collect feedback only during the exploration phases,
and moreover, that each user receives just one non-zero allocation during each exploration phase.
We then use the value returned by \ucgetudubs
(line~\ref{lin:nonparucgetudub}, \algnonpar)
as the reported demand for the latter phase in each bracket.
In the remainder of this section, we will denote this value in the $q$\ssth
bracket for user $i$ by $\udbrubiq$.
Additionally, we will let
$\tiq$ denote the round index in the exploration phase of the
$q$\ssth bracket in which user $i$ received a non-zero allocation.

Our first result shows that $\udbrubiq$  is an upper bound on
$\udtruei$ under $\Ecal$.

\insertprethmspacing
\begin{lemma}
\label{lem:nonparucgetudub}
Consider any user $i$ and bracket $q > 0$.
Let $\udbrubiq$ denote the point returned by \ucgetudub.
Under $\Ecal$, we have $\udbrubiq\geq\udtruei$.
\end{lemma}
\begin{proof}
Recall that \ucgetudubs invokes \ucubtraverses to obtain a node $(h,k)$, and then returns
$\udmax k/2^h$.
If $(h,k)$ is a threshold node, i.e. $\udtruei\in \Ihk$,
then the statement is trivially true as $\udmax k/2^h$ is the right-most
point of $\Ihk$.
(It is worth observing that $\udmax k/2^h \notin \Ihk$, unless $k=2^h$,
see~\eqref{eqn:Ihk}.)

If $\udtruei\notin\Ihk$, we will show that at some node threshold node $(\ell, \kthrl)$,
\ucubtraverses chose the right child $(\ell+1, 2\kthrl)$ instead of the left child
$(\ell+1, 2\kthrl-1)$, and moreover that the left child was the threshold node at height $\ell+1$,
i.e. $(\ell+1,\kthrlpo) = (\ell+1, 2\kthrl-1)$.
Therefore, \ucgetudubs returns a point to the right of $\udtruei$,
hence proving the lemma.

To show the above claim, observe that $(0,1)$ is a threshold node and $(h,k)$ is not.
We will let $(\ell, \kthrl)$ be the last threshold node in the path chosen by \ucubtraverse.
Next, assume, by way of contradiction, that the right child was the threshold node.
Under $\Ecal$, by Lemma~\ref{lem:Bitthreshnodes}, we have
$\Blbiitthhkk{i,}{\tiq}{\ell+1}{2\kthrl} <
\threshi$, which means that in line~\ref{lin:nonparucubtraverseifcondn} we will have chosen
the right node.
This is a contradiction since $(\ell,\kthrl)$ was the last threshold node in the path.
Therefore, the left child was the threshold node.
Finally, since $(\ell,\kthrl)$ was the last threshold node, it means that
in line~\ref{lin:nonparucubtraverseifcondn}, we chose the right child.
\end{proof}

The next step in proving Theorem~\ref{thm:nonparsp} is to prove the following lemma.
%
\insertprethmspacing
\begin{lemma}
\label{lem:udbrubiq}
Consider any user $i$ and let $Q>0$.
Assume $\payoffi$ satisfies Assumption~\ref{asm:ntg}.
Let $G\in(0,\ntg]$ and let $\epsG$ be as given in Definition~\ref{def:ntg}.
Let $\tiq$ denote the round index during which user $i$ receives a non-zero allocation in the
exploration phase.
Then, under $\Ecal$,
\emph{
\begin{align*}
\LiQ \,&\defeq\, \sum_{q=1}^Q \left(\udbrubiq - \udtruei \right)
\\
&\leq
C'\frac{\Lf^{\nicefrac{1}{2}}\sigmamax\udmax}{G^{\nicefrac{3}{2}}}  \betatt{2\tiQ} \sqrtQ
+ 586\frac{\Lf\sigmamax^2\udmax^3}{G^3\epsG^3}\betasqtt{2\tiQ}
+ \frac{64\udmax^2\sigmamax^2}{G^2\epsG^2}\betasqtt{2\tiQ}
+ \frac{160\Lf^2\udmax}{G^2} + \frac{16\Lf\udmax}{G\epsG} + 1.
\end{align*}
}
Here, $C'$ is a global constant.
\end{lemma}

Since we collect feedback for each user only once per bracket,
 $\udbrubiq$ is an upper confidence bound constructed using $q$ observations.
Therefore,
the LHS of Lemma~\ref{lem:udbrubiq} can be interpreted as the loss term for the following 
online learning task that occurs over $Q$ rounds: on each round, a learner may evaluate any point
on the interval $[0,\udmax]$; at the end of each round, the learner needs to
output an upper confidence bound for $\udtruei$; her loss at round $q$ is the difference between this
upper bound and the true unit demand $\udtruei$.

With the above interpretation,
we wil find it convenient to express some of the quantities we have seen before
differently.
First consider, $\Nitiqpohk$ which is the number of times $(h,k)$ was in the path chosen by
\ucrecordfbs in the first $\tiq$ rounds~\eqref{eqn:Nithk}.
By observing that the user will have received allocations in the rounds
$\{\tis\}_{s=1}^q$ in the first $q$ brackets, we can write
\[
\Nitiqpohk =\, \sum_{s=1}^{\tiq} \indfone\left((h,k)\in\Pis\right)
=\, \sum_{s=1}^{q} \indfone\left((h,k)\in\Pitis\right).
\]
Similarly, we have
\begingroup
\allowdisplaybreaks
\begin{align*}
\Npitiqpohk =\, \sum_{s=1}^{q} \indfone\hspace{-0.02in}\left((\Hitis,\Kitis)=(h,k)\right),
\hspace{0.2in}
\VSitiqpohk =\, \sum_{s=1}^{q}
\frac{1}{\sigma^2_{i,\tis}}\indfone\left((h,k)\in\Pitis\right).
\end{align*}
\endgroup
Additionally, since there is no resource contention when a user is allocated during the exploration
phase, we may obtain feedback for any allocation we wish.
Therefore, the path $\Pitiq$ chosen by \ucrecordfbs will be the same as the path chosen by
\ucubtraverse.
The following lemma bounds $\Nitiqpohk$ for nodes $(h,k)$ that do not contain $\threshi$.

\begin{lemma}
\label{lem:Nitiqpo}
Consider user $i$ and let $(h,k)$ be such that \emph{$\Ihk\subset(\udtruei, \udmax]$} and
$\Deltaihk=\payoffi(\lhk)-\threshi> L/2^h$.
Under $\Ecal$, for all $t\geq 1$,
\[
\Nitiqpohk \,\leq\,
\sigmamax^2 \max\left( \tauhtiq, \uitiqhk \right) + 1
=
\sigmamax^2 \max\left( \frac{\betatiq^2}{\Lf^2}4^h,
        \frac{4\betatiq^2}{\left(\Deltaihk-L2^{-h}\right)^2} \right) + 1.
\]
\end{lemma}
\begin{proof}
The statement is clearly true for unexpanded nodes, so we will show this for
$(h,k)\in\Tcalitiq$.
First, we will decompose $\Nitiqpohk$
as follows,
\begin{align*}
\Nitiqpohk
&= \sum_{s=1}^q\indfone\left((h,k)\in\Pitis \wedge \Nitishk\leq \sigmamax^2\tauhtiq\right)
\\
 &\hspace{0.5in}
    +  \sum_{s=1}^q\indfone\left((h,k)\in\Pitis \wedge \Nitishk> \sigmamax^2\tauhtiq \right)
\\
&\leq \lceil\sigmamax^2\tauhtiq\rceil 
    +  \sum_{s=\lceil\sigmamax^2\tauhtiq\rceil}^q\indfone\left((h,k)\in\Pitis \wedge \Nitishk>
                \sigmamax^2\tauhtiq \right)
\numberthis\label{eqn:NitiqpohkDecomp}
\end{align*}
In the second step, we have bound the first summation by observing that $\Nitishk$  values are
constant from bracket $s$ to $s+1$ unless $(h,k)\in\Pitis$ in which case they increase by one.
Therefore, at most $\lceil \sigmamax^2\tauhtiq \rceil$ such terms can be non-zero.
For the second sum, we have used the fact that $\Nitishk$ cannot be larger than
$\sigmamax^2\tauhtis$
in the first $\lfloor \sigmamax^2\tauhtis \rfloor$ rounds.

Now, consider bracket $s$. We will show, by way of contradiction, that
the following cannot hold simultaneously.
\[
\Nitishk > \sigmamax^2\max(\uitishk, \tauhtis),
\hspace{0.5in}
(h,k)\in\Pitis.
\]
Recall the definition of the threshold nodes $\{(h', \kthrhh{h'})\}_{h'\geq 0}$ from the
beginning of Appendix~\ref{sec:proofs_nonpar}.
Observe that $(0,1), (h,k)\in\Pitiq$, but $(0,1)$ is a threshold node, while $(h,k)$ is not.
Let $(\ell,\kthrl)$ be the last threshold node in $\Pitiq$ with children
$(\ell+1, \kthrlpo)$  and $(\ell+1, \kthrlpo+1)$;
here, the left child is the threshold node for the same reason outlined in the proof of
Lemma~\ref{lem:nonparucgetudub}, while the right child is an ancestor of $(h,k)$.
Since $\Nitishk >\sigmamax^2\tauhtis$, we have $\VSitishk>\tauhtis$, and therefore,
$\VSiitthhkk{i}{\tis}{h''}{k''} > \tau_{h''\tis}$ for all of $(h,k)$'s ancestors $(h'', k'')$.
This observation, along with the fact that $(h,k)$ is not a threshold node implies that
\ucubtraverses will have reached node $(\ell, \kthrl)$ and then proceeded one more step to choose
the right child $(\ell+1,\kthrlpo+1)$.
Therefore,
$\threshi \geq \Blbiitthhkk{i}{\tis}{\ell+1}{\kthrlpo + 1}$,
since the
\textbf{if} condition (line~\ref{lin:nonparucubtraverseifcondn}, \algnonpar)
in \ucubtraverses chose the right node.
However, $\Nitishk >\sigmamax^2\uitishk$ and the fact that $(h,k)$ is a descendant of
$(\ell+1,\kthrlpo+1)$ implies that
$\VSiitthhkk{i}{\tis}{\ell+1}{\kthrlpo+1}>\VSiitthhkk{i}{\tis}{h}{k}\geq \tauhtis>
\taulpotis$, and therefore, by the definition of the event $\Ecal$~\eqref{eqn:Ecal},
we have $\Blbiitthhkk{i}{\tis}{\ell+1}{\kthrlpo + 1} > \threshi$.
This is a contradiction.

To complete the proof, we will relax~\eqref{eqn:NitiqpohkDecomp} further to otain,
\begingroup
\allowdisplaybreaks
\begin{align*}
&\Nitiqpohk
\leq \lceil\sigmamax^2\tauht\rceil 
    + \hspace{-0.05in}\sum_{s=\lceil\sigmamax^2\tauhtiq\rceil}^t\hspace{-0.1in}\indfone\left((h,k)\in\Pitis
            \,\wedge \Nitishk> \sigmamax^2\max(\tauhtiq, \uitishk)
        \right)
\\
&\hspace{0.4in}
    + \sum_{s=\lceil\sigmamax^2\tauhtiq\rceil}^t\indfone\left((h,k)\in\Pitis
            \,\wedge \Nitishk> \sigmamax^2\tauhtiq
            \,\wedge \Nitishk\leq \sigmamax^2\uitishk
        \right)
\\
&\hspace{0.1in}
\leq \, \lceil\sigmamax^2\tauht\rceil +
\\
&\hspace{0.4in}
    \sum_{s=\lceil\sigmamax^2\tauhtiq\rceil}^t\indfone\left((h,k)\in\Pitis
            \,\wedge \Nitishk> \sigmamax^2\tauhtiq
            \,\wedge \Nitishk\leq \sigmamax^2\uitishk
        \right)
\end{align*}
\endgroup
Here, the first sum in the first step vanishes by the above contradiction.
To bound the remaining sum,
observe that if $\uitiqhk<\tauhtiq$, each term in the sum is $0$ and we have
$\Nitiqpohk \leq \lceil\sigmamax^2\tauht\rceil$.
If $\uitiqhk>\tauhtiq$, we can use that the sum starts at $\lceil\sigmamax^2\tauhtiq\rceil$
and a similar reasoning as we did in~\eqref{eqn:NitiqpohkDecomp}, to show that there are
at most $\lceil\sigmamax^2\uitiqhk\rceil - \lceil\sigmamax^2\tauhtiq\rceil$ non-zero terms
in this summation.
Therefore,
$\Nitiqpohk \leq \lceil\sigmamax^2\uitiqhk\rceil$.
\end{proof}

We can now bound $\LiQ$.

\textbf{Proof of Lemma~\ref{lem:udbrubiq}:}
Assume that $\Ecal$ holds, so that all claims hold with probability at least
$1-\delta$.
Recall the definitions of the quantities $\hG, \kiup, \liup, \riup$
from Appendix~\ref{sec:nonparintermediatedefns}.
We can bound $\LiQ$ as follows,
\begingroup
\allowdisplaybreaks
\begin{align*}
&\LiQ = \sum_{q=1}^Q \left(\udbrubiq - \udtruei \right)
    \leq \udmax\left(\LioneQ + \LitwoQ + \LithreeQ + \frac{1}{G}\LifourQ \right)
\hspace{0.4in}\text{where},
\label{eqn:LiQdecomp} \numberthis
\\
&\LioneQ = \sum_{q=1}^Q \indfone(\Hitiq < \hG),
\hspace{0.8in}
\LitwoQ = \sum_{q=1}^Q \indfone(\Hitiq \geq \hG \,\wedge\, \udbrubiq \in [\riup, \udmax]),
\\
&\LithreeQ = \sum_{q=1}^Q \indfone(\Hitiq \geq \hG \,\wedge\, \udbrubiq \in [\liup, \riup)),
\\
&\LifourQ = \sum_{q=1}^Q \indfone(\Hitiq \geq \hG \,\wedge\, \udbrubiq \in (\udtruei, \liup])\cdot
        (\payoffi(\udbrubiq) - \threshi),
\end{align*}
\endgroup
Here, $\LioneQ$ bounds the number of rounds in which $\Hitiq<\hG$; therefore,
when bounding the remaining terms we can
focus on the rounds where $\Hitiq\geq \hG$.
$\LitwoQ$ considers evaluations greather than or equal to $\liup$ and
$\LithreeQ$ considers evaluations in $[\riup, \liup)$, i.e. rounds $q$
where $(\hG, \kiup)\in\Pitiq$.
In both cases, we simply bound $(\udbrubiq-\udtruei) \leq \udmax$.
Finally, $\LifourQ$ accounts for the rounds when $\udbrubiq \in (\udtruei, \liup]$.
Here, we have used the fact that when $a \in (\udtruei, \liup)
\subset (\udtruei-\epsG, \udtruei+\epsG)$,
by Assumption~\ref{asm:ntg},
we have $|a - \udtruei| \leq (\udmax/G)|\payoffi(a) - \threshi|$.
We will now bound each of the above terms individually.
For brevity, we will denote $\betattildeiQ=\beta$, where, recall $\betat$ is as defined
in~\eqref{eqn:betataudefn}, and $\betattildeiQ$ is the value used in the exploration phase
round for user $i$ in bracket $Q$.

\Litboundheader{Bounding $\LioneQ$:}
We will bound $\LioneQ$ by summing up the $\NpitiQpohk$ values for all nodes up to height $\hG-1$.
When assigning points to nodes in \ucrecordfb, recall that we always proceed to the child node
of $(h,k)$ 
if $\VSithk>\tauht$, in which case it is not counted in $\Npithk$.
Therefore, $\NpitiQpohk\leq 1 + \sigmamax^2\tauhtiQ$.
This leads us to the following bound,
\begin{align*}
\LioneQ
&= \sum_{h=0}^{\hG-1} \sum_{k=1}^{2^h} \NpitiQpohk
\leq \sum_{h=0}^{\hG-1} 2^h \left( 1+\sigmamax^2\frac{\beta^2}{L^2} 4^h \right)
\leq 2^{\hG} + \frac{\sigmamax^2\beta^2}{L^2}\sum_{h=0}^{\hG-1} 8^h
\\
&\leq 2^{\hG} + \frac{\sigmamax^2\beta^2}{7\Lf^2} 8^{\hG}
\leq \frac{8\Lf\udmax}{G\epsG} + \frac{512}{7}\frac{L\sigmamax^2\udmax^3\beta^2}{G^3\epsG^3}.
\numberthis\label{eqn:LioneQbound}
\end{align*}
Here, the fourth step uses the fact that $\sum_{h=0}^m 8^h = (8^{m+1}-1)/7$, and the last step
uses~\eqref{eqn:hG}.

\Litboundheader{Bounding $\LitwoQ$:}
First observe that we can write
$\LitwoQ= \sum_{k=1}^{2^\hG}\indfone(k>\kiup) \NitiQpohGk$.
This follows from  our definition of $\hG, \kiup, \kidown$,
and two observations.
First, we can write $\udbrubiq=\udmax k/2^h$ for some $(h,k)$, in which case, the allocation
returned by $\ucgetudrecforub$ in the $q$\ssth round will have been in $\Ihk$.
Second, $\Nithk$ counts all evaluations at node $h$ and its children.
Moreover, we can use the NTG condition and the definition of $\hG$  to conclude
\begin{align*}
\forall\,  k>\kiup, \quad
\Deltaiihhkk{i}{\hG}{k} &= \payoffi(\lhk) - \threshi
= \payoffi(\lhk) - \payoffi(\udtruei+\epsG) + \payoffi(\udtruei+\epsG) - \payoffi(\udtruei)
\\
&\geq \payoffi(\udtruei+\epsG) - \payoffi(\udtruei)
\geq \frac{G\epsG}{\udmax},
\end{align*}
We will now apply Lemma~\ref{lem:udbrubiq}.
Using the above conclusion and the fact that
$\frac{G\epsG}{8\udmax} < \frac{L}{2^\hG} < \frac{G\epsG}{4\udmax}$ from~\eqref{eqn:hG},
we have,
%
\begin{align*}
\NitiQpohGk &\leq
\sigmamax^2 \max\left( \frac{\beta^2}{L^2}4^{\hG},
        \frac{4\beta^2}{\left(\Deltaihk-L2^{-h}\right)^2} \right) + 1
\\
&\leq
\sigmamax^2 \beta^2 \max\left( \frac{64\udmax^2}{G^2\epsG^2}, \frac{64\udmax^2}{9G^2\epsG^2}
        \right) + 1
\leq
\frac{64\udmax^2\sigmamax^2\beta^2}{G^2\epsG^2} + 1.
\end{align*}
Finally, by applying~\eqref{eqn:hG} once again, we have,
\begin{align*}
\LitwoQ &= \sum_{k=1}^{2^\hG}\indfone(k>\kiup) \NitiQpohGk
\leq 2^\hG  \left(\frac{64\udmax^2\sigmamax^2\beta^2}{G^2\epsG^2} + 1 \right)
\\
&\leq \frac{512\Lf\udmax^3\sigmamax^2}{G^3\epsG^3}\beta^2 + \frac{8\Lf\udmax}{G\epsG}.
\numberthis\label{eqn:LitwoQbound}
\end{align*}

\Litboundheader{Bounding $\LithreeQ$:}
Observe that we can write
$\LithreeQ= \Niitthhkk{i}{\tiQ+1}{\hG}{\kiup}$.
By the NTG condition and the definition of $\hG$~\eqref{eqn:hG} we have,
\begin{align*}
\Deltaiihhkk{i}{\hG}{\kiup} = \payoffi(\liup) - \threshi \geq \frac{G}{\udmax}(\liup - \udtruei) >
\frac{G\epsG}{2\udmax}.
\end{align*}
Additionally, by~\eqref{eqn:hG}, we have
$\frac{L}{2^\hG} < \frac{G\epsG}{4\udmax}$.
Applying Lemma~\ref{lem:udbrubiq} once again, we get,
\begin{align*}
\LithreeQ= \Niitthhkk{i,}{\tiQ+1}{\hG}{\kidown}
\leq \frac{64\udmax^2\sigmamax^2}{G^2\epsG^2}\beta^2 + 1.
\numberthis\label{eqn:LithreeQbound}
\end{align*}

\Litboundheader{Bounding $\LifourQ$:}
Recall the definitions of $\Icalh, \Jcalh$ from Appendix~\ref{sec:nonparintermediatedefns}.
Let $H\geq\hG$ be a positive integer whose value will be determined shortly.
We will define three subsets of nodes $\Ncalone, \Ncaltwo,\Ncalthree$ in our infinite tree.
Recall, from the beginning of Appendix~\ref{sec:proofs_nonpar}, that a descendant of a node
$(h,k)$ could be $(h,k)$ or its children, its children's children, etc.
Let $\Ncalone$ denote the descendants of $\IcalH$,
let $\Ncaltwo = \bigcup_{h=\hG}^{H-1} \Icalh$,
and let $\Ncalthree$ denote the descendants of
$\bigcup_{h=\hG}^{H} \Jcalh$.
We can now see that $\LifourQ$ can be bound as follows.
\begin{align*}
\LifourQ \leq \Lcalone + \Lcaltwo + \Lcalthree,
\hspace{0.3in}
\Lcal_i =  \sum_{q=1}^Q \indfone((\Hitiq,\Kitiq)\in\Ncal_i
        \,\wedge\,\Hitiq\geq\hG)
    \cdot(\payoffi(\udbrubiq)-\threshi).
\end{align*}
Recall from~\eqref{eqn:Nithk}, that $(\Hitiq,\Kitiq)$ are the last nodes in the path $\Pitiq$ chosen
by \ucrecordfbs in round $\tiq$ of bracket $q$,
and moreover, in that bracket $q$, we used $\udbrubiq$ as the upper bound in the latter phase of
Algorithm~\ref{alg:mmflearnsp}.
Therefore, when 
$\udbrubiq\in(\udtruei, \liup]$ and $\Hitiq\geq \hG$,
we also have
$(\Hitiq,\Kitiq) \in \Ncalone\cup\Ncaltwo\cup\Ncalthree$.
We will now bound $\Lcalone, \Lcaltwo$, and $\Lcalthree$.

By~\eqref{eqn:Icalhpayoffbound}, we have $\payoffi(a)-\threshi<4L2^{-H}$
when $a\in \Ihk$, for any $(h,k)\in\Ncalone$.
This leads to the following straightforward bound for $\Lcalone$,
\begin{align*}
\numberthis\label{eqn:LcaloneboundQ}
\Lcalone =  \sum_{s=1}^Q \indfone((\Hitiq,\Kitiq)\in\Ncalone)
    \cdot(\payoffi(\udbrubiq)-\threshi)
\leq \frac{4L}{2^{H}} T.
\end{align*}
Next, we bound $\Lcaltwo$ as shown below.
\begin{align*}
\Lcaltwo
&= \sum_{h=\hG}^{H-1}\sum_{(h,k)\in\Icalh} \frac{4\Lf}{2^h}\NpitiQpohk
\leq  \sum_{h=\hG}^{H-1} \frac{4\Lf}{2^h} \frac{4\Lf}{G}
\left(1 + \frac{ 4^h\sigmamax^2 \beta^2}{\Lf^2}\right)
\\
&\leq \frac{16\Lf^2}{G} \sum_{h=\hG}^{H-1} 2^{-h}
+  \frac{16\sigmamax^2 \beta^2}{G} \sum_{h=\hG}^{H-1} 2^{h}
\leq \frac{32\Lf^2}{G} 
+  \frac{16\sigmamax^2 \beta^2}{G} 2^H.
\numberthis\label{eqn:LcaltwoboundQ}
\end{align*}
In the first step we have used $\Npitiqpohk$ (instead of $\Nitiqpohk$),
since $\Lcaltwo$ only counts allocations where $(\Hit,\Kit)$ were in
$\bigcup_{h=\hG}^{H-1} \Icalh$.
We have also used~\eqref{eqn:Icalhpayoffbound} to bound $\payoffi(\udbrubiq) - \threshi$.
In the second step, by the same reasoning used in the bound for $\LioneQ$, we have
$\NpitiQpohk\leq 1 + \sigmamax^2\tauhtiQ$ for any node $(h,k)$;
moreover, we have used~\eqref{eqn:IcalhJcalh} to bound the number of nodes in $\Icalh$.
The remaining steps are obtained by algebraic manipulations.

Finally, we bound $\Lcalthree$ as follows.
\begin{align*}
\Lcalthree
&= \sum_{h=\hG}^{H}\sum_{(h,k)\in\Jcalh} \frac{8\Lf}{2^{h}}\NitiQpohk
\leq  \sum_{h=\hG}^{H} \frac{8\Lf}{2^h} \frac{8\Lf}{G}
\left(1 + \frac{ 4\beta^2\sigmamax^2}{\Lf^2}4^h\right)
\\
&\leq \frac{64\Lf^2}{G} \sum_{h=\hG}^{H} 2^{-h}
+  \frac{256\sigmamax^2 \beta^2}{G} \sum_{h=\hG}^{H} 2^{h}
\leq \frac{128\Lf^2}{G} 
+  \frac{512\sigmamax^2 \beta^2}{G} 2^H.
\numberthis\label{eqn:LcalthreeboundQ}
\end{align*}
Above, in the first step we have used $\NitiQpohk$ (instead of $\NpitiQpohk$),
since $\Lcalthree$ counts allocations where $(\Hit,\Kit)$ belonged to the descendants of
$\bigcup_{h=\hG}^{H} \Jcalh$;
additionally, we have used~\eqref{eqn:Icalhpayoffbound} and the fact that
parents of nodes in $\Jcalh$ are in $\Icalhmo$ to bound $\payoffi(\allocit) - \threshi$.
In the second step, first we have used~\eqref{eqn:IcalhJcalh} to bound the number of nodes in
$\Jcalh$;
to bound the number of evaluations in each such node, we have applied Lemma~\ref{lem:udbrubiq}
along with the fact that $\Deltaihk>2L2^{-h}$ for nodes in $\Jcalh$ by their definition;
therefore,
\[
\NitiQpohk \leq 1 + \sigmamax^2 \max\left(\frac{\beta^2}{L^2}4^h,
\frac{4\beta^2}{\left(\Deltaihk-L2^{-h}\right)^2} \right)
\leq 1 + \frac{4\sigmamax^2 \beta^2}{\Lf^2} 4^h.
\]
The remaining steps in~\eqref{eqn:LcalthreeboundQ} are obtained via algebraic manipulations.
Combining~\eqref{eqn:LcaloneboundQ},~\eqref{eqn:LcaltwoboundQ}, and~\eqref{eqn:LcalthreeboundQ}
results in the following bound for $\LifourQ$.
\begin{align*}
\LifourQ \leq \frac{4L}{2^{H}} Q
+  \frac{528\sigmamax^2 \beta^2}{G} 2^H
+ \frac{160\Lf^2}{G}
\leq C'\frac{\Lf^{\nicefrac{1}{2}}\sigmamax \beta T^{\nicefrac{1}{2}}}{G^{\nicefrac{1}{2}}}
+ \frac{160\Lf^2}{G}
\numberthis\label{eqn:LifourQbound}
\end{align*}
Here $C'$ is a global constant.
The last step is obtained by choosing $H$ such that $2^H \asymp \frac{\sqrt{GLQ}}{\sigmamax\beta}$.

The lemma now follows by combining~\eqref{eqn:LiQdecomp},%
~\eqref{eqn:LioneQbound},%
~\eqref{eqn:LitwoQbound},%
~\eqref{eqn:LithreeQbound},%
and
~\eqref{eqn:LifourQbound}.
Moreover, from~\eqref{eqn:trelationttilde} we have
$\beta=\betattildeiQ\leq \betatt{2\tiQ}$.
\qedwhite

We are now ready to prove Theorem~\ref{thm:nonparsp}.

\textbf{\textit{Proof of Theorem~\ref{thm:nonparsp}.}}
Assume that $\Ecal$ holds, so that all claims hold with probability at least
$1-\delta$.

\textbf{Efficiency:}
We will decompose and bound the loss as shown below:
\begingroup
\allowdisplaybreaks
\begin{align*}
\LOTT
\,&= \sum_{t\in\Expl} \lott + \sum_{t\notin\Expl} \lott
\,\leq n\qT + \sum_{t\notin\Expl} \lott
\leq n\qT + \sum_{t\notin\Expl} \sum_{i=1}^n (\demubit - \demtrueit)
\numberthis
\label{eqn:LTspnonpar}
\\
&\leq n\qT + \volmax\sum_{i=1}^n \sum_{t\notin\Expl} (\udubit - \udtruei)
\;\leq n\qT + \volmax\sum_{i=1}^n \sum_{q=1}^{\qT} r'(q)(\udbrubiq - \udtruei)
\\
&\leq n\qT + \volmax r'(\qT)\sum_{i=1}^n \sum_{q=1}^{\qT}(\udbrubiq - \udtruei)
\leq n\qT + \frac{5\volmax}{6}n\qT^{\nicefrac{1}{2}} \sum_{i=1}^n \LiqT 
\\
&\leq
n\qT + \frac{5\volmax}{6}n^2
\bigg(C'\frac{\Lf^{\nicefrac{1}{2}}\sigmamax\udmax}{G^{\nicefrac{3}{2}}}  \betatwoT \qT
+ 586\frac{\Lf\sigmamax^2\udmax^3}{G^3\epsG^3}\betatwoT^2 \sqrtqT
+ \frac{64\udmax^2\sigmamax^2}{G^2\epsG^2}\betatwoT^2 \sqrtqT
\\
&\hspace{1.5in}
+  \sqrtqT\left(\frac{160\Lf^2\udmax}{G^2} + \frac{4\Lf\udmax}{G\epsG} + 1\right)
\bigg)
\end{align*}
\endgroup
Here, in the third step we have used 
Lemma~\ref{lem:ubinstloss} to bound the instantaneous losses in the second sum.
In the fourth step, we have observed
$\demubit-\demtrueit = \volit(\udubit - \udtruei) \leq \volmax(\udubit - \udtruei)$.
For the fifth step,  we have upper expressed the second sum in terms of $\udbrubiq$, by summing
over the brackets---note that there may be more terms in the summation as we consider all rounds
in the {$\qT$}\ssth bracket.
In the sixth step, we have observed that $r'$ is an increasing function,
and in the seventh step, we have used
that $r'(q)=\lfloor 5nq/6\rfloor$ (line~\ref{lin:nonparrpq}, \algnonpar),
and observed that the term inside the summation is $\LiQ$ from Lemma~\ref{lem:udbrubiq}.
The last step applies the bound in Lemma~\ref{lem:udbrubiq} while also observing that
$\tiqT \leq T$ and that $\betat$  is increasing in $t$.
The bound on the loss follows by an application of Lemma~\ref{lem:qTglm} with $a=n$ to
upper bound $\qT$.

\textbf{Fairness:}
This follows by applying Lemma~\ref{lem:spfairness} with $r=n$ and the upper bound for
$\qT$ in Lemma~\ref{lem:qTglm}.

\textbf{Strategy-proofness:}
For a policy $\pi$, denote $\normallocpiit = \allocpiit/\volit$.
Using Lemma~\ref{lem:spsp} we obtain,
\begingroup
\allowdisplaybreaks
\begin{align*}
\UpiiT - \UiT &\leq \sum_{t\in \Expl} (\utili(\normallocpiit) - \utili(\normallocit))
 \leq \Lipi \hspace{-0.10in}\sum_{t\in \Expl, \udit>0} (\normallocpiit - \normallocit)^+
\leq \Lipi\udmax\qT \leq 3 \Lipi \udmax \nmtwth \Ttwth.
\end{align*}
\endgroup
Here, we have observed that we allocate to each user only once in each exploration phase.
The last step uses Lemma~\ref{lem:qTglm}.
\qedwhite

\subsection{Proof of Theorem~\ref{thm:nonparnsp}}
\label{sec:nonparnspproofs}

In this section, we will prove Theorem~\ref{thm:nonparnsp}.
We will begin with two lemmas to bound the number of allocations away from the demand
for each user $i$ assuming that $\Ecal$ holds~\eqref{eqn:Ecal}.
The first of these, bounds the number of allocations larger than the demand.

\insertprethmspacing
\begin{lemma}
\label{lem:Nbadgtr}
Consider user $i$ and let $(h,k)$ be such that \emph{$\Ihk\subset(\udtruei, \udmax]$} and
$\Deltaihk=\payoffi(\lhk)-\threshi> L/2^h$.
Under $\Ecal$, for all $t\geq 1$,
\[
\Nitpohk \,\leq\,
\sigmamax^2 \max\left( \tauht, \uithk \right) + 1
=
\sigmamax^2 \max\left( \frac{\betattilde^2}{L^2}4^h,
        \frac{4\betattilde^2}{\left(\Deltaihk-L2^{-h}\right)^2} \right) + 1.
\]
\end{lemma}
\begin{proof}
The statement is clearly true for unexpanded nodes, so let we will show this for
$(h,k)\in\Tcalit$.
First, we will decompose $\Nitpohk=\sum_{s=1}^{t} \indfone\left((h,k)\in\Pis\right)$ as follows,
\begin{align*}
\Nitpohk 
&= \sum_{s=1}^t\indfone\left((h,k)\in\Pis \wedge \Nishk\leq \sigmamax^2\tauht\right)
\\
 &\hspace{0.5in}
    +  \sum_{s=1}^t\indfone\left((h,k)\in\Pis \wedge \Nishk> \sigmamax^2\tauht \right)
\\
&\leq \lceil\sigmamax^2\tauht\rceil 
    +  \sum_{s=\lceil\sigmamax^2\tauht\rceil}^t\indfone\left((h,k)\in\Pis \wedge \Nishk>
                \sigmamax^2\tauht \right)
\numberthis\label{eqn:NitpohkDecomp}
\end{align*}
In the second step, we have bound the first summation by observing that $\Nishk$  values are
constant from round $s$ to $s+1$ unless $(h,k)\in\Pis$ in which case they increase by one.
Therefore, at most $\lceil \sigmamax^2\tauht \rceil$ such terms can be non-zero.
For the second sum, we have used the fact that $\Nishk$ cannot be larger than $\sigmamax^2\tauht$
in the first $\lfloor \sigmamax^2\tauht \rfloor$ rounds.
To bound the second summation, we will consider round $s$ and show, by way of contradiction, that
the following cannot hold simultaneously,
\begin{align*}
\Nishk > \uishk\sigmamax^2,
\hspace{0.3in}
\Nishk>\sigmamax^2\tauht,
\hspace{0.3in}
\udis\geq\lhk.
\numberthis
\label{eqn:Nbadgtrcontradiction}
\end{align*}
Recall that $\udis$ is the recommended unit demand at round $s$.
First observe that under $\Ecal$,
\begin{align*}
&\Nishk>\uishk\sigmamax^2
\implies \VSishk>\uishk, 
\implies \Blbishk \geq \flbishk\geq \threshi, \\
&\hspace{0.2in}
\implies \Blbis(h,k') \geq \threshi, \quad\text{for all $k'\geq k$}, \\
&\hspace{0.2in}
\implies \Bis(h,k') =\min(\Bubis(h,k')-\threshi, \threshi-\Blbis(h,k')) < 0,
    \quad\text{for all $k'\geq k$}, 
\numberthis
\label{eqn:Bithkpbound}
\end{align*}
The second step uses the conditions for $\Ecal$~\eqref{eqn:Ecalit} and the definition
for $\Blbithk$ for expanded nodes~\eqref{eqn:BlbBubexpanded}.
The third step uses Lemma~\ref{lem:BubBlbmonotonicity},
and the last step simply uses the definition of $\Bit$~\eqref{eqn:Bit}.
The conclusion in~\eqref{eqn:Bithkpbound} says that all of the nodes to the right of $(h,k)$ at
height $h$
will have negative $\Bishk$ value if $\Nishk>\uishk$.

Note that if user $i$ received an allocation $\allocis=\normallocis\volis$ at round $s$, it could be
because the reported demand was
$\demis=\udis\volis=\normallocis\volis$, or because it was
$\demis=\udis\volis>\normallocis\volis$, but
received less due to resource contention.
Let $(h,k')$ be the node at height $h$ that contained the recommendations from
\ucgetudrecs (line~\ref{lin:nonparucgetudrec}, \algnonpar) which was used as the reported demand
for \mmfs (line~\ref{lin:nspudirec}, Algorithm~\ref{alg:mmflearnnsp}), i.e. $\udis\in\Ihhkk{h}{k'}$.
Let $(\ell, \kthrl)$ be the common ancestor of $(h, \kthrh)$, $(h, k)$, and $(h, k')$
whose children are $(\ell+1, \kthrlpo)$, $(\ell+1,\kthrlpo+1)$;
here, $(\ell+1, \kthrlpo)$ is the left child of $(\ell, \kthrl)$ and an ancestor of $(h,\kthrh)$,
while $(\ell+1,\kthrlpo+1)$
is the right child; this is the case since $\Ihk\subset(\udtruei, \udmax]$.
Since $\kthrh<k\leq k'$, there are three possible cases here, all of which lead to a contradiction
of the statemen in~\eqref{eqn:Nbadgtrcontradiction}.
\begin{enumerate}
\item \emph{$(\ell+1, \kthrlpo)$ is a common ancestor of $(h, \kthrh)$ and $(h, k)$, while
      $(\ell+1, \kthrlpo+1)$ is an ancestor of $(h, k')$}:
        By~\eqref{eqn:Bithkpbound}, all of $(\ell+1, \kthrlpo+1)$'s descendants at height
        $h$ will have negative $\Bis$ value.
        By Lemma~\ref{lem:fgB}, all of their parents at height $h-1$ will
        also have negative $\Bis$ value. Continuing this argument, we have that
        $\Bis(h+1, \kthrlpo+1)< 0$.
Since $\VSishk>\tauht$, we also have $\VSiitthhkk{i}{s}{h''}{k''}>\tauht$ for all of
$(h,k)$'s ancestors $(h'',k'')$, including in particular, $(\ell,\kthrlpo)$ and its ancestors.
Moreover, since $\udis$ was the chosen recommendation, the while loop in \ucgetudrecs
will have reached $(\ell,\kthrlpo)$ and have proceeded one more step to select
$(\ell+1,\kthrlpo+1)$. Since we always choose the child with the higher $\Bis$ value at each stage,
$\Bis(\ell+1, \kthrlpo) \leq \Bis(\ell+1, \kthrlpo+1) < 0$.
However, by Lemma~\ref{lem:Bitthreshnodes}, $\Bis(\ell+1,\kthrlpo)>0$, which is a contradiction.

\item $(\ell+1, \kthrlpo+1)$ is an ancestor of $(h, k)$ with
        $k=k'$:
    Since $\VSishk>\tauht$, we have that the while loop in \ucgetudrecs chose each node on the path
    from $(0,1)$ to $(h,k)$. Since the $\Bis$ value never decreases along a chosen path,
\[
\Bis(\ell+1, \kthrlpo) \leq \Bis(\ell+1, \kthrlpo+1) \leq \dots \leq \Bis(h, k) < 0.
\]
Once again, the contradiction follows from the fact that $\Bis(\ell+1,\kthrlpo)>0$ by
Lemma~\ref{lem:Bitthreshnodes}.

\item $(\ell+1, \kthrlpo+1)$ is a common ancestor ancestor of $(h, k)$ and $(h, k')$ with
        $k<k'$:
Let $(p, q)$ be the common ancestor of $(h,k)$ and $(h,k')$, with the left child $(p+1,2q-1)$
leading to $(h,k)$ and the right child $(p+1,2q)$ leading to $(h,k')$.
Since $\VSishk>\tauht$ it means that \ucgetudrecs chose each node on the path from $(0,1)$ to
$(p,q)$ and proceeded to choose $(p+1,2q)$.
Since, by Lemma~\ref{lem:Bitthreshnodes}, $\Bis(\ell,\kthrl)>0$ and since the $\Bis$ value never
decreases along a chosen path, $\Bis(p+1,2q)>0$.
However, by~\eqref{eqn:Bithkpbound} and a reasoning similar to the first case above,
$\Bis(p+1,2q)<0$, which is a contradiction.
\end{enumerate}

%

To complete the proof, note that \mmfs does not allocate more than the requested demand for a user
$i$; therefore,
$
(h,k)\in\Pis
\implies
\normallocis\in\Pis
\implies
\udis \geq \lhk
$.
Using~\eqref{eqn:NitpohkDecomp}, we write,
\begingroup
\allowdisplaybreaks
\begin{align*}
\Nitpohk \;
&\leq \lceil\sigmamax^2\tauht\rceil \,+ \\
&\hspace{0.2in}
    \sum_{s=\lceil\sigmamax^2\tauht\rceil}^t\indfone\left((h,k)\in\Pis \,\wedge\, \Nishk>
                \sigmamax^2\tauht \,\wedge\, \Nishk \leq \sigmamax^2\uithk
            \right) \,+
\\
&\hspace{0.4in}
    \sum_{s=\lceil\sigmamax^2\tauht\rceil}^t\indfone\left((h,k)\in\Pis \,\wedge\, \Nishk>
                \sigmamax^2\tauht \,\wedge\, \Nishk > \sigmamax^2\uithk
            \right)
\\
&\leq \lceil\sigmamax^2\tauht\rceil \,+ \\
&\hspace{0.2in}
    \sum_{s=\lceil\sigmamax^2\tauht\rceil}^t\indfone\left((h,k)\in\Pis \,\wedge\, \Nishk>
                \sigmamax^2\tauht \,\wedge\, \Nishk \leq \sigmamax^2\uithk
            \right)
\end{align*}
\endgroup
In the last step, each term in the second summation vanishes by the above
contradiction and the facts, $\uithk\geq\uishk$ for all $s\leq t$,
and
$(h,k)\in\Pis\implies\udis\geq\lhk$.
Now, if $\uithk<\tauht$, each term in the summation is $0$ as
$\Nishk> \sigmamax^2\tauht \,\wedge\, \Nishk \leq \sigmamax^2\uithk$ cannot be true;
therefore, $\Nitpohk\leq \lceil\sigmamax^2\tauht\rceil$.
If $\uithk\geq\tauht$, using a similar reasoning as we did in~\eqref{eqn:NitpohkDecomp},
and the fact that the summation starts at $\lceil\sigmamax^2\tauht\rceil$,
we can bound the sum by $\lceil\sigmamax^2\uithk\rceil - \lceil\sigmamax^2\tauht\rceil$;
Therefore, $\Nitpohk\leq \lceil\sigmamax^2\uithk\rceil$.
\end{proof}

\insertprethmspacing
\begin{remark}\emph{
It is worth pointing out why monotonicity of the confidence intervals is necessary for the
correctness of the algorithm.
If they were not, the $\Bit$ values may be large for large
allocation values since users may not have received large allocations
(see for example, the bottom figure in Figure~\ref{fig:treeillus}); consequently, the
$\Bubit$ values will be large and the $\Blbit$ values will be small.
Hence, the $\Bit=\min(\threshi-\Blbit,\Bubit-\threshi)$ values will also be large causing
\ucgetudrecs to recommend a large allocation.
This could lead to pathological situations where the \ucgetudrecs keeps recommending large
allocations, but a smaller allocation is repeatedly chosen due to contention on limited resources.
If the confidence intervals are monotonic,
then $\Bit$ will be small for large allocations even if they have not been evaluated.
This ensures that \ucgetudrecs does
not recommend large allocations.
In particular, if the lower confidence bound is larger than $\threshi$ for any allocation,
monotonicity ensures that $\Bit$ is negative for any larger allocation.
The proof of Lemma~\ref{lem:Nbadgtr} captures this intuition.
}
\label{rem:monotonicity}
\end{remark}

Our second lemma bounds the number of allocations smaller than the demand
whenever
\mmfs allocates an amount of resources equal to the demand.
For this, we first define $\VSrecithk,\Nrecithk$ below, which are variants of $\VSithk,\Nithk$,
but only consider rounds when  the allocation was equal to the recommendation.
We have:
\begingroup
\allowdisplaybreaks
\begin{align*}
\VSrecithk &= \sum_{s=1}^{t-1} \frac{1}{\sigmais^2}
    \indfone\left(\udit=\normallocit \,\wedge\, (h,k)\in\Pis\right),
\\
\Nrecithk &= \sum_{s=1}^{t-1} \indfone\left(\udit=\normallocit \,\wedge\, (h,k)\in\Pis\right).
\numberthis \label{eqn:Nrecithk}
\end{align*}
\endgroup
Our second main lemma in this section bounds $\Nrecithk$ for allocations lower than $\udtruei$.
Its proof follows along similar lines to the proof of Lemma~\ref{lem:Nbadgtr} in several places;
as such, we will frequently refer to calculations from above.

\insertprethmspacing
\begin{lemma}
\label{lem:Nbadles}
Consider user $i$ and let $(h,k)$ be such that $\Ihk\subset[0, \udtruei)$ and
$\Deltaihk=\threshi - \payoffi(\rhk)> L/2^h$.
Under $\Ecal$,
Let $\uithk$ be as defined in~\eqref{eqn:uithk}.
Then, for all $t\geq 1$,
\[
\Nrecitpohk \,\leq\, 
\sigmamax^2 \max\left( \tauht, \uithk \right) + 1
=
\sigmamax^2 \max\left( \frac{\betattilde^2}{L^2}4^h,
        \frac{4\betattilde^2}{\left(\Deltaihk-L2^{-h}\right)^2} \right) + 1.
\]
\end{lemma}
\begin{proof}
The statement is true for unexpanded nodes, so, as before we will consider $(h,k)\in\Tcalit$.
By following the same reasoning as in~\eqref{eqn:NitpohkDecomp}, we have,
\begin{align*}
\Nrecitpohk 
&\leq \lceil\sigmamax^2\tauht\rceil 
    +  \sum_{s=\lceil\sigmamax^2\tauht\rceil}^t\indfone\left(\udis=\normallocis \wedge
            (h,k)\in\Pis \wedge \Nrecishk>
                \sigmamax^2\tauht \right).
\numberthis\label{eqn:NrecitpohkDecomp}
\end{align*}
To bound the summation in the RHS above,
consider any round $s\leq t$.
We will show, by way of contradiction, that the following statements cannot hold
simultaneously.
\[
\Nrecishk > \uishk\sigmamax^2,
\hspace{0.3in}
\Nrecishk> \sigmamax^2\tauht,
\hspace{0.3in}
(h,k)\in\Pis,
\hspace{0.3in}
\normallocis = \udis.
\]
Let $(\ell,\kthrl)$ be the last threshold node on the path from $(0,1)$ to $(h,k)$,
with children $(\ell+1,\kthrlpo-1)$ and $(\ell+1,\kthrlpo)$;
here, the left child $(\ell+1,\kthrlpo-1)$ is an ancestor of $(h,k)$ and $(\ell+1,\kthrlpo)$ is
the threshold node at height $\ell+1$ since $\Ihk\subset[0,\udtruei)$.
Since $\VSishk>\VSrecishk>\tauhs$, we also have $\VSiitthhkk{i}{s}{h''}{k''}>\tauht$ for all of
$(h,k)$'s ancestors $(h'',k'')$, including, in particular, $(l,\kthrl)$ and its ancestors.
Moreover, since $\udis$ was the chosen recommendation, the while loop in \ucgetudrecs
will have reached $(\ell,\kthrl)$ and have proceeded one more step to choose
$(\ell+1,\kthrlpo-1)$ since the while condition is satisfied.
%
By Lemma~\ref{lem:fgB}, and the fact that \ucgetudrecs (line~\ref{lin:nonparucgetudrec}, \algnonpar)
chooses the child with the larger $\Bit$ value at each node,
we have that the $\Bit$ values are non-decreasing along a chosen path.
This leads us to the following conclusion:
\[
\Bit(\ell+1,\kthrlpo) \leq
\Bit(\ell+1,\kthrlpo-1) \leq
\Bit(h,k) < 0.
\]
Here, the last inequality uses the definition of the event
$\Ecalit$~\eqref{eqn:Ecalit}, and that $\Nithk\geq \uithk\sigmamax^2\implies\VSithk\geq \uithk$.
However, by Lemma~\ref{lem:Bitthreshnodes}, $\Bit(\ell+1,\kthrlpo)>0$, which is a contradiction.

Finally, using~\eqref{eqn:NrecitpohkDecomp}, we obtain the following bound.
\begingroup
\allowdisplaybreaks
\begin{align*}
\Nrecitpohk \;
&\leq \lceil\sigmamax^2\tauht\rceil \,+ \\
&\hspace{0.1in}
    \sum_{s=\lceil\sigmamax^2\tauht\rceil}^t\hspace{-0.17in}\indfone\hspace{-0.03in}
            \left(\normallocis=\udis,
                (h,k)\in\Pis, \Nrecishk>
                \sigmamax^2\tauht, \Nrecishk \leq \sigmamax^2\uithk
            \right) \,+
\\
&\hspace{0.2in}
    \sum_{s=\lceil\sigmamax^2\tauht\rceil}^t\hspace{-0.17in}\indfone\hspace{-0.03in}
            \left(\normallocis=\udis,
                (h,k)\in\Pis, \Nrecishk>
                \sigmamax^2\tauht, \Nrecishk > \sigmamax^2\uithk
            \right)
\\
&\leq \lceil\sigmamax^2\tauht\rceil \,+ \\
&\hspace{0.1in}
    \sum_{s=\lceil\sigmamax^2\tauht\rceil}^t\hspace{-0.17in}\indfone\hspace{-0.03in}
            \left(\normallocit=\udit,
                (h,k)\in\Pis, \Nrecishk>
                \sigmamax^2\tauht, \Nrecishk \leq \sigmamax^2\uithk
            \right)
\end{align*}
\endgroup
The proof is completed by the same line of reasoning as at the end of the proof of
Lemma~\ref{lem:Nbadgtr} to show
$\Nrecitpohk\leq \max( \lceil\sigmamax^2\tauht\rceil, \lceil\sigmamax^2\uithk\rceil )$.
\end{proof}

We are now ready to prove the Theorem.
Some of the calculations used in this proof will be similar to those used in the proof of
Lemma~\ref{lem:udbrubiq}, where we bound $\LiQ$.
However, for the sake of clarity, and to keep this proof self-contained, we will repeat those
calculations.

\textbf{\textit{Proof of Theorem~\ref{thm:nonparnsp}.}}
\textbf{Efficiency:}
Recall the bound on the loss from Lemma~\ref{lem:nsploss}.
By observing $(\allocit - \demtrueit)^+ = \volit(\normallocit - \udtruei)^+ \leq
\volmax(\normallocit - \udtruei)^+$,
and similarly, $(\demtrueit - \allocit)^+ \leq
\volmax(\udtruei - \normallocit)^+$,
we first relax this bound as follows.
\begin{align*}
\LOTT 
&\leq 1 + \volmax\sum_{i=1}^n \sum_{t=2}^T (\normallocit-\udtruei)^+ \,+\,
    \volmax\sum_{i=1}^n \,\sum_{t= 2, \,\normallocit=\udit}^T \hspace{-0.1in}(\udtruei-\normallocit)^+
\numberthis\label{eqn:LTnonpardecompzero}
\end{align*}
We can now bound the individual summations above as shown below.
Recall the definitions of the quantities $\hG, \kidown,\kiup, \lidown,\ridown, \liup, \riup$
from Appendix~\ref{sec:nonparintermediatedefns}.
We have:
\begingroup
\allowdisplaybreaks
\begin{align*}
\numberthis\label{eqn:LOTnonpardecompone}
\sum_{t=2}^T (\normallocit-\udtruei)^+
&\leq
\udmax\sum_{t=2}^T \indfone(\normallocit\in[\riup, \udmax])
\,+\,  \udmax \sum_{t=2}^T \indfone(\normallocit\in\IhGkiup)
\\
&\hspace{1.0in}
\,+\, \sum_{t=2}^T \indfone(\normallocit\in(\udtruei, \liup))\cdot(\normallocit-\udtruei).
\\
\numberthis\label{eqn:LOTnonpardecomptwo}
\sum_{t=2\,\normallocit=\udit}^T (\udtruei-\normallocit)^+
&\leq
  \udmax \sum_{t=2}^T \indfone(\normallocit\in[0, \lidown))
\,+\,  \udmax \sum_{t=2}^T \indfone(\normallocit\in\IhGkidown)
\\
&\hspace{1.0in}
+ \sum_{t=2}^T \indfone(\udit=\normallocit\wedge\normallocit\in(\ridown, \udtruei))%
            \cdot(\udtruei-\normallocit).
\end{align*}
\endgroup
Here,~\eqref{eqn:LOTnonpardecompone} follows from the fact that $(\normallocit-\udtruei)^+$ is
positive only when $\normallocit\in(\udtruei, \udmax] = (\udtruei, \liup) \cup [\liup, \riup) \cup
[\riup, \udmax]$.
Moreover, we have bound $(\normallocit-\udtruei) \leq \udmax$ when $\normallocit\in
(\udtruei, \liup) \cup [\liup, \riup)$.
We obtain~\eqref{eqn:LOTnonpardecomptwo} via an analogous reasoning on the interval $[0, \udtruei)$.
Combining this with~\eqref{eqn:LTnonpardecompzero}, we have the following bound.
\begingroup
\allowdisplaybreaks
\begin{align*}
&\LOTT \leq 1 + \volmax\udmax\sum_{i=1}^n
\left(\LioneT + \LitwoT + \LithreeT + \LifourT + \frac{1}{G}\LifiveT +
    \frac{1}{G}\LisixT \right),
\hspace{0.2in}
\text{where,}
\numberthis\label{eqn:LOTnonpardecomp}
\\
&\LioneT = \sum_{t=1}^T \indfone(\Hit < \hG),
\hspace{0.6in}
\LitwoT = \sum_{t=2}^T \indfone(\normallocit\in[0, \lidown) \cup [\riup, \udmax]
    \,\wedge\,\Hit\geq\hG),
\\
&\LithreeT = \sum_{t=2}^T \indfone(\normallocit\in\IhGkidown \,\wedge\,\Hit\geq\hG),
\hspace{0.6in}
\LifourT = \sum_{t=2}^T \indfone(\normallocit\in\IhGkiup \,\wedge\,\Hit\geq\hG),
\\
&\LifiveT = \sum_{t=2}^T \indfone(\normallocit\in(\udtruei, \liup)\,\wedge\,\Hit\geq\hG)
    \cdot(\payoffi(\normallocit)-\threshi),
\\
&\LisixT = 
 \sum_{t=2}^T \indfone(\udit=\normallocit\wedge\normallocit\in(\ridown, \udtruei)\,\wedge\,\Hit\geq\hG)%
   \cdot(\threshi-\payoffi(\threshi)).
\end{align*}
\endgroup
In~\eqref{eqn:LOTnonpardecomp}, $\LioneT$ bounds the number of rounds in which $\Hit<\hG$; therefore,
when bounding
each of the terms in~\eqref{eqn:LOTnonpardecompone} and~\eqref{eqn:LOTnonpardecomptwo}, we can
focus on the rounds where $\Hit\geq \hG$.
$\LitwoT$ accounts for the first term in the RHS of~\eqref{eqn:LOTnonpardecompone} and the first
term of~\eqref{eqn:LOTnonpardecomptwo}.
$\LithreeT$ accounts for the second term in the RHS of~\eqref{eqn:LOTnonpardecompone},
while $\LifourT$ accounts for the second term of~\eqref{eqn:LOTnonpardecomptwo}.
Finally, $\LifiveT$ accounts for the third term in the RHS of~\eqref{eqn:LOTnonpardecompone},
while $\LisixT$ accounts for the third term of~\eqref{eqn:LOTnonpardecomptwo}.
Here, we have used the fact that when $a \in (\ridown,\udtruei)\cup (\udtruei, \liup)
\subset (\udtruei-\epsG, \udtruei+\epsG)$,
we have, by Assumption~\ref{asm:ntg},
$|a - \udtruei| \leq (G/\udmax)|\payoffi(a) - \threshi|$.
We will now bound each of the above terms individually.


\Litboundheader{Bounding $\LioneT$:}
We will bound this by summing up the $\Npithk$ values for all nodes up to height $\hG-1$.
When assigning points to nodes in \ucrecordfb, recall that we always proceed to the child node
if $\VSithk>\tauht$, in which case it is not counted in $\Npithk$.
Therefore, $\Npithk\leq 1 + \sigmamax^2\tauht$.
This leads us to the following bound,
\begin{align*}
\LioneT
&= \sum_{h=0}^{\hG-1} \sum_{k=1}^{2^h} \Npiitthhkk{i,}{T+1}{h}{k}
\leq \sum_{h=0}^{\hG-1} 2^h \left( 1+\sigmamax^2\frac{\betaTdagger^2}{L^2} 4^h \right)
\leq 2^{\hG} + \frac{\sigmamax^2\betaTdagger^2}{L^2}\sum_{h=0}^{\hG-1} 8^h
\\
&\leq 2^{\hG} + \frac{\sigmamax^2\betaTdagger^2}{7\Lf^2} 8^{\hG}
\leq \frac{8\Lf\udmax}{G\epsG} + \frac{512}{7}\frac{L\sigmamax^2\udmax^3\betaTdagger^2}{G^3\epsG^3}.
\numberthis\label{eqn:LioneTbound}
\end{align*}
Here, the fourth step uses the fact that $\sum_{h=0}^m 8^h = (8^{m+1}-1)/7$, and the last step
uses~\eqref{eqn:hG}.

\Litboundheader{Bounding $\LitwoT$:}
First observe that we can write
$\LitwoT= \sum_{k=1}^{2^\hG}\indfone(k<\kidown \,\vee\,k>\kiup) \NitpohGk$;
this follows from our definition of $\hG, \kiup, \kidown$, and the fact that $\Nithk$ counts all
evaluations at node $h$ and its children.
Additionally, we have the following relations
by the NTG condition and the definition of $\hG$ above,
\begin{align*}
\forall\,  k>\kiup, \quad
\Deltaiihhkk{i}{\hG}{k} &= \payoffi(\lhk) - \threshi
= \payoffi(\lhk) - \payoffi(\udtruei+\epsG) + \payoffi(\udtruei+\epsG) - \payoffi(\udtruei)
\\
&\geq \payoffi(\udtruei+\epsG) - \payoffi(\udtruei)
\geq \frac{G\epsG}{\udmax},
\\
\forall\,  k<\kidown, \quad
\Deltaiihhkk{i}{\hG}{k} &= \threshi - \payoffi(\rhk)
\geq \payoffi(\udtruei) - \payoffi(\udtruei - \epsG)
\geq \frac{G\epsG}{\udmax}.
\end{align*}
Moreover,  from~\eqref{eqn:hG} we have
$\frac{L}{2^\hG} < \frac{G\epsG}{4\udmax}$, and 
$\frac{2^\hG}{L} < \frac{8\udmax}{G\epsG}$.
Applying these conclusions to Lemmas~\ref{lem:Nbadgtr} and~\ref{lem:Nbadles}, we have
for all $k<\kidown$ or $k>\kiup$,
\begin{align*}
\Niitthhkk{i,}{T+1}{\hG}{k} &\leq\,
\sigmamax^2 \max\left( \frac{\betaTtilde^2}{L^2}4^{\hG},
        \frac{4\betaTtilde^2}{\left(\Deltaihk-L2^{-h}\right)^2} \right) + 1
\\
&\leq
\sigmamax^2 \betaTtilde^2 \max\left( \frac{64\udmax^2}{G^2\epsG^2}, \frac{64\udmax^2}{9G^2\epsG^2}
        \right) + 1
\leq
\frac{64\udmax^2\sigmamax^2\betaTtilde^2}{G^2\epsG^2} + 1.
\end{align*}
Finally, by applying~\eqref{eqn:hG} once again, we have,
\begin{align*}
\LitwoT &= \sum_{k=1}^{2^\hG}\indfone(k<\kidown \,\vee\,k>\kiup) \Niitthhkk{i,}{T+1}{\hG}{k}
\leq 2^\hG  \left(\frac{64\udmax^2\sigmamax^2\betaTtilde^2}{9G^2\epsG^2} + 1 \right)
\\
&\leq \frac{512\Lf\udmax^3\sigmamax^2}{G^3\epsG^3}\betaTtilde^2 + \frac{8\Lf\udmax}{G\epsG}.
\numberthis\label{eqn:LitwoTbound}
\end{align*}

\Litboundheader{Bounding $\LithreeT$:}
Observe that we can write
$\LithreeT= \Niitthhkk{i,}{T+1}{\hG}{\kidown}$.
By the NTG condition and the definition of $\hG$~\eqref{eqn:hG} we have,
\begin{align*}
\Deltaiihhkk{i}{\hG}{\kidown} = \threshi - \payoffi(\ridown) \geq \frac{G}{\udmax}(\udtruei-\ridown) >
\frac{G\epsG}{2\udmax}.
\end{align*}
Additionally, by~\eqref{eqn:hG}, we have
$\frac{L}{2^\hG} < \frac{G\epsG}{4\udmax}$.
Applying Lemma~\ref{lem:Nbadles}, we get,
\begin{align*}
\LithreeT= \Niitthhkk{i,}{T+1}{\hG}{\kidown}
\leq \frac{64\udmax^2\sigmamax^2}{G^2\epsG^2}\betaTtilde^2 + 1.
\numberthis\label{eqn:LithreeTbound}
\end{align*}

\Litboundheader{Bounding $\LifourT$:}
Following a similar argument as $\LithreeT$ and via an application of
Lemma~\ref{lem:Nbadgtr}, we have
\begin{align*}
\LifourT= \Niitthhkk{i}{T+1}{\hG}{\kiup}
\leq \frac{64\udmax^2\sigmamax^2}{G^2\epsG^2}\betaTtilde^2 + 1.
\numberthis\label{eqn:LifourTbound}
\end{align*}

\Litboundheader{Bounding $\LifiveT$:}
Recall the definitions of $\Icalh, \Jcalh$ from Appendix~\ref{sec:nonparintermediatedefns}.
Let $H\geq\hG$ be a positive integer whose value will be determined shortly.
We will define three subsets of nodes $\Ncalone, \Ncaltwo,\Ncalthree$ in our infinite tree.
Let $\Ncalone$ denote the descendants of $\IcalH$;
let $\Ncaltwo = \bigcup_{h=\hG}^{H-1} \Icalh$;
and let $\Ncalthree$ denote the descendants of
$\bigcup_{h=\hG}^{H} \Jcalh$.
We can now see that $\LifiveT$ can be bound as follows.
\begin{align*}
\LifiveT \leq \Lcalone + \Lcaltwo + \Lcalthree,
\hspace{0.3in}
\Lcal_i =  \sum_{t=2}^T \indfone((\Hit,\Kit)\in\Ncal_i
        \,\wedge\,\Hit\geq\hG)
    \cdot(\payoffi(\normallocit)-\threshi).
\end{align*}
Recall from~\eqref{eqn:Nithk}, that $(\Hit,\Kit)$ are the last nodes in the path $\Pit$ chosen
by \ucrecordfb.
The above bound follows from the fact that for any allocation satisfying,
$\allocit\in(\udtruei, \liup) \wedge \Hit\geq \hG$,
the last node $(\Hit,\Kit)$ should be in $\Ncalone\cup\Ncaltwo\cup\Ncalthree$.


By~\eqref{eqn:Icalhpayoffbound}, we have $\payoffi(a)-\threshi<4L2^{-H}$
when $a\in \Ihk$, for any $(h,k)\in\Ncalone$.
This leads to the following straightforward bound for $\Lcalone$,
\begin{align*}
\numberthis\label{eqn:Lcalonebound}
\Lcalone =  \sum_{t=2}^T \indfone((\Hit,\Kit)\in\Ncalone
        \,\wedge\,\Hit\geq\hG)
    \cdot(\payoffi(\normallocit)-\threshi)
\leq \frac{4L}{2^{H}} T.
\end{align*}
Next, we bound $\Lcaltwo$ as shown below.
\begin{align*}
\Lcaltwo
&= \sum_{h=\hG}^{H-1}\sum_{(h,k)\in\Icalh} \frac{4\Lf}{2^h}\NpiTpohk
\leq  \sum_{h=\hG}^{H-1} \frac{4\Lf}{2^h} \frac{4\Lf}{G}
\left(1 + \frac{ 4^h\sigmamax^2 \betaTtilde^2}{\Lf^2}\right)
\\
&\leq \frac{16\Lf^2}{G} \sum_{h=\hG}^{H-1} 2^{-h}
+  \frac{16\sigmamax^2 \betaTtilde^2}{G} \sum_{h=\hG}^{H-1} 2^{h}
\leq \frac{32\Lf^2}{G} 
+  \frac{16\sigmamax^2 \betaTtilde^2}{G} 2^H.
\numberthis\label{eqn:Lcaltwobound}
\end{align*}
In the first step we have used $\Npithk$,
since $\Lcaltwo$ only counts allocations where $(\Hit,\Kit)$ were in
$\bigcup_{h=\hG}^{H-1} \Icalh$;
additionally, we have used~\eqref{eqn:Icalhpayoffbound} to bound $\payoffi(\allocit) - \threshi$.
In the second step, by the same reasoning used in the bound for $\LioneT$, we have
$\NpiTpohk\leq 1 + \sigmamax^2\tauhT$ for any node $(h,k)$;
moreover, we have used~\eqref{eqn:IcalhJcalh} to bound the number of nodes in $\Icalh$.
The remaining steps are obtained by algebraic manipulations.
Finally, we bound $\Lcalthree$ as follows.
\begin{align*}
\Lcalthree
&= \sum_{h=\hG}^{H}\sum_{(h,k)\in\Jcalh} \frac{8\Lf}{2^{h}}\NiTpohk
\leq  \sum_{h=\hG}^{H} \frac{8\Lf}{2^h} \frac{8\Lf}{G}
\left(1 + \frac{ 4\betaTtilde^2\sigmamax^2}{\Lf^2}4^h\right)
\\
&\leq \frac{64\Lf^2}{G} \sum_{h=\hG}^{H} 2^{-h}
+  \frac{256\sigmamax^2 \betaTtilde^2}{G} \sum_{h=\hG}^{H} 2^{h}
\leq \frac{128\Lf^2}{G} 
+  \frac{512\sigmamax^2 \betaTtilde^2}{G} 2^H.
\numberthis\label{eqn:Lcalthreebound}
\end{align*}
Above, in the first step we have used $\Nithk$ (instead of $\Npithk$),
since $\Lcalthree$ counts allocations where $(\Hit,\Kit)$ belonged to the descendants of
$\bigcup_{h=\hG}^{H-1} \Jcalh$;
additionally, we have used~\eqref{eqn:Icalhpayoffbound} and the fact that
parents of nodes in $\Jcalh$ are in $\Icalhmo$ to bound $\payoffi(\allocit) - \threshi$.
In the second step, first we have used~\eqref{eqn:IcalhJcalh} to bound the number of nodes in
$\Jcalh$;
to bound the number of evaluations in each such node, we have applied Lemma~\ref{lem:Nbadgtr}
along with the fact that $\Deltaihk>2L2^{-h}$ for nodes in $\Jcalh$ by their definition;
therefore,
\[
\NiTpohk \leq 1 + \sigmamax^2 \max\left(\frac{\betaTtilde^2}{L^2}4^h,
\frac{4\betaTtilde^2}{\left(\Deltaihk-L2^{-h}\right)^2} \right)
\leq 1 + \frac{4\sigmamax^2 \betaTtilde^2}{\Lf^2} 4^h.
\]
The remaining steps in~\eqref{eqn:Lcalthreebound} are obtained via algebraic manipulations.
Combining~\eqref{eqn:Lcalonebound},~\eqref{eqn:Lcaltwobound}, and~\eqref{eqn:Lcalthreebound}
results in the following bound for $\LifiveT$.
\begin{align*}
\LifiveT \leq \frac{4L}{2^{H}} T
+  \frac{528\sigmamax^2 \betaTtilde^2}{G} 2^H
+ \frac{160\Lf^2}{G}
\leq C'\frac{\Lf^{\nicefrac{1}{2}}\sigmamax \betaTtilde T^{\nicefrac{1}{2}}}{G^{\nicefrac{1}{2}}}
+ \frac{160\Lf^2}{G}
\numberthis\label{eqn:LifiveTbound}
\end{align*}
Here $C'$ is a global constant.
The last step is obtained by choosing $H$ such that $2^H \asymp
\frac{\sqrt{GLT}}{\sigmamax\betaTtilde}$.

\Litboundheader{Bounding $\LisixT$:}
Our method for obtaining this bound follows along similar lines to the bound of $\LifiveT$,
but using the $\Icalph,\Jcalph$ sets as defined in Appendix~\ref{sec:nonparintermediatedefns}.
Therefore, we will outline the argument highlighting only the important differences.

As before, let $H\geq\hG$ be a positive integer whose value will be determined shortly.
Next, let $\Ncalone$ denote the descendants of $\IcalpH$;
let $\Ncaltwo = \bigcup_{h=\hG}^{H-1} \Icalph$;
and let $\Ncalthree$ denote the descendants of
$\bigcup_{h=\hG}^{H} \Jcalph$.
It follows that $\LisixT$ can be bound as follows.
\begin{align*}
\LisixT \leq \Lcalone + \Lcaltwo + \Lcalthree,
\hspace{0.20in}
\Lcal_i =  \sum_{t=2}^T \indfone(\udit=\normallocit \,\wedge\, (\Hit,\Kit)\in\Ncal_i
        \,\wedge\,\Hit\geq\hG)
    \cdot(\payoffi(\normallocit)-\threshi).
\end{align*}
We will now bound $\Lcalone, \Lcaltwo, \Lcalthree$.
As in~\eqref{eqn:Lcalonebound}, we can bound 
$\Lcalone
\leq \frac{4L}{2^{H}} T$.
Denoting
$\Nrecpithk = \sum_{s=1}^{t-1} \indfone(\udis=\normallocis, (H_{it}, K_{it}) = (h,k))$,
we bound $\Lcaltwo$ as shown below.
\begingroup
\allowdisplaybreaks
\begin{align*}
\Lcaltwo
&= \sum_{h=\hG}^{H-1}\sum_{(h,k)\in\Icalh} \frac{4\Lf}{2^h}\NrecpiTpohk
\leq \sum_{h=\hG}^{H-1}\sum_{(h,k)\in\Icalh} \frac{4\Lf}{2^h}\NreciTpohk
\\
&\leq \frac{32\Lf^2}{G} 
+  \frac{16\sigmamax^2 \betatwoT^2}{G} 2^H.
\end{align*}
\endgroup
The first step simply uses $\Nrecpithk\leq\Nrecithk$, while the last step follows from the same
calculations as in~\eqref{eqn:Lcaltwobound}.
Finally, we have the following bound for $\Lcalthree$.
\begin{align*}
\Lcalthree
&= \sum_{h=\hG}^{H}\sum_{(h,k)\in\Icalh} \frac{8\Lf}{2^{h-1}}\NreciTpohk
\leq  \sum_{h=\hG}^{H} \frac{8\Lf}{2^h} \frac{8\Lf}{G}
\left(1 + \frac{ 4\betaTtilde^2\sigmamax^2}{\Lf^2}4^h\right)
\\
&\leq \frac{128\Lf^2}{G} 
+  \frac{512\sigmamax^2 \betaTtilde^2}{G} 2^H.
\end{align*}
Above, in the first step we have used the bounds on $\threshi-\payoffi$ for points
in $\Jcalph$ given at the end of Appendix~\ref{sec:nonparintermediatedefns}.
To bound $\Nrecithk$, we have applied Lemma~\ref{lem:Nbadles}
along with the fact that $\Deltaihk>2L2^{-h}$ for nodes in $\Jcalph$ by their definition;
therefore,
\[
\Nrecithk \leq 1 + \sigmamax^2 \max\left(\frac{\betat^2}{L^2}4^h,
\frac{4\betat^2}{\left(\Deltaihk-L2^{-h}\right)^2} \right)
\leq 1 + \frac{4\sigmamax^2 \betat^2}{\Lf^2} 4^h.
\]
The remainder of the calculations for bounding $\Lcalthree$, are similar
to~\eqref{eqn:Lcalthreebound}.
The expressions for the bounds for $\Lcalone,\Lcaltwo,\Lcalthree$ are identical to the bounds
for $\LifiveT$, and hence we obtain the following bound for $\LisixT$,
which is the same as the RHS of~\eqref{eqn:LifiveTbound}. 
Here $C'$ is a global constant.
\begin{align*}
\LisixT
\leq C'\frac{\Lf^{\nicefrac{1}{2}}\sigmamax \betaT T^{\nicefrac{1}{2}}}{G^{\nicefrac{1}{2}}}
+ \frac{160\Lf^2}{G}
\numberthis\label{eqn:LisixTbound}
\end{align*}

The bound on the loss is obtained by combining~\eqref{eqn:LOTnonpardecomp},%
~\eqref{eqn:LioneTbound},%
~\eqref{eqn:LitwoTbound},%
~\eqref{eqn:LithreeTbound},%
~\eqref{eqn:LifourTbound},%
~\eqref{eqn:LifiveTbound},
and~\eqref{eqn:LisixTbound}
and observing that
$\betaTtilde\leq \betatwoT\asymp \sqrt{\log(nT/\delta)}$~\eqref{eqn:trelationttilde}.

\textbf{Fairness:}
We will bound $\UeiiT-\UiT$ for our asymptotic fairness result in the following manner.
\begin{align*}
\UeiiT - \UiT
&\leq \sum_{t=1}^T
    \indfone(\udit=\normallocit\,\wedge\,\normallocit<\udtruei) \cdot \left(\utili(\udtruei) - \utili(\udit)\right)
\\
&\leq
\sum_{t=2}^T \indfone(\normallocit\in[0, \lidown) \,\wedge\,\normallocit=\udit) \Lipi\udmax
\,+\, \sum_{t=2}^T \indfone(\normallocit\in\IhGkidown\,\wedge\,\normallocit=\udit) \Lipi\udmax
\\
&\hspace{1.0in}
+ \sum_{t=2}^T \indfone(\udit=\normallocit\wedge\normallocit\in[\ridown, \udtruei))%
            \cdot\Lipi(\udtruei-\normallocit).
\\
&\leq \Lipi\udmax\left(\LioneT + \LitwoT + \LifourT\right) + \frac{\Lipi\udmax}{G}\LisixT.
\end{align*}
Here, the first step uses the bound for $\UeiiT-\UiT$ in Lemma~\ref{lem:nspfairness}.
The second step
decomposes the allocations in $[0, \udtruei)$ into the intervals
$[0,\lidown)$, $[\lidown, \ridown)$, and $[\ridown, \udtruei)$ and used $\Lipi$-Lipschitzness
of the utility functions.
Moreover, for the first two summations, we have used the fact that
$\utili(\udtruei)-\utili(\udit) \leq \Lipi(\udtruei-\udit)^+ \leq \Lipi\udmax$.
The last step is obtained by comparing the expressions for
$\LioneT, \LitwoT, \LifourT, \LisixT$  in~\eqref{eqn:LOTnonpardecomptwo}
and~\eqref{eqn:LOTnonpardecomp}.
The claim follows from the bounds
in%
~\eqref{eqn:LioneTbound},%
~\eqref{eqn:LitwoTbound},%
~\eqref{eqn:LifourTbound},
and~\eqref{eqn:LisixTbound}.
\qedwhite

{
\bibliography{kk,bib_mmf}

\begin{thebibliography}{69}
\providecommand{\natexlab}[1]{#1}
\providecommand{\url}[1]{\texttt{#1}}
\expandafter\ifx\csname urlstyle\endcsname\relax
  \providecommand{\doi}[1]{doi: #1}\else
  \providecommand{\doi}{doi: \begingroup \urlstyle{rm}\Url}\fi

\bibitem[had()]{hadoopfs}
{Hadoop Fair Scheduler}.
\newblock URL \url{https://hadoop.apache.org/docs/current/hadoop-yarn/
  hadoop-yarn-site/FairScheduler.html}.

\bibitem[lbl()]{lbl2020}
{Lawrence Berkeley National Laboratory}.
\newblock URL \url{https://www.lbl.gov/}.

\bibitem[lin()]{linuxfq2020}
{Linux Manual: Fair Queue Traffic Policing}.
\newblock URL \url{https://www.man7.org/linux/ man-pages/man8/tc-fq.8.html}.

\bibitem[psc()]{psc2020}
{Pittsburgh Supercomputing Center}.
\newblock URL \url{https://www.psc.edu/}.

\bibitem[twi()]{twitter2020}
{Twitter Streaming API}.
\newblock URL \url{https://developer.twitter.com/en/docs/
  tutorials/consuming-streaming-data}.

\bibitem[wik()]{wikipedia2020}
{WikiShark}.
\newblock URL \url{https://www.wikishark.com/}.

\bibitem[Aleksandrov and Walsh(2017)]{aleksandrov2017pure}
Martin Aleksandrov and Toby Walsh.
\newblock Pure nash equilibria in online fair division.
\newblock In \emph{IJCAI}, pages 42--48, 2017.

\bibitem[Amin et~al.(2013)Amin, Rostamizadeh, and Syed]{amin2013learning}
Kareem Amin, Afshin Rostamizadeh, and Umar Syed.
\newblock Learning prices for repeated auctions with strategic buyers.
\newblock In \emph{Advances in Neural Information Processing Systems}, pages
  1169--1177, 2013.

\bibitem[Athey and Segal(2013)]{athey2013efficient}
Susan Athey and Ilya Segal.
\newblock An efficient dynamic mechanism.
\newblock \emph{Econometrica}, 81\penalty0 (6):\penalty0 2463--2485, 2013.

\bibitem[Auer(2003)]{auer03ucb}
Peter Auer.
\newblock {Using Confidence Bounds for Exploitation-exploration Trade-offs}.
\newblock \emph{J. Mach. Learn. Res.}, 2003.

\bibitem[Babaioff et~al.(2013)Babaioff, Kleinberg, and
  Slivkins]{babaioff2013multi}
Moshe Babaioff, Robert Kleinberg, and Aleksandrs Slivkins.
\newblock Multi-parameter mechanisms with implicit payment computation.
\newblock In \emph{Proceedings of the Fourteenth ACM Conference on Electronic
  Commerce}, pages 35--52, 2013.

\bibitem[Babaioff et~al.(2014)Babaioff, Sharma, and
  Slivkins]{babaioff2014characterizing}
Moshe Babaioff, Yogeshwer Sharma, and Aleksandrs Slivkins.
\newblock Characterizing truthful multi-armed bandit mechanisms.
\newblock \emph{SIAM Journal on Computing}, 43\penalty0 (1):\penalty0 194--230,
  2014.

\bibitem[Boutin et~al.(2014)Boutin, Ekanayake, Lin, Shi, Zhou, Qian, Wu, and
  Zhou]{boutin2014apollo}
Eric Boutin, Jaliya Ekanayake, Wei Lin, Bing Shi, Jingren Zhou, Zhengping Qian,
  Ming Wu, and Lidong Zhou.
\newblock Apollo: Scalable and coordinated scheduling for cloud-scale
  computing.
\newblock In \emph{11th $\{$USENIX$\}$ Symposium on Operating Systems Design
  and Implementation ($\{$OSDI$\}$ 14)}, pages 285--300, 2014.

\bibitem[Bubeck et~al.(2010)Bubeck, Munos, Stoltz, and Szepesvari]{bubeck2010x}
S{\'e}bastien Bubeck, R{\'e}mi Munos, Gilles Stoltz, and Csaba Szepesvari.
\newblock {X-armed Bandits}.
\newblock \emph{arXiv preprint arXiv:1001.4475}, 2010.

\bibitem[Chaudhuri et~al.(2015)Chaudhuri, Kakade, Netrapalli, and
  Sanghavi]{chaudhuri2015convergence}
Kamalika Chaudhuri, Sham~M Kakade, Praneeth Netrapalli, and Sujay Sanghavi.
\newblock Convergence rates of active learning for maximum likelihood
  estimation.
\newblock In \emph{Advances in Neural Information Processing Systems}, pages
  1090--1098, 2015.

\bibitem[Chen et~al.(1999)Chen, Hu, Ying, et~al.]{chen1999strong}
Kani Chen, Inchi Hu, Zhiliang Ying, et~al.
\newblock Strong consistency of maximum quasi-likelihood estimators in
  generalized linear models with fixed and adaptive designs.
\newblock \emph{The Annals of Statistics}, 27\penalty0 (4):\penalty0
  1155--1163, 1999.

\bibitem[Chen et~al.(2018)Chen, Liu, Li, and Li]{chen2018scheduling}
Li~Chen, Shuhao Liu, Baochun Li, and Bo~Li.
\newblock Scheduling jobs across geo-distributed datacenters with max-min
  fairness.
\newblock \emph{IEEE Transactions on Network Science and Engineering},
  6\penalty0 (3):\penalty0 488--500, 2018.

\bibitem[Cole et~al.(2013)Cole, Gkatzelis, and Goel]{cole2013mechanism}
Richard Cole, Vasilis Gkatzelis, and Gagan Goel.
\newblock Mechanism design for fair division: allocating divisible items
  without payments.
\newblock In \emph{Proceedings of the fourteenth ACM conference on Electronic
  commerce}, pages 251--268, 2013.

\bibitem[Crankshaw et~al.(2018)Crankshaw, Sela, Zumar, Mo, Gonzalez, Stoica,
  and Tumanov]{crankshaw2018inferline}
Daniel Crankshaw, Gur-Eyal Sela, Corey Zumar, Xiangxi Mo, Joseph~E Gonzalez,
  Ion Stoica, and Alexey Tumanov.
\newblock Inferline: Ml inference pipeline composition framework.
\newblock \emph{arXiv preprint arXiv:1812.01776}, 2018.

\bibitem[Dani et~al.(2008)Dani, Hayes, and Kakade]{dani2008stochastic}
Varsha Dani, Thomas~P Hayes, and Sham~M Kakade.
\newblock Stochastic linear optimization under bandit feedback.
\newblock 2008.

\bibitem[Daskalakis et~al.(2006)Daskalakis, Mehta, and
  Papadimitriou]{daskalakis2006note}
Constantinos Daskalakis, Aranyak Mehta, and Christos Papadimitriou.
\newblock A note on approximate nash equilibria.
\newblock In \emph{International Workshop on Internet and Network Economics},
  pages 297--306. Springer, 2006.

\bibitem[de~la Pena et~al.(2004)de~la Pena, Klass, and Lai]{de2004self}
Victor~H de~la Pena, Michael~J Klass, and Tze~Leung Lai.
\newblock Self-normalized processes: exponential inequalities, moment bounds
  and iterated logarithm laws.
\newblock \emph{Annals of probability}, pages 1902--1933, 2004.

\bibitem[Delimitrou and Kozyrakis(2013)]{delimitrou2013paragon}
Christina Delimitrou and Christos Kozyrakis.
\newblock Paragon: Qos-aware scheduling for heterogeneous datacenters.
\newblock \emph{ACM SIGPLAN Notices}, 48\penalty0 (4):\penalty0 77--88, 2013.

\bibitem[Delimitrou and Kozyrakis(2014)]{delimitrou2014quasar}
Christina Delimitrou and Christos Kozyrakis.
\newblock Quasar: resource-efficient and qos-aware cluster management.
\newblock \emph{ACM SIGPLAN Notices}, 49\penalty0 (4):\penalty0 127--144, 2014.

\bibitem[Demers et~al.(1989)Demers, Keshav, and Shenker]{demers1989analysis}
Alan Demers, Srinivasan Keshav, and Scott Shenker.
\newblock Analysis and simulation of a fair queueing algorithm.
\newblock \emph{ACM SIGCOMM Computer Communication Review}, 19\penalty0
  (4):\penalty0 1--12, 1989.

\bibitem[Feder et~al.(2007)Feder, Nazerzadeh, and
  Saberi]{feder2007approximating}
Tomas Feder, Hamid Nazerzadeh, and Amin Saberi.
\newblock Approximating nash equilibria using small-support strategies.
\newblock In \emph{Proceedings of the 8th ACM conference on Electronic
  commerce}, pages 352--354, 2007.

\bibitem[Filippi et~al.(2010)Filippi, Cappe, Garivier, and
  Szepesv{\'a}ri]{filippi2010parametric}
Sarah Filippi, Olivier Cappe, Aur{\'e}lien Garivier, and Csaba Szepesv{\'a}ri.
\newblock Parametric bandits: The generalized linear case.
\newblock In \emph{Advances in Neural Information Processing Systems}, pages
  586--594, 2010.

\bibitem[Freeman et~al.(2018)Freeman, Zahedi, Conitzer, and
  Lee]{freeman2018dynamic}
Rupert Freeman, Seyed~Majid Zahedi, Vincent Conitzer, and Benjamin~C Lee.
\newblock Dynamic proportional sharing: A game-theoretic approach.
\newblock \emph{Proceedings of the ACM on Measurement and Analysis of Computing
  Systems}, 2\penalty0 (1):\penalty0 1--36, 2018.

\bibitem[Gheshlaghi~Azar et~al.(2014)Gheshlaghi~Azar, Lazaric, and
  Brunskill]{azar2014online}
Mohammad Gheshlaghi~Azar, Alessandro Lazaric, and Emma Brunskill.
\newblock {Online stochastic optimization under correlated bandit feedback}.
\newblock In \emph{International Conference on Machine Learning}, pages
  1557--1565. PMLR, 2014.

\bibitem[Ghodsi et~al.(2011)Ghodsi, Zaharia, Hindman, Konwinski, Shenker, and
  Stoica]{ghodsi2011dominant}
Ali Ghodsi, Matei Zaharia, Benjamin Hindman, Andy Konwinski, Scott Shenker, and
  Ion Stoica.
\newblock Dominant resource fairness: Fair allocation of multiple resource
  types.
\newblock In \emph{Nsdi}, volume~11, pages 24--24, 2011.

\bibitem[Ghodsi et~al.(2012)Ghodsi, Sekar, Zaharia, and
  Stoica]{ghodsi2012multi}
Ali Ghodsi, Vyas Sekar, Matei Zaharia, and Ion Stoica.
\newblock {Multi-resource Fair Queueing for Packet Processing}.
\newblock In \emph{Proceedings of the ACM SIGCOMM 2012 conference on
  Applications, technologies, architectures, and protocols for computer
  communication}, pages 1--12, 2012.

\bibitem[Ghodsi et~al.(2013)Ghodsi, Zaharia, Shenker, and
  Stoica]{ghodsi2013choosy}
Ali Ghodsi, Matei Zaharia, Scott Shenker, and Ion Stoica.
\newblock {Choosy: Max-min fair sharing for datacenter jobs with constraints}.
\newblock In \emph{Proceedings of the 8th ACM European Conference on Computer
  Systems}, pages 365--378, 2013.

\bibitem[Grill et~al.(2015)Grill, Valko, and Munos]{grill2015black}
Jean-Bastien Grill, Michal Valko, and R{\'e}mi Munos.
\newblock Black-box optimization of noisy functions with unknown smoothness.
\newblock In \emph{Advances in Neural Information Processing Systems}, pages
  667--675, 2015.

\bibitem[Gutman and Nisan(2012)]{gutman2012fair}
Avital Gutman and Noam Nisan.
\newblock Fair allocation without trade.
\newblock \emph{arXiv preprint arXiv:1204.4286}, 2012.

\bibitem[Hahne(1991)]{hahne1991round}
Ellen~L. Hahne.
\newblock Round-robin scheduling for max-min fairness in data networks.
\newblock \emph{IEEE Journal on Selected Areas in communications}, 9\penalty0
  (7):\penalty0 1024--1039, 1991.

\bibitem[Hammond(1990)]{hammond1990interpersonal}
Peter~J Hammond.
\newblock \emph{Interpersonal comparisons of utility: Why and how they are and
  should be made}.
\newblock Number 90-93. European University Institute, 1990.

\bibitem[Hindman et~al.(2011)Hindman, Konwinski, Zaharia, Ghodsi, Joseph, Katz,
  Shenker, and Stoica]{hindman2011mesos}
Benjamin Hindman, Andy Konwinski, Matei Zaharia, Ali Ghodsi, Anthony~D Joseph,
  Randy~H Katz, Scott Shenker, and Ion Stoica.
\newblock {Mesos: A platform for fine-grained resource sharing in the data
  center.}
\newblock In \emph{NSDI}, volume~11, pages 22--22, 2011.

\bibitem[Huang and Bensaou(2001)]{huang2001max}
Xiao~Long Huang and Brahim Bensaou.
\newblock On max-min fairness and scheduling in wireless ad-hoc networks:
  analytical framework and implementation.
\newblock In \emph{Proceedings of the 2nd ACM international symposium on Mobile
  ad hoc networking \& computing}, pages 221--231, 2001.

\bibitem[Jones et~al.(1993)Jones, Perttunen, and Stuckman]{jones93direct}
D.~R. Jones, C.~D. Perttunen, and B.~E. Stuckman.
\newblock {Lipschitzian Optimization Without the Lipschitz Constant}.
\newblock \emph{J. Optim. Theory Appl.}, 1993.

\bibitem[Kandasamy et~al.(2015)Kandasamy, Schenider, and
  P{\'{o}}czos]{kandasamy15addBO}
Kirthevasan Kandasamy, Jeff Schenider, and Barnab{\'{a}}s P{\'{o}}czos.
\newblock {High Dimensional Bayesian Optimisation and Bandits via Additive
  Models}.
\newblock In \emph{International Conference on Machine Learning}, 2015.

\bibitem[Kandasamy et~al.(2020)Kandasamy, Gonzalez, Jordan, and
  Stoica]{kandasamy2020mechanism}
Kirthevasan Kandasamy, Joseph~E Gonzalez, Michael~I Jordan, and Ion Stoica.
\newblock Mechanism design with bandit feedback.
\newblock \emph{arXiv preprint arXiv:2004.08924}, 2020.

\bibitem[Kojima and Manea(2010)]{kojima2010incentives}
Fuhito Kojima and Mihai Manea.
\newblock Incentives in the probabilistic serial mechanism.
\newblock \emph{Journal of Economic Theory}, 145\penalty0 (1):\penalty0
  106--123, 2010.

\bibitem[Li and Xue(2013)]{li2013egalitarian}
Jin Li and Jingyi Xue.
\newblock Egalitarian division under leontief preferences.
\newblock \emph{Economic Theory}, 54\penalty0 (3):\penalty0 597--622, 2013.

\bibitem[Li et~al.(2017)Li, Lu, and Zhou]{li2017provably}
Lihong Li, Yu~Lu, and Dengyong Zhou.
\newblock Provably optimal algorithms for generalized linear contextual
  bandits.
\newblock \emph{arXiv preprint arXiv:1703.00048}, 2017.

\bibitem[Li et~al.(2015)Li, Sheng, Tan, Zhang, Sun, Wang, Shi, and
  Li]{li2015energy}
Yuzhou Li, Min Sheng, Chee~Wei Tan, Yan Zhang, Yuhua Sun, Xijun Wang, Yan Shi,
  and Jiandong Li.
\newblock Energy-efficient subcarrier assignment and power allocation in ofdma
  systems with max-min fairness guarantees.
\newblock \emph{IEEE Transactions on Communications}, 63\penalty0 (9):\penalty0
  3183--3195, 2015.

\bibitem[Lipton et~al.(2003)Lipton, Markakis, and Mehta]{lipton2003playing}
Richard~J Lipton, Evangelos Markakis, and Aranyak Mehta.
\newblock Playing large games using simple strategies.
\newblock In \emph{Proceedings of the 4th ACM conference on Electronic
  commerce}, pages 36--41, 2003.

\bibitem[Liu et~al.(2019)Liu, Mania, and Jordan]{liu2019competing}
Lydia~T Liu, Horia Mania, and Michael~I Jordan.
\newblock Competing bandits in matching markets.
\newblock In \emph{Proceedings of the Twenty-Third Conference on Artificial
  Intelligence and Statistics (AISTATS)}, 2019.

\bibitem[Liu et~al.(2013)Liu, Dai, and Luo]{liu2013max}
Ya-Feng Liu, Yu-Hong Dai, and Zhi-Quan Luo.
\newblock Max-min fairness linear transceiver design for a multi-user mimo
  interference channel.
\newblock \emph{IEEE Transactions on Signal Processing}, 61\penalty0
  (9):\penalty0 2413--2423, 2013.

\bibitem[Mansour et~al.(2015)Mansour, Slivkins, and
  Syrgkanis]{mansour2015bayesian}
Yishay Mansour, Aleksandrs Slivkins, and Vasilis Syrgkanis.
\newblock Bayesian incentive-compatible bandit exploration.
\newblock In \emph{Proceedings of the Sixteenth ACM Conference on Economics and
  Computation}, pages 565--582, 2015.

\bibitem[McGregor et~al.(2000)McGregor, Braun, and Brown]{mcgregor2000nlanr}
Tony McGregor, H-W Braun, and Jeff Brown.
\newblock The nlanr network analysis infrastructure.
\newblock \emph{IEEE Communications Magazine}, 38\penalty0 (5):\penalty0
  122--128, 2000.

\bibitem[Mogul and Wilkes(2019)]{mogul2019nines}
Jeffrey~C Mogul and John Wilkes.
\newblock Nines are not enough: meaningful metrics for clouds.
\newblock In \emph{Proceedings of the Workshop on Hot Topics in Operating
  Systems}, pages 136--141, 2019.

\bibitem[Nace and Pi{\'o}ro(2008)]{nace2008max}
Dritan Nace and Michal Pi{\'o}ro.
\newblock Max-min fairness and its applications to routing and load-balancing
  in communication networks: a tutorial.
\newblock \emph{IEEE Communications Surveys \& Tutorials}, 10\penalty0
  (4):\penalty0 5--17, 2008.

\bibitem[Nazerzadeh et~al.(2008)Nazerzadeh, Saberi, and
  Vohra]{nazerzadeh2008dynamic}
Hamid Nazerzadeh, Amin Saberi, and Rakesh Vohra.
\newblock Dynamic cost-per-action mechanisms and applications to online
  advertising.
\newblock In \emph{Proceedings of the 17th International Conference on World
  Wide Web}, pages 179--188, 2008.

\bibitem[Parkes et~al.(2015)Parkes, Procaccia, and Shah]{parkes2015beyond}
David~C Parkes, Ariel~D Procaccia, and Nisarg Shah.
\newblock Beyond dominant resource fairness: Extensions, limitations, and
  indivisibilities.
\newblock \emph{ACM Transactions on Economics and Computation (TEAC)},
  3\penalty0 (1):\penalty0 1--22, 2015.

\bibitem[Procaccia(2013)]{procaccia2013cake}
Ariel~D Procaccia.
\newblock {Cake cutting: not just child's play}.
\newblock \emph{Communications of the ACM}, 56\penalty0 (7):\penalty0 78--87,
  2013.

\bibitem[Roberts and Postlewaite(1976)]{roberts1976incentives}
Donald~John Roberts and Andrew Postlewaite.
\newblock The incentives for price-taking behavior in large exchange economies.
\newblock \emph{Econometrica: Journal of the Econometric Society}, pages
  115--127, 1976.

\bibitem[Rusmevichientong and Tsitsiklis(2010)]{rusmevichientong2010linearly}
Paat Rusmevichientong and John~N Tsitsiklis.
\newblock Linearly parameterized bandits.
\newblock \emph{Mathematics of Operations Research}, 35\penalty0 (2):\penalty0
  395--411, 2010.

\bibitem[Rzadca et~al.(2020)Rzadca, Findeisen, Swiderski, Zych, Broniek,
  Kusmierek, Nowak, Strack, Witusowski, Hand, et~al.]{rzadca2020autopilot}
Krzysztof Rzadca, Pawel Findeisen, Jacek Swiderski, Przemyslaw Zych, Przemyslaw
  Broniek, Jarek Kusmierek, Pawel Nowak, Beata Strack, Piotr Witusowski, Steven
  Hand, et~al.
\newblock Autopilot: workload autoscaling at google.
\newblock In \emph{Proceedings of the Fifteenth European Conference on Computer
  Systems}, pages 1--16, 2020.

\bibitem[Schummer(2004)]{schummer2004almost}
James Schummer.
\newblock Almost-dominant strategy implementation: exchange economies.
\newblock \emph{Games and Economic Behavior}, 48\penalty0 (1):\penalty0
  154--170, 2004.

\bibitem[Sen et~al.(2018)Sen, Kandasamy, and Shakkottai]{sen2018multi}
Rajat Sen, Kirthevasan Kandasamy, and Sanjay Shakkottai.
\newblock Multi-fidelity black-box optimization with hierarchical partitions.
\newblock In \emph{International conference on machine learning}, pages
  4538--4547, 2018.

\bibitem[Sen et~al.(2019)Sen, Kandasamy, and Shakkottai]{sen2019noisy}
Rajat Sen, Kirthevasan Kandasamy, and Sanjay Shakkottai.
\newblock Noisy blackbox optimization using multi-fidelity queries: A tree
  search approach.
\newblock In \emph{The 22nd international conference on artificial intelligence
  and statistics}, pages 2096--2105, 2019.

\bibitem[Shang et~al.(2018)Shang, Kaufmann, and Valko]{shang2018adaptive}
Xuedong Shang, Emilie Kaufmann, and Michal Valko.
\newblock Adaptive black-box optimization got easier: Hct only needs local
  smoothness.
\newblock 2018.

\bibitem[Shieh et~al.(2011)Shieh, Kandula, Greenberg, Kim, and
  Saha]{shieh2011sharing}
Alan Shieh, Srikanth Kandula, Albert~G Greenberg, Changhoon Kim, and Bikas
  Saha.
\newblock {Sharing the Data Center Network.}
\newblock In \emph{NSDI}, volume~11, pages 23--23, 2011.

\bibitem[Srinivas et~al.(2010)Srinivas, Krause, Kakade, and
  Seeger]{srinivas10gpbandits}
Niranjan Srinivas, Andreas Krause, Sham Kakade, and Matthias Seeger.
\newblock {Gaussian Process Optimization in the Bandit Setting: No Regret and
  Experimental Design}.
\newblock In \emph{International Conference on Machine Learning}, 2010.

\bibitem[Tang et~al.(2014)Tang, Lee, He, and Liu]{tang2014long}
Shanjiang Tang, Bu-sung Lee, Bingsheng He, and Haikun Liu.
\newblock Long-term resource fairness: Towards economic fairness on
  pay-as-you-use computing systems.
\newblock In \emph{Proceedings of the 28th ACM international conference on
  Supercomputing}, pages 251--260, 2014.

\bibitem[Venkataraman et~al.(2016)Venkataraman, Yang, Franklin, Recht, and
  Stoica]{venkataraman2016ernest}
Shivaram Venkataraman, Zongheng Yang, Michael Franklin, Benjamin Recht, and Ion
  Stoica.
\newblock Ernest: Efficient performance prediction for large-scale advanced
  analytics.
\newblock In \emph{13th $\{$USENIX$\}$ Symposium on Networked Systems Design
  and Implementation ($\{$NSDI$\}$ 16)}, pages 363--378, 2016.

\bibitem[Verma et~al.(2015)Verma, Pedrosa, Korupolu, Oppenheimer, Tune, and
  Wilkes]{verma2015large}
Abhishek Verma, Luis Pedrosa, Madhukar Korupolu, David Oppenheimer, Eric Tune,
  and John Wilkes.
\newblock {Large-scale cluster management at Google with Borg}.
\newblock In \emph{Proceedings of the Tenth European Conference on Computer
  Systems}, pages 1--17, 2015.

\bibitem[Weed et~al.(2016)Weed, Perchet, and Rigollet]{weed2016online}
Jonathan Weed, Vianney Perchet, and Philippe Rigollet.
\newblock Online learning in repeated auctions.
\newblock In \emph{Conference on Learning Theory}, pages 1562--1583, 2016.

\bibitem[Zaharia et~al.(2010)Zaharia, Chowdhury, Franklin, Shenker, Stoica,
  et~al.]{zaharia2010spark}
Matei Zaharia, Mosharaf Chowdhury, Michael~J Franklin, Scott Shenker, Ion
  Stoica, et~al.
\newblock {Spark: Cluster computing with working sets.}
\newblock \emph{HotCloud}, 10\penalty0 (10-10):\penalty0 95, 2010.

\end{thebibliography}
}

\end{document}